\documentclass[11pt]{article}
\pdfoutput=1

\usepackage{mathrsfs}
\usepackage{amsmath,amssymb}
\usepackage{bm}
\usepackage{natbib}
\usepackage[usenames]{color}
\usepackage{amsthm}

\usepackage{multirow} 
\usepackage{enumitem}

\usepackage{subfigure}
\usepackage{mathtools}

\usepackage[utf8]{inputenc} % allow utf-8 input
\usepackage[T1]{fontenc}    % use 8-bit T1 fonts
\usepackage{url}            % simple URL typesetting
\usepackage{booktabs}       % professional-quality tables
\usepackage{amsfonts}       % blackboard math symbols
\usepackage{nicefrac}       % compact symbols for 1/2, etc.
\usepackage{microtype}      % microtypography
\usepackage{xcolor}         % colors

\usepackage[colorlinks,
linkcolor=red,
anchorcolor=blue,
citecolor=blue
]{hyperref}

\allowdisplaybreaks

\usepackage{mylatexstyle}

\usepackage{setspace}
\usepackage[left=1in, right=1in, top=1in, bottom=1in]{geometry}

\usepackage{xcolor}
\newenvironment{nscenter}
 {\parskip=5pt\par\nopagebreak\centering}
 {\par\noindent\ignorespacesafterend}

\ifdefined\final
\usepackage[disable]{todonotes}
\else
\usepackage[textsize=tiny]{todonotes}
\fi
\newcommand{\Sb}{\mathbf{S}}

\newcommand{\grad}[2]{\frac{\partial #1}{\partial #2}}

\DeclareMathOperator{\spanv}{span}
\DeclareMathOperator{\Proj}{Proj}
\DeclareMathOperator{\dist}{dist}

\title{\huge Looped Transformers with Layer Normalization Provably Learn the Power Method}

\author
{
Lyumin Wu\thanks{%\scriptsize Department of Statistics \& Actuarial Science, School of Computing \& Data Science, 
\scriptsize School of Computing \& Data Science, The University of Hong Kong; {\tt lyuminwu001@connect.hku.hk}}
\qquad
Chenyang Zhang\thanks{%\scriptsize Department of Statistics \& Actuarial Science, School of Computing \& Data Science, 
\scriptsize School of Computing \& Data Science, The University of Hong Kong; {\tt chyzhang@connect.hku.hk}}
\qquad
Yuan Cao\thanks{
\scriptsize School of Computing \& Data Science, 
	The University of Hong Kong;  {\tt yuancao@hku.hk}}
}

\date{}

\begin{document}

\maketitle

\begin{abstract}
Transformers have achieved remarkable success across a wide range of applications, and a growing body of work suggests that part of their strength comes from their ability to learn and execute algorithmic procedures. However, our understanding of how transformers learn such algorithms remains limited, especially in the presence of layer normalization (LN). In this work, we study principal component prediction as a concrete testbed for understanding the training dynamics of transformers with LN. We prove that a looped linear transformer with LN, trained by gradient descent, converges to a solution that implements the power method, with each self-attention layer performing one power iteration. Notably, the model is trained only for principal component prediction, rather than being explicitly supervised to implement the power method. Our finding thus reveals an ``algorithmic implicit bias'' of looped transformers with LN: principal-component prediction can in principle be achieved by many mechanisms, yet gradient descent selects one that realizes the power method. We further provide a concrete comparison between transformers with and without LN: even with layerwise guidance from power iterations, a transformer without LN cannot exactly learn the power method, whereas the corresponding transformer with LN can, leading to a provable performance gap in principal component prediction. Our results provide, to our knowledge, the first theoretical analysis of the training dynamics of looped and single-layer transformers with LN, and shed light on the role of LN in transformer models.
\end{abstract}

\section{Introduction}
Transformers \citep{Vaswani2017Attention} have become the dominant architecture across a wide range of applications, including natural language processing \citep{Wolf2020Transformers,Touvron2023Llama}, computer vision \citep{Dosovitskiy2020image,Rao2021Dynamicvit}, and reinforcement learning \citep{Parisotto2020Stabilizing,Janner2021Offline,Chen2021Decision}. Despite their widespread adoption, the underlying mechanisms of transformers remain insufficiently understood due to their architectural complexity.

Several recent studies have sought to interpret transformers by analyzing their ability to execute specific algorithms. A prominent line of research focuses on in-context learning (ICL). In terms of expressive power, \citet{Garg2022What} empirically demonstrated the ICL capabilities of transformers across various function classes. \citet{Bai2023Transformers} theoretically established that transformers can implement a broad spectrum of standard machine learning algorithms, such as least squares, ridge regression, and Lasso. Turning to optimization dynamics, \citet{Zhang2024Trained} analyzed the training trajectory of a single linear attention layer under gradient flow for in-context linear regression, while \citet{Gatmiry2024Can} showed that linear looped transformers can learn to implement multi-step gradient descent for the same task. Many existing works have also investigated how transformers perform probabilistic and unsupervised learning algorithms. \citet{Chen2024Unveiling} proved that a modified two-layer multi-head transformer can learn $n$-gram Markov chains. \citet{Edelman2024evolution} demonstrated that transformers can acquire statistical induction heads in bigram models trained on samples from a Markov chain. More recently, \citet{Cao2025Transformers} showed that transformers can simulate maximum likelihood estimation to learn Bayesian networks, while \citet{He2025Learning} provided a constructive existence proof that a transformer can be manually programmed to perform Principal Component Analysis (PCA).

Notably, most existing theoretical studies of transformers learning or executing algorithms do not incorporate layer normalization (LN), creating a gap between current theory and practical transformer architectures. In this paper, we study a stylized but explicit setting for analyzing LN through principal component prediction. We first analyze end-to-end population gradient descent for a looped linear transformer with LN, and then compare normalized and unnormalized single-layer transformers under layerwise supervision by power-iteration targets.

\noindent\textbf{Contributions.} We summarize the main contributions of this work as follows. \par
\begin{itemize}[leftmargin=*,nosep]

\item We establish a convergence guarantee for looped linear transformers with LN in principal component prediction, and provide a precise characterization of the gradient descent training dynamics (Theorem~\ref{thm:looped_end_to_end}). Interestingly, the learned transformer is equivalent to the power method, with each self-attention layer performing one power iteration, even though the model is not directly supervised by the power method. Our result thus reveals an ``algorithmic implicit bias'' of transformers with LN: although there are many algorithms that can predict principal components, looped transformers trained by gradient descent specifically learn the power method.

\item We provide a concrete comparison between transformers with and without LN in principal component prediction. We show that, even when a looped transformer without LN is trained with layerwise guidance from power iterations, it still cannot perfectly learn the power method (Theorem~\ref{thm:unnorm convergence}). In contrast, when the same layerwise guidance is used to train transformers with LN, the model successfully learns the power method (Theorem~\ref{thm:convergence}). This difference leads to a concrete performance gap in predicting the leading principal components (Theorem~\ref{thm:looped error}).

\item To the best of our knowledge, this work provides the first theoretical analysis of the training dynamics of looped and single-layer transformers with LN. To address the input-dependent nonlinearities introduced by LN, we develop new theoretical tools, including a Schur's lemma argument showing the preservation of specific weight-matrix structures during training and a dominated-convergence-based analysis that controls gradient behavior. Our results show that LN keeps the optimization dynamics well-behaved even when model weights diverge, highlighting its crucial role in shaping the optimization landscape of transformers.
\end{itemize}

\noindent\textbf{Notation.} For a vector $\vb \in \RR^d$, we denote its $\ell_2$-norm by $\|\vb\|_2$. 
For a matrix $\Ab \in \RR^{m \times n}$, we denote by $\|\Ab\|_2$ its spectral norm and by $\|\Ab\|_{\mathsf{max}}$ its maximum absolute entry. For a symmetric matrix $\Ab$, we denote by $\lambda_{\min}(\Ab)$ its smallest eigenvalue. The notation $\Ab_{:,j} \in \RR^{m}$ refers to the $j$-th column of $\Ab$, while $\Ab_{i,:} \in \RR^{n}$ refers to its $i$-th row.
The symbols $\mathbf{0}_n$ and $\mathbf{0}_{m\times n}$ denote the zero vector in $\RR^n$ 
and the zero matrix in $\RR^{m\times n}$, respectively. 
We denote by $\Ib_d$ the $d$-dimensional identity matrix, 
and by $\Ib$ the identity matrix when the dimension is clear from the context. 
For a vector $\ub \in \RR^d$ and a linear subspace $\cV\subseteq \RR^d$, we define their distance as $\mathrm{Dist}(\ub,\cV) = \min_{\vb\in \cV} \| \ub - \vb \|_2$. 
For an integer $n$, we denote $[n] = \{ 1,2, \ldots, n \}$. For any $\alpha \in \mathbb{R}$, let $\lfloor \alpha \rfloor$ denote the largest integer that is smaller than or equal to $\alpha$. For two sequences $\{a_n\}$ and $\{b_n\}$, denote $a_n = O(b_n)$ if there exists a constant $C> 0$ such that $|a_n|\leq C |b_n|$ for all large enough $n$. Denote $a_n = \Omega(b_n)$ if $b_n = O(a_n)$. We say $a_n = \Theta(b_n)$ if both $a_n = O(b_n)$ and $a_n = \Omega(b_n)$ hold.

\section{Related Works}
\textbf{Expressive power of transformers.}
A growing body of recent work studies the expressive power of transformers from the perspective of their ability to represent functions and implement algorithms. \citet{yuntransformers, dehghaniuniversal, perez2021attention, wei2022statistically} focused on the universal approximation of transformers and showed that transformers or attention mechanisms can approximate arbitrary sequence-to-sequence mappings and even simulate Turing machines. \citet{likhosherstov2021expressive} analyzed the intrinsic expressive power of self-attention matrices, while \citet{sanford2024representational} characterized both the representational strengths and fundamental limitations of attention layers through parameter complexity lower bounds.
Beyond universality, several works investigated the fine-grained algorithmic behaviors of transformers. \citet{bhattamishra2020ability, bhattamishra2022simplicity, liutransformers} studied transformers’ ability to recognize formal languages and learn automata-like structures, revealing both shortcut learning and simplicity bias. \citet{sahiner2022unraveling} interpreted attention through convex duality and connected transformer computations to finite-dimensional convex optimization problems.
A closely related line of work views transformers as implicit optimizers. \citet{olsson2022context} empirically identified induction heads and interpreted attention as implementing gradient-like updates, while \citet{Garg2022What} analyzed what classes of functions can be learned purely in context. \citet{dong2022survey} provided a comprehensive survey of in-context learning. \citet{guo2023transformers, Bai2023Transformers} further demonstrated the in-context learning capacities of transformers by constructing multi-head ReLU transformers capable of performing a variety of learning methods, including ridge regression, lasso regression, generalized linear models, and learning with shallow neural networks. \citet{chen2024transformers} demonstrated that transformers can utilize the multi-head structure to solve sparse regression tasks. \citet{Cao2025Transformers} studied how transformers perform in-context maximum likelihood estimation for Bayesian networks and autoregressively generate new samples from the learned probabilistic model.
\citet{He2025Learning} demonstrated that transformers can implement spectral methods and perform statistical estimation tasks such as PCA and Gaussian mixture clustering. 

\textbf{Optimization of transformers.}
Besides the expressive power of transformer models, a line of recent works investigates the optimization of transformers, mainly focusing on shallow architectures (single-layer or two-layer models).
\citet{kunstnernoise2023, li2024optimization} studied using the sign gradient to optimize transformers and investigated generalization. \citet{pan2023toward, zhang2020adaptive} investigated the advantages of adaptive optimizers in training transformers. \citet{li2023transformers} studied how transformers learn semantic structure, showing that embeddings and self-attention encode topical co-occurrence patterns by strengthening similarity and attention among words from the same topic. \citet{Tian2023Scan, tianjoma2024} investigated the training dynamics of transformers jointly with a decoder layer and MLP layers. \citet{tarzanagh2023transformers, ataee2023max, Huang2025How}
investigated the implicit bias of training the softmax attention layer and its connection to the SVM solution. \citet{Chen2024Unveiling, Nichani2024How, Edelman2024evolution} studied how transformers can be trained to utilize the induction heads to infer the causal structure/Markov chain. \citet{gao2024global} addressed the global convergence of transformers under certain conditions. Furthermore, many other existing works investigate the optimization of transformers under the so-called ``in-context learning'' settings. \citet{Ahn2023Transformers, Zhang2024Trained, Huang2024context} studied the training dynamics of single-layer linear/softmax attention in solving in-context linear regression tasks. \citet{Li2024One} analyzed the in-context learning of transformers in solving one-nearest neighbor selection. \citet{zhang2024context} showed that the MLP component in the transformer block can learn to simulate the initialization of in-context gradient descent. \citet{siyu2024training} investigated mechanisms of multi-head attention and demonstrated that all heads exhibit two distinct patterns and jointly interact to solve in-context linear regression.  
\citet{Huang2025Transformers} considered a Chain of Thought (CoT) training regime and showed that transformer models can be trained to perform multi-step gradient descent. \citet{Yang2025Multi} studied how transformers learn to perform symbolic multi-step reasoning via chain-of-thought, proving that even one-layer multi-head transformers can autonomously specialize and coordinate attention heads to implement sequential algorithmic procedures such as path-finding in trees.

\section{Preliminaries}
In this section, we introduce the details of our problem setup. We first give an overview of the principal component prediction problem, and then introduce the specific transformer architecture we consider to solve the problem.
\subsection{Principal Component Prediction}
In this subsection, we introduce the principal component prediction task considered in this paper. 

\begin{definition}\label{def:principal_component}
    For data vectors $\xb_1, \dots, \xb_n \in \RR^d$, let \(\Xb = [\xb_1,\ldots, \xb_n] \in \RR^{d\times n}\) be the data matrix, and suppose that the matrix $\Xb\Xb^\top$ has eigenvalue decomposition
    \begin{equation*}%\label{eq: XX decomposition}
\Xb\Xb^\top
=
\sum_{i=1}^d
\lambda_i(\Xb)\vb_i(\Xb)\vb_i(\Xb)^\top,
\quad
\lambda_1(\Xb)=\cdots=\lambda_{r(\Xb)}(\Xb)>
\lambda_{r(\Xb)+1}(\Xb)\ge\cdots\ge\lambda_d(\Xb),
\end{equation*}
where \(r(\Xb)\in\{1,\ldots,d-1\}\) is the multiplicity of the largest eigenvalue of $\Xb\Xb^\top$. Then the leading principal subspace of $\Xb$ is defined as $\cV(\Xb) \coloneqq \mathrm{span}\{\vb_1(\Xb),\ldots,\vb_{r(\Xb)}(\Xb)\}$, and each nonzero $\vb\in \cV(\Xb)$ is called a leading principal component of $\Xb$.
\end{definition}
In this work, we aim to study whether transformers can be trained to predict a leading principal component of an input matrix. We make several assumptions on the distribution of the data matrix $\Xb$, which are specified as follows.

\begin{assumption}\label{assump:distribution}
    $\Xb\in\RR^{d\times n}$ is generated from a distribution $\PP_{\Xb}$ satisfying:
    \begin{enumerate}[leftmargin=*, label=(A\arabic*),nosep]
        \item  Each column of $\Xb$ has unit norm. 
        \item  For any orthogonal matrix $\Pb$, $\Xb$ and $\Pb\Xb $ have the same distribution.
        \item  $\Upsilon(n,d) \coloneqq \EE_{\Xb\sim \PP_\Xb}  \big[ \lambda_d(\Xb)^{-1}\big ]\allowbreak < \infty$, where $\lambda_d(\Xb)$ is defined in Definition~\ref{def:principal_component}.
    \end{enumerate}
\end{assumption}
% Assumption~\ref{assump:disxtribution}(A3) is a technical assumption ensuring that certain gradients have finite expectation in our analysis. A convenient sufficient (but stronger) condition for (A3) is a spectral lower bound, e.g., $\lambda_{\min}(\Xb\Xb^\top)\ \ge\ \underline{\lambda}\ >\ 0$ (deterministically or almost surely),
% which implies $\Upsilon(n,d)\le \underline{\lambda}^{-1}$. 
% Below we give an example satisfying Assumption~\ref{assump:distribution}.
Assumption~\ref{assump:distribution}(A3) is a technical integrability condition ensuring that the population gradients in our analysis are well defined. A convenient sufficient (but stronger) condition for (A3) is a uniform spectral lower bound, e.g., $\lambda_{\min}(\Xb\Xb^\top)\ge \underline{\lambda}>0$ (deterministically or almost surely), which implies $
\Upsilon(n,d)\le \underline{\lambda}^{-1}$. Below we give an example satisfying Assumption~\ref{assump:distribution}.

% Below we give an example satisfying Assumption~\ref{assump:distribution}.

\begin{example}\label{first example}
Let $\Ub\in\RR^{d\times n}$ be a fixed matrix whose columns have unit norm and satisfy $\lambda_{\min}(\Ub\Ub^\top) > 0$.
Draw an orthogonal matrix $\Rb$ uniformly at random over all orthogonal $d\times d$ matrices, and define $\Xb=\Rb\Ub$.
Then $\Xb$ has unit columns, and by construction $\Xb$ is rotation-invariant in distribution. Moreover,
$\Xb\Xb^\top$ shares the same eigenvalues as $\Ub\Ub^\top$. This provides the bound:
\begin{equation*}
\EE\big[\lambda_d(\Xb)^{-1}\big]
\le
\lambda_{\min}(\Ub\Ub^\top)^{-1}
<\infty.
\end{equation*}
Therefore, $\Xb$ satisfies Assumption~\ref{assump:distribution}.
\end{example}
Example~\ref{first example} shows that Assumption~\ref{assump:distribution} is satisfied by a broad class of rotation-invariant constructions: one can start from any fixed full-rank matrix and randomly rotate it to obtain training data satisfying the assumption. Moreover, it is also easy to see that, when $n\geq d$, any distribution satisfying Assumption~\ref{assump:distribution}(A1) and (A2) can be easily utilized to generate data by rejecting samples whose $d$-th eigenvalue is smaller than a certain positive threshold. 

% Example~\ref{first example} shows that Assumption~\ref{assump:distribution} is a mild assumption: one can simply start from an arbitrary fixed full-rank matrix, and randomly rotate it to obtain the training data that satisfy this assumption.

\subsection{Transformer Architecture}
In this subsection, we introduce the transformer model considered in our analysis.

\textbf{Self-attention layer.}
We consider a one-layer self-attention module with a residual connection, followed by LN. 
%, which is often taken in recent theoretical frameworks . 
We define its input matrix $\Eb$ by combining the data matrix $\Xb \in \RR^{d \times n}$ with an independent random vector $\ab$ uniformly distributed on the unit sphere $\mathbb{S}^{d-1}$ as
\begin{equation*}
    \Eb=\begin{pmatrix} \Xb & \mathbf{0}_{d\times 1} \\ \mathbf{0}_{d\times n} & \ab \end{pmatrix}\in \RR^{2d\times (n+1)}.
\end{equation*}
Here, the vector \(\ab\) serves as an initial query vector from which the model predicts a leading principal component of \(\Xb\). Following the setting in \citet{Ahn2023Transformers,Ahn2024Linear,Gatmiry2024Can,Zhang2024Trained}, we remove the softmax$(\cdot)$ nonlinearity, consolidate the projection and value matrices into $\Vb \in \RR^{2d \times 2d}$, and merge the key and query matrices into $\Wb \in \RR^{2d \times 2d}$. 
The transformer layer parameterized by $\btheta = (\Vb, \Wb)$ is then defined as:
\begin{equation*}
\mathrm{TF}(\Eb; \btheta) = \mathrm{LN}\Big(\Eb + \Vb\Eb\Eb^{\top}\Wb\Eb\Big),
\end{equation*}
where the column-wise normalizer $\mathrm{LN}:\RR^{2d\times (n+1)}\to\RR^{2d\times (n+1)}$ is given by
$[\mathrm{LN}(\Ab)]_{:,j}
= {\Ab_{:,j}}/{\|\Ab_{:,j}\|_2}
$ for $j\in [n+1]$. Note that instead of the standard LN, which includes a mean-centering operation, we adopt a formulation structurally equivalent to the RMSNorm \citep{Zhang2019Root} (up to a scalar constant $\sqrt{2d}$). This choice aligns with the design of state-of-the-art Large Language Models, including Llama \citep{Touvron2023Llama} and Gemma \citep{Team2024Gemma,Team2025Gemma}, where mean-centering is removed to improve stability and efficiency without sacrificing performance. 
% The final prediction is defined as the last $d$ rows of the final column of the output matrix
% \begin{equation}\label{transformer prediction}
%   \yb(\Eb; \btheta)= \mathrm{LN}\Big(\Eb + \Vb\Eb\Eb^{\top}\Wb\Eb\Big)_{d+1:2d, n+1}.
% \end{equation}
% For later use, we also express the parameter matrices in  $d \times d$ block form as
% \begin{equation*}
%     \Vb =
%     \begin{pmatrix}
%         \Vb_{11} & \Vb_{12} \\
%         \Vb_{21} & \Vb_{22}
%     \end{pmatrix},
%     \quad
%     \Wb =
%     \begin{pmatrix}
%         \Wb_{11} & \Wb_{12} \\
%         \Wb_{21} & \Wb_{22}
%     \end{pmatrix}
%     \in \mathbb{R}^{2d \times 2d},
% \end{equation*}
% where $\Vb_{ij}, \Wb_{ij} \in \mathbb{R}^{d \times d}$ for $i,j \in [2]$.

\noindent\textbf{Looped $L$-layer Transformer.} 
The looped transformer is defined by stacking \(L\) self-attention layers with shared parameters, yielding an \(L\)-layer transformer denoted by \(\mathrm{TF}_L\). The parameters \(\btheta\) are shared across all layers to enforce structural recurrence. Specifically, the output of layer \(\ell-1\) serves as the input to layer \(\ell\), which is defined recursively as
\begin{equation*}
    \Eb^{(0)} = \Eb, \quad 
    \Eb^{(\ell)} = \mathrm{TF}(\Eb^{(\ell-1)}; \btheta) 
    \quad \text{for } \ell \in [L], \quad \mathrm{TF}_L(\Eb; \btheta) := \Eb^{(L)}.
\end{equation*}
 % The final output of the looped model is \(\mathrm{TF}_L(\Eb; \btheta) = \Eb^{(L)}\). 
 The final prediction of the looped model is defined as the last $d$ rows of the final column of $\mathrm{TF}_L(\Eb; \btheta)$:
 % , denoted by \(\yb_L\), is extracted from the output of the final layer:
\begin{equation}\label{eq:looped transformer prediction}
    \yb_L(\Eb; \btheta) = \big[\mathrm{TF}_L(\Eb; \btheta) \big]_{d+1:2d, n+1}.
\end{equation}

\section{Main Result}
In this section, we first characterize the training dynamics of looped linear transformers with LN in principal component prediction, and then provide a detailed comparison between transformers with and without LN, revealing the structural advantage of LN. Finally, we establish out-of-distribution (OOD) performance guarantees for the trained transformers.

\subsection{Looped Transformer Learns Power Method}
\label{subsec:e2e_looped_training}
In this subsection, we establish training guarantees for looped linear transformer models \eqref{eq:looped transformer prediction}.

We consider a training objective that measures the squared distance from the output of the looped \(L\)-layer model \(\yb_L(\Eb;\btheta)\) to the principal subspace of \(\Xb\Xb^\top\):
\[
\mathcal L_L(\btheta)
:=
\EE_{\Xb,\ab}\!\left[ \mathrm{Dist}^2(\yb_L(\Eb;\btheta) , \cV(\Xb))\right] = \EE_{\Xb,\ab}\!\left[ 
\left\|
\yb_L(\Eb;\btheta)- \mathrm{Proj}_{\cV(\Xb)}(\yb_L(\Eb;\btheta))
\right\|_2^2
\right],
\]
where $\cV(\Xb)$ is defined in Definition~\ref{def:principal_component} and the expectation is taken over the joint distribution of $\Xb$ and $\ab$.
We consider training the looped transformer by gradient descent with learning rate $\eta$, i.e.,
\begin{equation}\label{GD}
\begin{aligned}
    \Vb^{(t+1)}=\Vb^{(t)}-\eta \nabla_{\Vb} \mathcal{L}_L(\btheta^{(t)}),\quad
    \Wb^{(t+1)}=\Wb^{(t)}-\eta \nabla_{\Wb} \mathcal{L}_L(\btheta^{(t)}).
\end{aligned}
\end{equation}
We consider initialization 
$\Wb^{(0)}=\mathbf{0}_{2d\times 2d}$ and $\Vb^{(0)}\in \RR^{2d\times 2d}$ with $d\times d$ blocks 
$\Vb_{11}^{(0)}=\mathbf{0}$, $\Vb_{12}^{(0)}=\mathbf{0}$, $\Vb_{21}^{(0)}=\Ib_d$, $\Vb_{22}^{(0)}=\mathbf{0}$. This configuration ensures non-zero gradients at initialization. Similar initialization strategies with specific non-zero blocks are common in recent studies \citep{Li2024One,Zhang2024Trained,Gatmiry2024Can,Huang2025Transformers}.

\begin{theorem}%[End-to-end training dynamics of looped transformers]
\label{thm:looped_end_to_end}
Suppose that \(d\ge3\), \(L\ge1\), \(0<\eta\le d/(2Ln)\), and Assumption~\ref{assump:distribution} holds. %We further assume that \(\EE[\lambda_d(\Xb)^{-1}]<\infty\), where $\lambda_d(\Xb)$ is defined in Definition~\ref{def:principal_component}. 
The gradient descent dynamics exhibit the following properties:
\begin{enumerate}[leftmargin = *,nosep]
\item \textbf{Structure of the Parameter Matrices:}
For all training steps $t \ge 0$, the parameter matrices $\Wb^{(t)}$ and $\Vb^{(t)}$ maintain a specific structure:
\begin{align}\label{eq:WVstructure}
    \Wb^{(t)}
=
\begin{pmatrix}
\mathbf 0_{d\times d}& w_t\Ib_d\\
\mathbf 0_{d\times d}&\mathbf 0_{d\times d}
\end{pmatrix},
\qquad
\Vb^{(t)}
=
\begin{pmatrix}
\mathbf 0_{d\times d}&\mathbf 0_{d\times d}\\
v_t\Ib_d&\mathbf 0_{d\times d}
\end{pmatrix}.
\end{align}
Furthermore, the scalar coefficients satisfy $w_t, v_t = \Theta\big((\eta\Upsilon_1 t/d)^{1/4}\big)$ for a positive constant \(\Upsilon_1=\Upsilon_1(L,d,n)>0\).

\item \textbf{Convergence:}
The training loss $\mathcal L_L(\btheta^{(t)}) $ converges to a strictly positive value $\mathcal L_L^{(\infty)}>0$:
\[
\mathcal L_L(\btheta^{(t)})
-
\mathcal L_L^{(\infty)} = O\Bigg(\sqrt{\frac{d\Upsilon_1}{\eta t}} \Bigg),
% \left(\frac{d\Upsilon_1}{\eta}\right)^{1/2}
% t^{-1/2}(1+o_t(1)),
\]
where $
\mathcal L_L^{(\infty)}
=
\EE_{\Xb,\ab}\!\left[
\mathrm{Dist}^2({(\Xb\Xb^\top)^L\ab}/{\|(\Xb\Xb^\top)^L\ab\|_2} , \cV(\Xb))
\right]
$. 
% , and $C$ is a constant that only depends on the distribution of $\Xb$.
\end{enumerate}
\end{theorem}

%the transformer's output $\yb(\Eb; \btheta^{(t)})$ has the form:

% We defer the detailed proof to Appendix~\ref{appendix:end-to-end training}. Theorem~\ref{thm:looped_end_to_end} provides both a convergence guarantee for population gradient descent and an exact characterization of the parameter structure along training. In particular, with $\Wb^{(t)}$ and $\Vb^{(t)}$ specified in \eqref{eq:WVstructure} and $w_t, v_t = \Theta\big((\eta\Upsilon_1 t/d)^{1/4}\big)$ diverging to infinity, the learned looped module realizes a normalized finite-step spectral iteration. More precisely, each layer applies one normalized update with matrix $I+\rho_t X X^\top$, where $\rho_t = w_t v_t$. To see this, with the learned parameters at step $t$, we have

We defer the detailed proof to Appendix~\ref{appendix:end-to-end training}. Theorem~\ref{thm:looped_end_to_end} not only gives a convergence guarantee for gradient descent but also gives a precise characterization of the structures of the weight matrices throughout training. Interestingly, 
with $\Wb^{(t)}$ and $\Vb^{(t)}$ specified in \eqref{eq:WVstructure} and $w_t, v_t = \Theta\big((\eta\Upsilon_1 t/d)^{1/4}\big)$ diverging to infinity, we can observe that 
\begin{nscenter}
\emph{\parbox{0.9\columnwidth}{As training progresses, the transformer converges to a model equivalent to the $L$-step power method, with each self-attention layer contributing one step of power iteration.}}
\end{nscenter}
% \begin{nscenter}
% % \vspace{0.05in}
%     \emph{\parbox{0.9\columnwidth}{The trained transformer is equivalently performing the power method, with each self-attention layer performing one step of power iteration.}}
% \end{nscenter}
To see this, with the learned parameters at step $t$, we have
\begin{align*}
\Eb^{(1)}
&=
\mathrm{LN}\!\left(
\begin{pmatrix}
\Xb & \mathbf{0}\\
\mathbf{0} & \ab
\end{pmatrix}
+
\begin{pmatrix}
\mathbf{0} & \mathbf{0}\\
v_t\Ib_d & \mathbf{0}
\end{pmatrix}
\begin{pmatrix}
\Xb\Xb^\top & \mathbf{0}\\
\mathbf{0} & \ab\ab^\top
\end{pmatrix}
\begin{pmatrix}
\mathbf{0} & w_t\Ib_d\\
\mathbf{0} & \mathbf{0}
\end{pmatrix}
\begin{pmatrix}
\Xb & \mathbf{0}\\
\mathbf{0} & \ab
\end{pmatrix}
\right) \\
&=
\begin{pmatrix}
\Xb & \mathbf{0}\\
\mathbf{0} & \ab_{1}
\end{pmatrix},
\quad
\text{where}
\quad
\ab_{1}
=
\frac{(\Ib_d+\rho_t\Xb\Xb^\top)\ab}
{\|(\Ib_d+\rho_t\Xb\Xb^\top)\ab\|_2},
\quad
\rho_t=w_tv_t.
\end{align*}
Iterating this update through \(L\) looped layers gives
\begin{align*}
&\yb_L(\Eb;\btheta^{(t)})
=
\frac{
(\Ib_d+\rho_t\Xb\Xb^\top)^L\ab
}{
\|(\Ib_d+\rho_t\Xb\Xb^\top)^L\ab\|_2
}.
\end{align*}
% \begin{align*}
% &\yb_L(\Eb;\btheta^{(t)})
% =
% \left[
% \mathrm{LN}\!\left(
% \Eb^{(L-1)}
% +
% \Vb^{(t)}\Eb^{(L-1)}\Eb^{(L-1)\top}\Wb^{(t)}\Eb^{(L-1)}
% \right)
% \right]_{d+1:2d,n+1}
% \\
% &=
% \left[
% \mathrm{LN}\!\left(
% \begin{pmatrix}
% \Xb & \mathbf{0}\\
% \mathbf{0} & \ab_{L-1}
% \end{pmatrix}
% +
% \begin{pmatrix}
% \mathbf{0} & \mathbf{0}\\
% v_t\Ib_d & \mathbf{0}
% \end{pmatrix}
% \begin{pmatrix}
% \Xb\Xb^\top & \mathbf{0}\\
% \mathbf{0} & \ab_{L-1}\ab_{L-1}^\top
% \end{pmatrix}
% \begin{pmatrix}
% \mathbf{0} & w_t\Ib_d\\
% \mathbf{0} & \mathbf{0}
% \end{pmatrix}
% \begin{pmatrix}
% \Xb & \mathbf{0}\\
% \mathbf{0} & \ab_{L-1}
% \end{pmatrix}
% \right)
% \right]_{d+1:2d,n+1}
% \\
% &=
% \frac{
% (\Ib_d+\rho_t\Xb\Xb^\top)\ab_{L-1}
% }{
% \|(\Ib_d+\rho_t\Xb\Xb^\top)\ab_{L-1}\|_2
% }
% =
% \frac{
% (\Ib_d+\rho_t\Xb\Xb^\top)^L\ab
% }{
% \|(\Ib_d+\rho_t\Xb\Xb^\top)^L\ab\|_2
% }.
% \end{align*}
As training progresses, $\rho_t$ diverges at the rate $\Theta((\eta\Upsilon_1 t/d)^{1/2})$. Thus, for large $t$, the term $\rho_t\Xb\Xb^\top$ dominates $\Ib_d$, and the output converges to the $L$-step power-method iterate:
\begin{equation*}
\lim_{t\to\infty} \yb_L(\Eb;\btheta^{(t)}) = \frac{(\Xb\Xb^\top)^L\ab}{\|(\Xb\Xb^\top)^L\ab\|_2}.
\end{equation*}
This limit exactly matches the $L$-step power method. The significance of this equivalence is that the model is trained only to predict principal components. This can be interpreted as an \textit{algorithmic implicit bias} result: among the many algorithms that can predict principal components, looped transformers with LN trained by gradient descent learn the power method.
We also emphasize that Theorem~\ref{thm:looped_end_to_end} shows that the loss converges to $\mathcal L_L^{(\infty)}>0$, but this result does not imply any residual gap between the trained transformer and the power method. The trained model converges to the $L$-step power method
exactly; the limiting loss is positive simply because, for finite $L$, the $L$-step power iterate need not lie in the leading principal subspace.

Recent work \citet{He2025Learning} studied the expressive capacity of multi-layer transformers for PCA. Leveraging the universal approximation power of ReLU activations, they showed that transformers without LN can approximate the top-$k$ principal components of a given data matrix. However, their study does not address whether such behavior can emerge through training. In contrast, our work focuses on the optimization dynamics of extracting the leading principal component. We theoretically demonstrate that looped linear transformers with LN trained by gradient descent learn the power method, even though the model is not directly supervised by the power method. Moreover, unlike the complex constructions required by the existence proofs in \citet{He2025Learning}, our analysis reveals that the transformer learns a simpler and more intuitive mechanism through training. Finally, our study explicitly incorporates LN into the theoretical analysis, distinguishing our findings from those of \citet{He2025Learning}.
% , which do not consider normalization layers.

The proof of Theorem~\ref{thm:looped_end_to_end} develops several theoretical tools for analyzing the training dynamics of transformer layers with LN, which may be of independent interest. First, we use Schur’s lemma to characterize the structures preserved by the parameter matrices along the optimization trajectory (Appendix \ref{lem:active_scalar_gradient}), extending previous analyses \citep{Li2024One,Wang2024Transformers} to architectures with LN. Second, LN introduces additional input-dependent terms in the population gradient that can in principle diverge, and we control these terms via a refined dominated-convergence-based argument (Appendix \ref{lem:scalar_derivative}). Finally, we show that, in the presence of LN, the training loss does not admit a finite minimizer, and is minimized only in the limit as the scale parameter $\rho_t = w_t v_t$ tends to infinity. By analyzing the quadratic increment \(\rho_{t+1}^2-\rho_t^2\) (Appendix \ref{lem:scalar_dynamics}), we obtain precise rates for the divergence of $\rho_t$.

\subsection{Comparison between Transformers with and without LN}
In this subsection, we provide a detailed comparison between transformers with and without LN trained under layerwise guidance by power iterations. 
% We show that transformers without LN cannot be trained to exactly implement the power method, even with layerwise guidance from exact power iterations. In contrast, when the same layerwise guidance is used to train transformers with LN, the model successfully learns the power method.

\subsubsection{Transformer without LN}
\label{sec:unnorm-main}
To understand the advantage of LN, we analyze how an unnormalized model learns to perform power iterations. Rather than considering direct end-to-end training of the looped model, we establish a stronger result: even layerwise guidance by power iterations cannot teach transformers without LN to perform the power method.
The unnormalized transformer layer is defined as
\begin{equation}\label{unnormalized transformer prediction}
\begin{aligned}
\tilde{\mathrm{TF}}(\Eb; \btheta)=\Eb+\Vb\Eb\Eb^\top\Wb\Eb, \quad\tilde{\yb}(\Eb;\btheta)=\big[\tilde{\mathrm{TF}}(\Eb;\btheta)\big]_{d+1:2d,n+1}.\end{aligned}
\end{equation}
% We first present the detailed setup for proving this result.
\noindent\textbf{Training with Layerwise Guidance.} We consider training a single-layer module \eqref{unnormalized transformer prediction} to learn one step of the power iteration, using the same initialization as in Theorem~\ref{thm:looped_end_to_end}.
% , and then looping this module at inference time.

The training objective measures the distance between the
output of the unnormalized transformer layer $ \tilde{\yb}(\Eb;\btheta)$ and a given training target $\yb_{\text{target}}$, and is given as
\begin{equation}\label{training target}
\tilde{\mathcal L}(\btheta)
\coloneqq
\EE_{\Xb,\ab}\left[\left\|\tilde{\yb}(\Eb;\btheta)-\yb_{\text{target}}\right\|_2^2\right].
\end{equation}
Note that for models without LN, it is not obvious whether the target response given by the one-step power method should be modified to yield better performance. In order to make a fair and thorough comparison between cases with and without LN, here we consider two choices of $\yb_{\text{target}}$ in \eqref{training target}:
\begin{align*}
     \yb_{\text{target}}^{(1)} = \frac{\Xb\Xb^\top\ab}{\|\Xb\Xb^\top\ab\|_2}, \quad 
    \yb_{\text{target}}^{(2)} = {\Xb\Xb^\top\ab},
\end{align*}
which correspond to normalized and unnormalized targets, respectively. For both choices, training guarantees for transformers without LN are given in the following theorem.

\begin{theorem}
\label{thm:unnorm convergence}
Suppose that  $d\ge3$, and Assumption~\ref{assump:distribution} holds. Define  $\Upsilon_2\coloneqq\EE[\|\Xb\Xb^\top\ab\|_2^2]$ and 
\begin{align*}
    \gamma^\star = \frac{\EE\bigl[\|\Xb\Xb^\top\ab\|_2\bigr] - \EE\bigl[\ab^\top\Xb\Xb^\top\ab\bigr]}
     {\EE\bigl[\|\Xb\Xb^\top\ab\|_2^2\bigr]} \mathds{1}\{\yb_{\textnormal{target}}= \yb_{\textnormal{target}}^{(1)}\} + \bigg[1 - \frac{\EE\bigl[\ab^\top\Xb\Xb^\top\ab\bigr]}
         {\EE\bigl[\|\Xb\Xb^\top\ab\|_2^2\bigr]}\bigg] \mathds{1}\{\yb_{\textnormal{target}}= \yb_{\textnormal{target}}^{(2)}\}.
\end{align*}
% \begin{equation*}
% % \label{eq:def-gamma-star-unnorm}
% \gamma^\star \coloneqq
% \begin{cases}
% \displaystyle
% \frac{\EE\bigl[\|\Xb\Xb^\top\ab\|_2\bigr] - \EE\bigl[\ab^\top\Xb\Xb^\top\ab\bigr]}
%      {\EE\bigl[\|\Xb\Xb^\top\ab\|_2^2\bigr]},
% & \hspace{-6pt}\text{if } \yb_{\textnormal{target}}= \yb_{\textnormal{target}}^{(1)} ,\\[8pt]
% \displaystyle
% 1 - \frac{\EE\bigl[\ab^\top\Xb\Xb^\top\ab\bigr]}
%          {\EE\bigl[\|\Xb\Xb^\top\ab\|_2^2\bigr]},
% & \hspace{-6pt}\text{if } \yb_{\textnormal{target}}= \yb_{\textnormal{target}}^{(2)}.
% \end{cases}
% \end{equation*}
Then for both choices of $\yb_{\textnormal{target}}$ and any 
%\begin{align*}
    $0<\eta \leq{d}/{8 \Upsilon_2\sqrt{1+16(\gamma^\star)^2}}$,
%\end{align*}
the following results hold.
\begin{enumerate}[leftmargin = *,nosep]
    \item \textbf{Structure of the Parameter Matrices:} For all training steps $t \ge 0$, the parameter matrices $\Wb^{(t)}$ and $\Vb^{(t)}$ maintain a specific structure:
    \begin{equation*}
        \Wb^{(t)}=
\begin{pmatrix}
\mathbf{0}_{d\times d} & \tilde{\alpha}_t \Ib_d\\[2pt]
\mathbf{0}_{d\times d} & \mathbf{0}_{d\times d} \vphantom{\tilde{\beta}_t}
\end{pmatrix},
\quad
\Vb^{(t)}=
\begin{pmatrix}
\mathbf{0}_{d\times d} & \mathbf{0}_{d\times d}\\[2pt]
\tilde{\beta}_t \Ib_d & \mathbf{0}_{d\times d}
\end{pmatrix},
    \end{equation*}
   where $\tilde{\alpha}_t,\tilde{\beta}_t$ are scalars. Moreover, their product $\tilde{\gamma}_t\coloneqq\tilde{\alpha}_t\tilde{\beta}_t$ converges to $\gamma^\star$ defined above.
 \item \textbf{Convergence}: The training loss $\tilde{\mathcal L}(\btheta^{(t)}) $ converges linearly to a strictly positive value $\tilde{\mathcal L}^\star>0$:
    \begin{align*}
     \tilde{\mathcal L}(\btheta^{(t)}) - \tilde{\mathcal L}^\star
    \le \rho^{t-1} \big(  \tilde{\mathcal L}(\btheta^{(1)}) - \tilde{\mathcal L}^\star \big),\quad
    |\tilde{\gamma}_t - \gamma^\star|
    \le   \rho^{(t-1)/2}\,
    \sqrt{(\tilde{\mathcal L}(\btheta^{(1)}) - \tilde{\mathcal L}^\star) / \Upsilon_2}
    % \sqrt{\frac{\rho^{t-1}}{\Upsilon_2}}\,
    % \sqrt{\tilde{\mathcal L}(\btheta^{(1)}) - \tilde{\mathcal L}^\star}
    \end{align*}
    for all $t \ge 1$, where $\rho \coloneqq 1 - 8\eta^2 {\Upsilon_2^2\gamma^\star}/{d^2}  \in [0, 1)$.
    % $\tilde{\gamma}_1=2\eta {\Upsilon_2\gamma^\star}/{d}  $ and
 \end{enumerate}  
\end{theorem}
We defer the detailed proof to Appendix~\ref{sec:appendix:noLN-full}. Theorem~\ref{thm:unnorm convergence} shows that the unnormalized single-layer attention model defined in \eqref{unnormalized transformer prediction} develops a structure in its parameter matrices similar to that in Theorem~\ref{thm:looped_end_to_end}, but with scalar coefficients converging to finite values. Moreover, the training loss $\tilde{\mathcal L}(\btheta^{(t)})$, defined by the difference between transformer layer and one-step power iteration outputs, converges linearly to a strictly positive limit, regardless of the choice of $\yb_{\text{target}}$ in \eqref{training target}. This implies that a transformer layer without LN cannot be trained to exactly perform the power iteration.

\subsubsection{Transformer with LN}
To provide a direct comparison with the results for transformers without LN, we also consider training transformers with LN under layerwise guidance. We denote the one-layer model output as
\begin{equation}\label{transformer prediction}
\begin{aligned}
\mathrm{TF}(\Eb; \btheta) = \mathrm{LN}\Big(\Eb + \Vb\Eb\Eb^{\top}\Wb\Eb\Big), \quad \yb(\Eb; \btheta)= \big[\mathrm{TF}(\Eb; \btheta) \big]_{d+1:2d, n+1}.\end{aligned}
\end{equation}
% \begin{equation*}
% \mathrm{TF}(\Eb; \btheta) = \mathrm{LN}\Big(\Eb + \Vb\Eb\Eb^{\top}\Wb\Eb\Big),
% \end{equation*}
% \begin{equation}\label{transformer prediction}
%   \yb(\Eb; \btheta)= \mathrm{LN}\Big(\Eb + \Vb\Eb\Eb^{\top}\Wb\Eb\Big)_{d+1:2d, n+1}.
% \end{equation}
We consider a training objective measuring the distance between the output of the transformer layer $\yb(\Eb;\btheta)$ and the ideal one-step power iteration:%, we minimize the Mean-Square Error (MSE) loss:
\begin{equation*}
    \mathcal{L}(\btheta)=\EE_{\Xb,\ab}\left[\left\|\yb(\Eb; \btheta)-\frac{\Xb\Xb^\top\ab}{\|\Xb\Xb^\top\ab\|_2}\right\|_2^2\right],
\end{equation*}
%where the expectation is taken over the joint distribution of $\Xb$ and $\ab$. 
We consider training the transformer layer using the same initialization as in Theorem~\ref{thm:looped_end_to_end}. The following theorem characterizes the gradient descent dynamics for transformers with LN.

\begin{theorem}
\label{thm:convergence}
Suppose that  $d\ge3$, $\eta>0$, and Assumption \ref{assump:distribution} holds. Define the constant $\Upsilon_3 \coloneqq \EE\!\left[
  \bigl(\|\ab\|_2^2 \|\Xb\Xb^\top\ab\|_2^2 - (\ab^\top \Xb\Xb^\top\ab)^2\bigr)
 \|\Xb\Xb^\top\ab\|_2^{-4}
\right]$, which satisfies $0<\Upsilon_3=O(d\Upsilon(n,d)/ n)<\infty$. 
The gradient descent dynamics exhibit the following properties:
\begin{enumerate}[leftmargin = *,nosep]
    \item \textbf{Structure of the Parameter Matrices:} For all training steps $t \ge 0$, the parameter matrices $\Wb^{(t)}$ and $\Vb^{(t)}$ maintain a specific sparse structure:
    \begin{equation*}
        \Wb^{(t)}=
\begin{pmatrix}
\mathbf{0}_{d\times d} & \alpha_t \Ib_d\\[2pt]
\mathbf{0}_{d\times d} & \mathbf{0}_{d\times d}
\end{pmatrix},
\quad
\Vb^{(t)}=
\begin{pmatrix}
\mathbf{0}_{d\times d} & \mathbf{0}_{d\times d}\\[2pt]
\beta_t \Ib_d & \mathbf{0}_{d\times d}
\end{pmatrix}.
    \end{equation*}
    Furthermore, the scalar coefficients satisfy $\alpha_t, \beta_t = \Theta\big((\eta\Upsilon_3 t/d)^{1/6}\big)$.
\item  \textbf{Convergence:} For any sufficiently small $\epsilon>0$, there exists a finite number of iterations $T^\star=O(  {d^{{3}/{2}}\Upsilon(n,d)^{{1}/{2}}}/{\eta n^{{1}/{2}}\epsilon^{{3}/{2}}}) $,
such that for all $t\ge T^\star$,
    \begin{equation*}
        \mathcal{L}(\btheta^{(t)}) \leq \epsilon.
    \end{equation*}
 \end{enumerate}  
\end{theorem}
We defer the detailed proof to Appendix~\ref{proof of theorem 1}. Theorem~\ref{thm:convergence} shows that the one-layer attention model \eqref{transformer prediction} learns one update step of the power method. In contrast to Theorem~\ref{thm:unnorm convergence}, Theorem~\ref{thm:convergence} shows that, for transformer layers equipped with LN, the population loss can be driven arbitrarily close to zero through training. At test time, we further show in Theorem~\ref{thm:looped error} that this contrast translates into a faster convergence rate when the learned module is looped.

\subsection{Out-of-distribution Performance of Looped Models}
In this section, we investigate the OOD capabilities of the trained transformers obtained from Theorems~\ref{thm:unnorm convergence} and~\ref{thm:convergence}, when looped at inference time, as well as the looped transformers obtained from Theorem~\ref{thm:looped_end_to_end}. We consider OOD ``test data'' generated according to the following assumption.

% effectively perform the power method on data distributions substantially more general than those specified in Assumption~\ref{assump:distribution}.
% of the transformers with and without LN trained under the layervise guidance of power iterations, as well as the looped model obtained via end-to-end training in

% looped transformers. 

% We show that the trained transformers obtained from Theorems~\ref{thm:unnorm convergence} and~\ref{thm:convergence}, when looped at inference time, as well as the looped transformers obtained from Theorem~\ref{thm:looped_end_to_end}, effectively perform the power method on data distributions substantially more general than those specified in Assumption~\ref{assump:distribution}. We consider ``test data'' generated according to the following assumption.
\begin{assumption}[Test Distribution]\label{test distribution}
    $\Xb_{\text{test}}\in\RR^{d\times n}$ is generated from a distribution $\PP_{\Xb_{\text{test}}}$ such that each column of $\Xb_{\text{test}}$ has unit norm.
\end{assumption}
Notably, Assumption \ref{test distribution} imposes only condition (A1) from Assumption \ref{assump:distribution}. This relaxation thus allows distribution shift between training and inference, constituting an OOD setting. We next define the necessary notation for a specific sample $\Xb_{\text{test}} \sim \PP_{\Xb_{\text{test}}}$ and a random vector $\ab \in \mathbb{S}^{d-1}$.
\begin{definition}
    For a sample $\Xb_{\text{test}}$, let $\Xb_{\text{test}}\Xb_{\text{test}}^\top$ admit  the eigendecomposition $\Xb_{\text{test}}\Xb_{\text{test}}^\top = \sum_{i=1}^{d} \lambda_i \vb_i \vb_i^\top$ with $\lambda_1 = \dots = \lambda_r > \lambda_{r+1} \ge \dots \ge \lambda_d \ge 0$, where $r$ is the multiplicity of $\lambda_1$. Let $\mathcal{Z}$ denote the principal subspace associated with $\lambda_1, \ldots , \lambda_r$. Then for a vector $\ab \in \mathbb{S}^{d-1}$, we define $c_{r+1} = \langle \vb_{r+1}, \ab \rangle$, denote $\ab_{\mathcal{Z}}$ as the projection of $\ab$ onto $\mathcal{Z}$, and let $\ab_{\perp} = \ab - \ab_{\mathcal{Z}}$. 
\end{definition}
The following theorem shows that, when looped at inference, the distinct training dynamics with and without LN lead to a performance gap in predicting the leading principal components.
\begin{theorem}
\label{thm:looped error}
 Suppose that $\ab$ is an arbitrary unit vector that is not orthogonal to $\cZ$, i.e., $\ab_{\mathcal{Z}}\neq \mathbf{0}$. Then, under Assumption \ref{test distribution}, the following results hold.
\begin{enumerate}[leftmargin = *,nosep]
     \item \textbf{Normalized model:}
    Let the single-layer model \eqref{transformer prediction} be trained for $t$ steps under the conditions of Theorem \ref{thm:convergence}. Then, unrolling $L$ times yields a deep model that gives an output $\yb_L^{(t)}$ when applied to the data matrix $\Xb_{\text{test}}$ and the vector $\ab$. Let $\phi_L^{(t)} \in [0,\pi/2]$ denote the canonical angle between $\yb_L^{(t)}$ and $\cZ$. Then:
    \begin{equation*}
    {|c_{r+1}|}\left(\frac{\lambda_{r+1}}{\lambda_1}\right)^{L}
    \le\lim_{t\to\infty}\sin\phi_L^{(t)}
    \le
    \frac{\|\ab_{\perp}\|_2}{\|\ab_{\cZ}\|_2}\left(\frac{\lambda_{r+1}}{\lambda_1}\right)^{L}.
    \end{equation*}

    \item \textbf{Unnormalized Model:}
   Let the single-layer model \eqref{unnormalized transformer prediction} be trained for $t$ steps under the conditions of Theorem \ref{thm:unnorm convergence}. Then, unrolling $L$ times yields a deep model that gives an output $\tilde{\yb}_L^{(t)}$ when applied to the data matrix $\Xb_{\text{test}}$ and the vector $\ab$. Let $\tilde{\phi}_L^{(t)} \in [0,\pi/2]$ denote the canonical angle between $\tilde{\yb}_L^{(t)}$ and $\cZ$. Then:
    \begin{align*}
    {|c_{r+1}|}\left(\frac{1+\gamma^\star\lambda_{r+1}}{1+\gamma^\star\lambda_1}\right)^{L}
    &\le\lim_{t\to \infty}\sin\tilde{\phi}_L^{(t)}
     \le
    \frac{\|\ab_{\perp}\|_2}{\|\ab_{\cZ}\|_2}\left(\frac{1+\gamma^\star\lambda_{r+1}}{1+\gamma^\star\lambda_1}\right)^{L}.
    \end{align*}
\end{enumerate}
\end{theorem}
We defer the detailed proof to Appendix~\ref{proof of theorem 2}. Note that this theorem covers both in-distribution and OOD settings since Assumption \ref{test distribution} subsumes Assumption \ref{assump:distribution}. This confirms that the model learns the power method mechanism, and that the learned update applies to any inference-time data matrix with unit-norm columns. We also comment that the assumption $\ab_{\mathcal{Z}}\neq \mathbf{0}$ is mild and holds with probability one if $\ab$ is randomly chosen from a non-degenerate and continuous distribution over $\SSS^{d-1}$.

A key step in the proof of Theorem~\ref{thm:looped error} is to show that, for a transformer with LN, looping the learned module yields exponential error decay with a per-layer contraction factor $R_t = {(1+\gamma_t\lambda_{r+1})}/{(1+\gamma_t\lambda_1)}$.
As training proceeds ($t \to \infty$), according to Theorem~\ref{thm:convergence}, the learned parameter $\gamma_t = \alpha_t\beta_t\to\infty$, and therefore this factor $R_t$ decreases, eventually converging to the optimal rate of $\lambda_{r+1}/\lambda_1$ established for the power method \citep{Parlett1998symmetric}. This demonstrates that the model learns an increasingly efficient algorithmic update, eventually matching the optimal convergence rate. In contrast, for a transformer without LN, irrespective of the training target in \eqref{training target}, looping the learned module also results in exponential error decay, but with a fixed and suboptimal factor $R^\star = {(1+\gamma^\star\lambda_{r+1})}/{(1+\gamma^\star\lambda_1)} > \lambda_{r+1}/\lambda_1$ that cannot be improved by further training. This performance gap highlights that LN can fundamentally alter the training dynamics, enabling the model to escape a suboptimal equilibrium and approach the theoretically optimal power method limit.

The following theorem shows the OOD performance of the looped transformers in Theorem~\ref{thm:looped_end_to_end}. 

\begin{theorem}
\label{thm:end-to-end looped error}
 Let the $L$-layer looped model \eqref{eq:looped transformer prediction} be trained by gradient descent for $t$ steps under the conditions of Theorem \ref{thm:looped_end_to_end}, and let $\yb_L^{(t)}$ be the model output when applied to the data matrix $\Xb_{\text{test}}$ and the vector $\ab$. Denote by $\psi_L^{(t)} \in [0,\pi/2]$ the canonical angle between $\yb_L^{(t)}$ and $\cZ$. Then under the conditions of Theorem \ref{thm:looped error}, it holds that
    \begin{equation*}
    {|c_{r+1}|}\left(\frac{\lambda_{r+1}}{\lambda_1}\right)^{L}
    \le\lim_{t\to\infty}\sin\psi_L^{(t)}
    \le
    \frac{\|\ab_{\perp}\|_2}{\|\ab_{\cZ}\|_2}\left(\frac{\lambda_{r+1}}{\lambda_1}\right)^{L}.
    \end{equation*}
\end{theorem}
The detailed proof of Theorem~\ref{thm:end-to-end looped error} is given in Appendix~\ref{appendix: proof of end-to-end looped error}. 
% Theorem~\ref{thm:end-to-end looped error} shows that the error, measured by the angular deviation between the looped transformer's output and the leading principal eigenspace, decays exponentially with the number of layers $L$ in both in-distribution and OOD settings.

\section{Numerical Experiments}
% In this section, we conduct numerical experiments to empirically validate our theoretical findings. 
In this section, we conduct numerical experiments to illustrate the training dynamics predicted by our theory and to provide finite-sample evidence for the mechanisms identified in the analysis. We first illustrate the training dynamics of the looped transformer \eqref{eq:looped transformer prediction}, and then highlight the critical role of LN by comparing normalized models \eqref{transformer prediction} and unnormalized models \eqref{unnormalized transformer prediction}. Due to space constraints, we defer visualizations of the training dynamics of one-layer transformers and the OOD performance of looped transformers to Appendix \ref{additional experiment results}.

\noindent\textbf{Experimental Setup.}
We employ the one-layer transformer architecture in \eqref{transformer prediction}, its unnormalized variant in \eqref{unnormalized transformer prediction}, and its looped variant in \eqref{eq:looped transformer prediction}. We consider two data-generation settings.

\noindent\textbf{Task 1.} The data matrices $\mathbf{X} \in \mathbb{R}^{d \times n}$ consist of columns uniformly distributed on the unit sphere $\mathbb{S}^{d-1}$. Specifically, each column $\mathbf{x}_j$ is generated by independently sampling $\mathbf{z}_j \sim \mathcal{N}(0, \mathbf{I}_d)$ from a standard multivariate Gaussian distribution and subsequently normalizing it: $\mathbf{x}_j = \mathbf{z}_j / \|\mathbf{z}_j\|_2$.

\noindent\textbf{Task 2.} Following Example \ref{first example}, we construct data matrices via $\mathbf{X} = \mathbf{R}\mathbf{U}$. Here, $\mathbf{R} \in \mathbb{R}^{d \times d}$ is an orthogonal matrix sampled uniformly from all orthogonal $d\times d$ matrices. The matrix $\mathbf{U} \in \mathbb{R}^{d \times n}$ is constructed using the standard basis vectors $\{\mathbf{e}_1, \dots, \mathbf{e}_d\}$ of $\mathbb{R}^d$. We assign a multiplicity $m_i>0$ to each basis vector $\mathbf{e}_i$ such that $\sum_{i=1}^d m_i = n$. The matrix $\mathbf{U}$ is explicitly formed as:
\begin{equation*}
    \mathbf{U} = \Big[ \underbrace{\mathbf{e}_1, \dots, \mathbf{e}_1}_{m_1 \text{ times}}, \underbrace{\mathbf{e}_2, \dots, \mathbf{e}_2}_{m_2 \text{ times}}, \dots, \underbrace{\mathbf{e}_d, \dots, \mathbf{e}_d}_{m_d \text{ times}} \Big].
\end{equation*}
Consequently, the eigenvalues of $\Ub\Ub^\top$ (and thus of $\Xb\Xb^\top$) are given by $\{m_1,m_2,\ldots,m_d\}$.

\noindent For both tasks, the model is trained on a dataset of size 10000 and an epoch number of 2000. We also independently sample a test dataset of size 2000 and report the test loss of the trained transformer layer. We use SGD with a learning rate of $\eta = 0.1$ and a batch size of $128$ to train the model. We use the same initialization for $\Wb$ and $\Vb$ as in Theorem \ref{thm:looped_end_to_end}.

% To show that our theoretical results hold in a less restrictive setting, we utilize SGD as our optimizer,
% \subsection{Convergence}
\begin{figure}[ht!]
\centering
\subfigure[ $\Wb/ \|\Wb \|_{\mathsf{max}}$, Task 1]{\includegraphics[width=0.3\textwidth]{
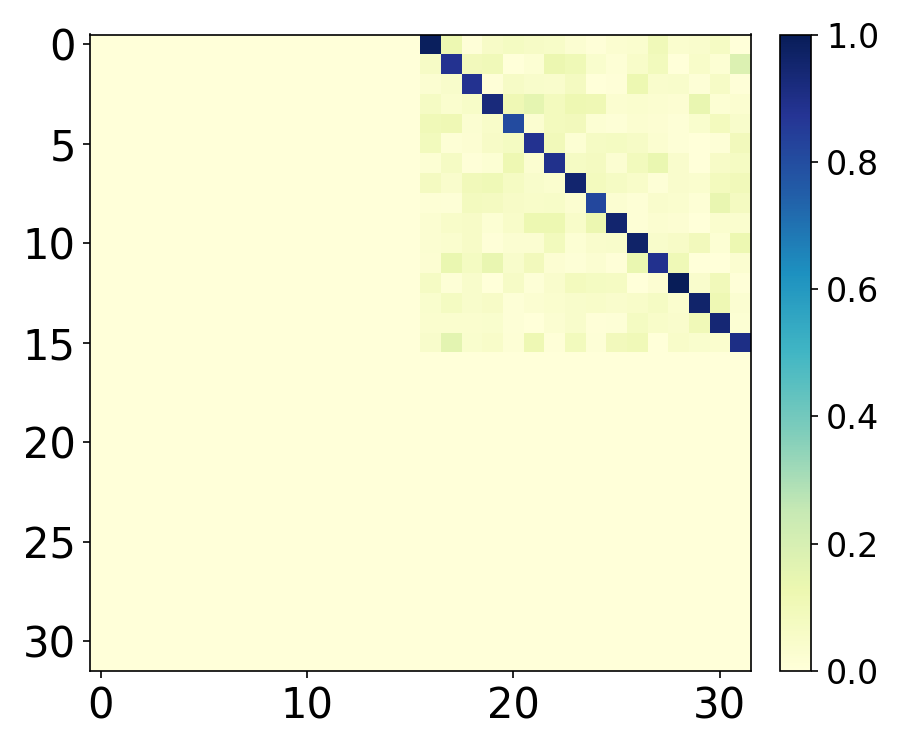}}
\subfigure[ $\Vb/ \|\Vb \|_{\mathsf{max}}$, Task 1]{\includegraphics[width=0.3\textwidth]{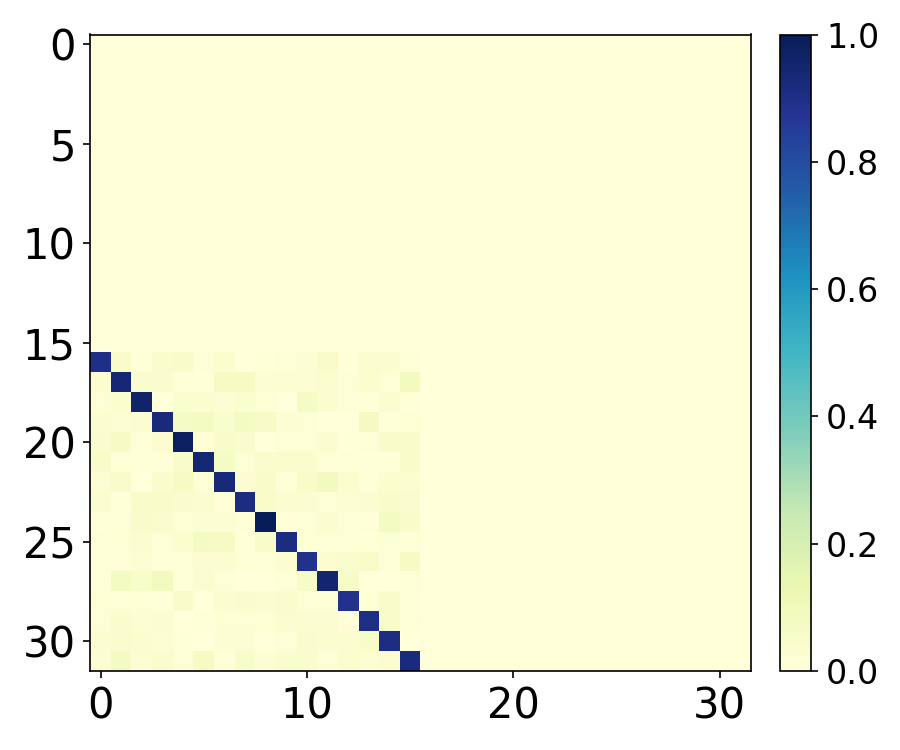}}
\subfigure[$\|\Wb \|_{\mathsf{max}},\|\Vb \|_{\mathsf{max}}$, Task 1]{\includegraphics[width=0.3\textwidth]{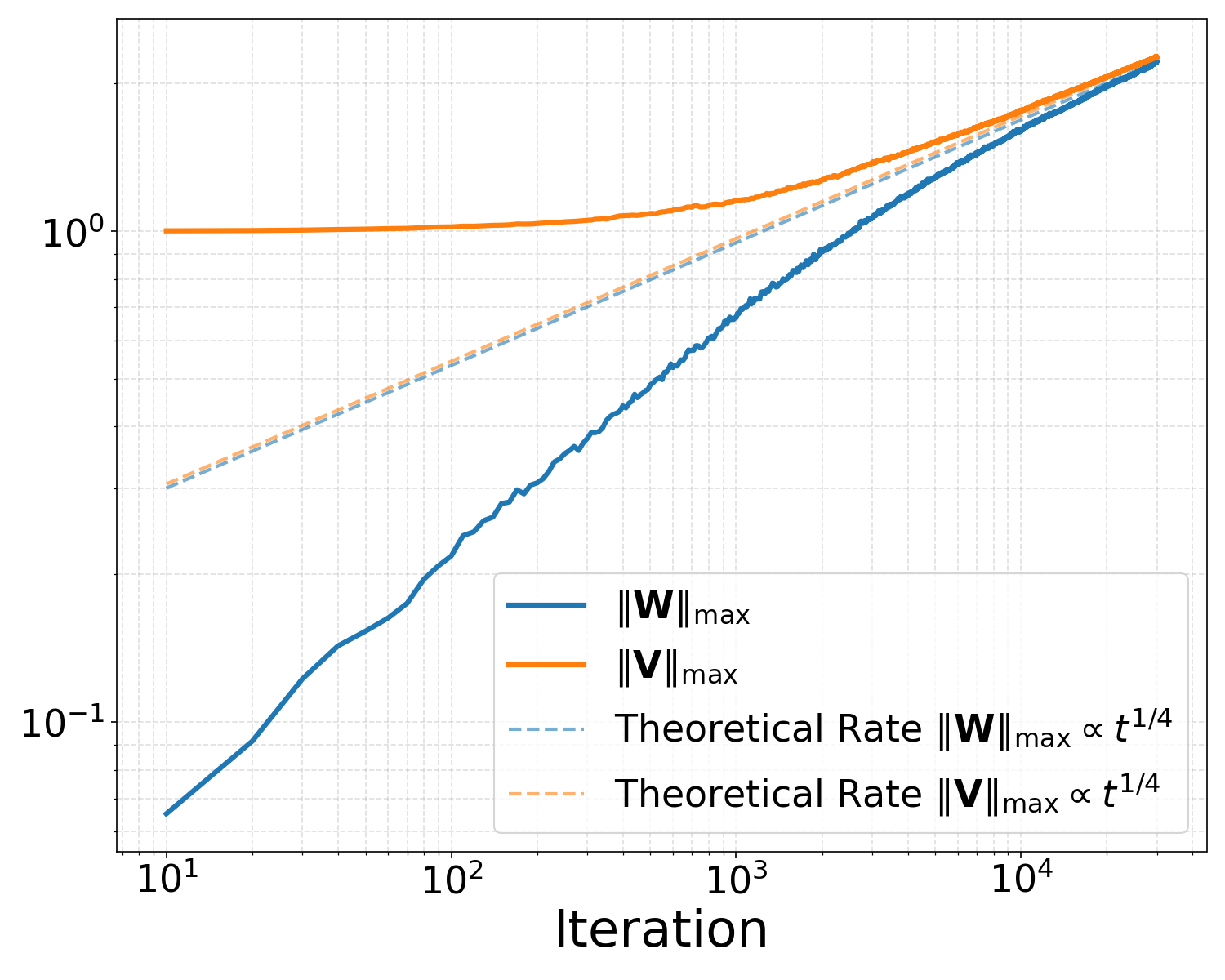}}\\
\subfigure[ $\Wb/ \|\Wb \|_{\mathsf{max}}$, Task 2]{\includegraphics[width=0.3\textwidth]{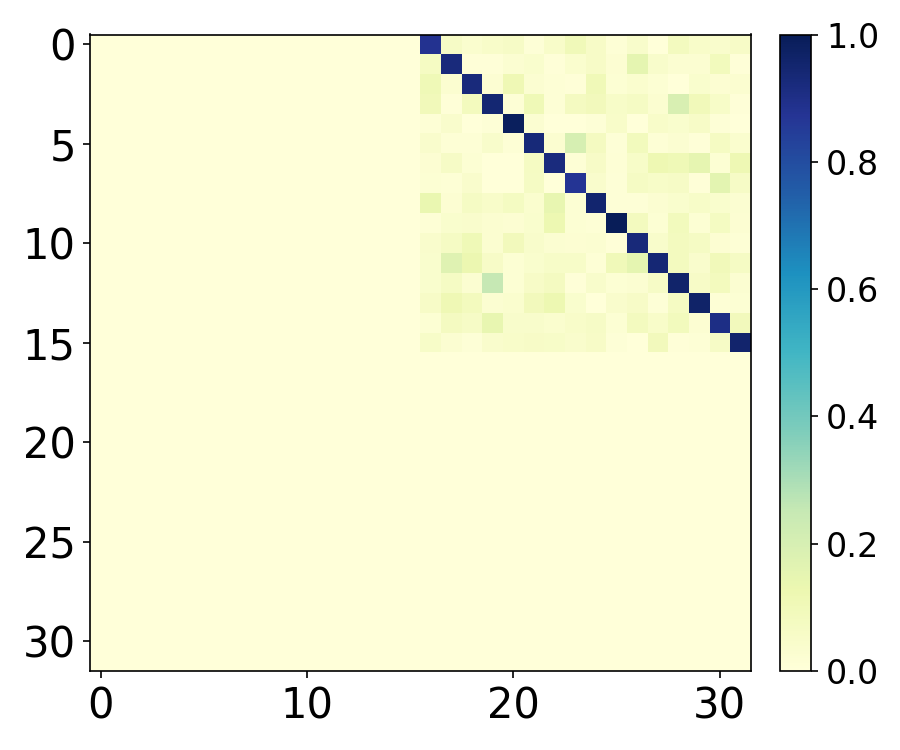}}
\subfigure[$\Vb/ \|\Vb \|_{\mathsf{max}}$, Task 2]{\includegraphics[width=0.3\textwidth]{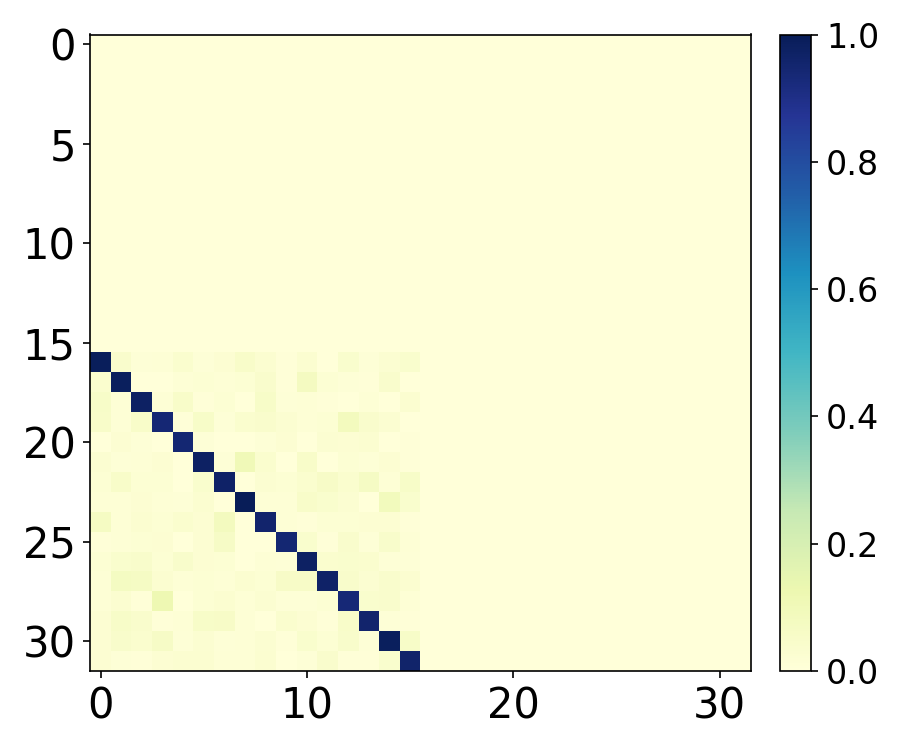}}
\subfigure[$\|\Wb \|_{\mathsf{max}},\|\Vb \|_{\mathsf{max}}$, Task 2]{\includegraphics[width=0.3\textwidth]{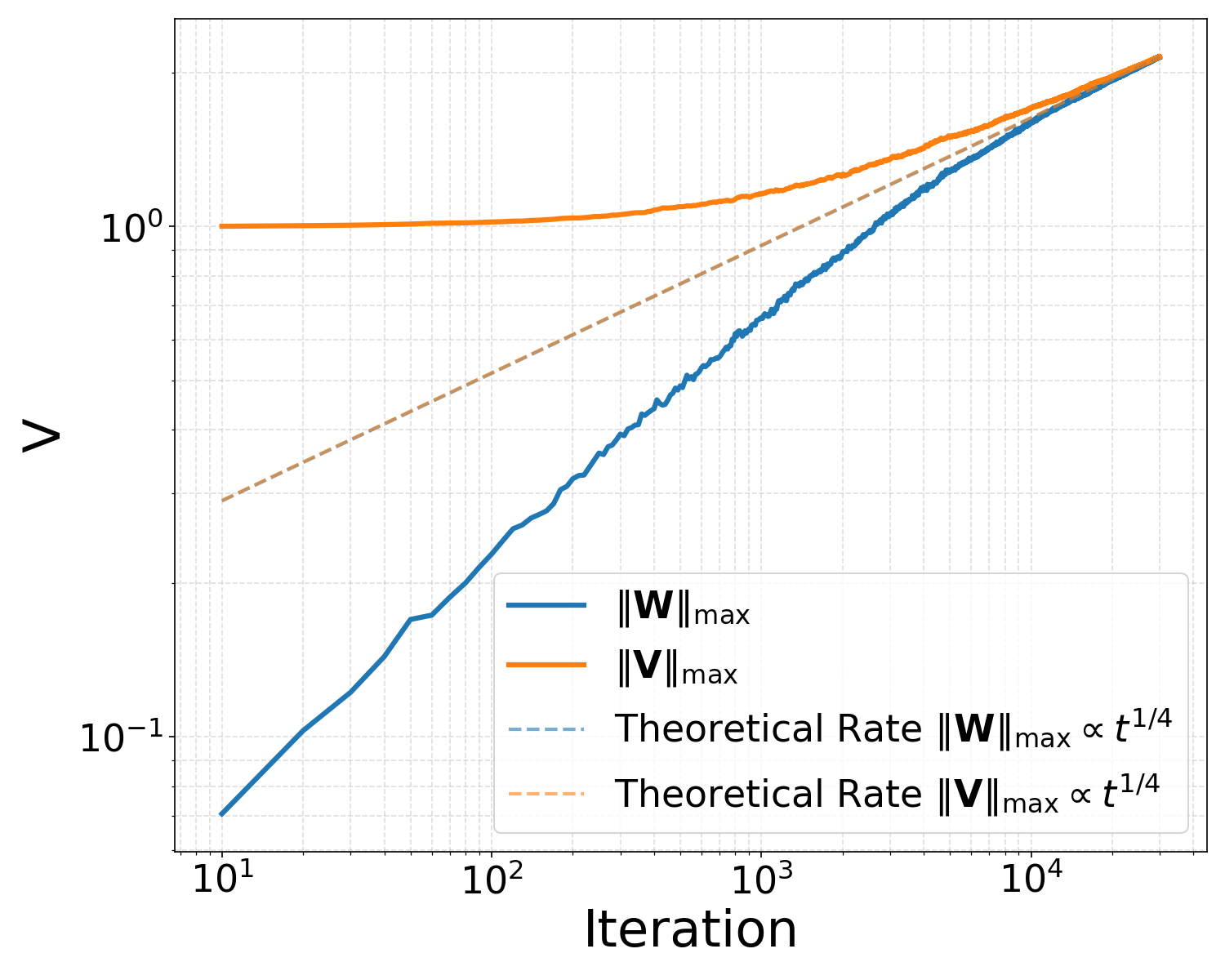}}
\caption{Heatmaps of parameter matrices and evolution of active parameters. Since the active parameters of both matrices diverge, we plot heatmaps to visualize the structures of $\Wb/ \|\Wb \|_{\mathsf{max}}$ and $\Vb/ \|\Vb \|_{\mathsf{max}}$. The dashed lines in the log-log plots represent the theoretical growth rate $\Theta(t^{1/4})$, vertically anchored to the final empirical data points.}
\label{fig:heatmaps}
\end{figure}

% \noindent\textbf{Convergence.} We first validate the main result of Theorem~\ref{thm:looped_end_to_end}, focusing in particular on the structure of the parameter matrices and the parameter growth rate. 

\noindent\textbf{Convergence.} We first examine whether the finite-sample SGD dynamics are qualitatively consistent with Theorem~\ref{thm:looped_end_to_end}, focusing on the structure of the parameter matrices and the parameter growth rate. We set $d=16$, $n=32$, and $L=10$, and plot the heatmaps of the matrices $\mathbf{V}$ and $\mathbf{W}$ after 2000 training iterations together with the evolution of the learned weights. As shown in Figure~\ref{fig:heatmaps}, all entries of both matrices become small except for the diagonal entries in the bottom-left block of $\mathbf{V}$ and the top-right block of $\mathbf{W}$, which is qualitatively consistent with the structure predicted by Theorem~\ref{thm:looped_end_to_end}. The active parameters also display a polynomial growth trend, and the log-log plots are consistent with the $\Theta(t^{1/4})$ scaling predicted by Theorem~\ref{thm:looped_end_to_end}.

% We set $d=16, n=32, L=10$ and plot the heatmaps of the matrices $\mathbf{V}$ and $\mathbf{W}$ after 2000 training iterations and the evolution of the learned weights. As shown in Figure~\ref{fig:heatmaps}, all entries of both matrices converge to near zero except for the elements on the diagonal in the bottom-left block of $\mathbf{V}$ and the top-right block of $\mathbf{W}$, aligning with the structure predicted by our analysis in Theorem \ref{thm:looped_end_to_end}. Furthermore, the evolution of these active parameters reveals a polynomial growth pattern. As shown in the log-log plots, the empirical trajectories asymptotically exhibit a slope that aligns with the theoretical $\Theta(t^{1/4})$ prediction, confirming the specific growth rate in Theorem \ref{thm:looped_end_to_end}.

\begin{figure}[ht!]
\centering
\subfigure[Training loss, Task 1]{\includegraphics[width=0.32\textwidth]{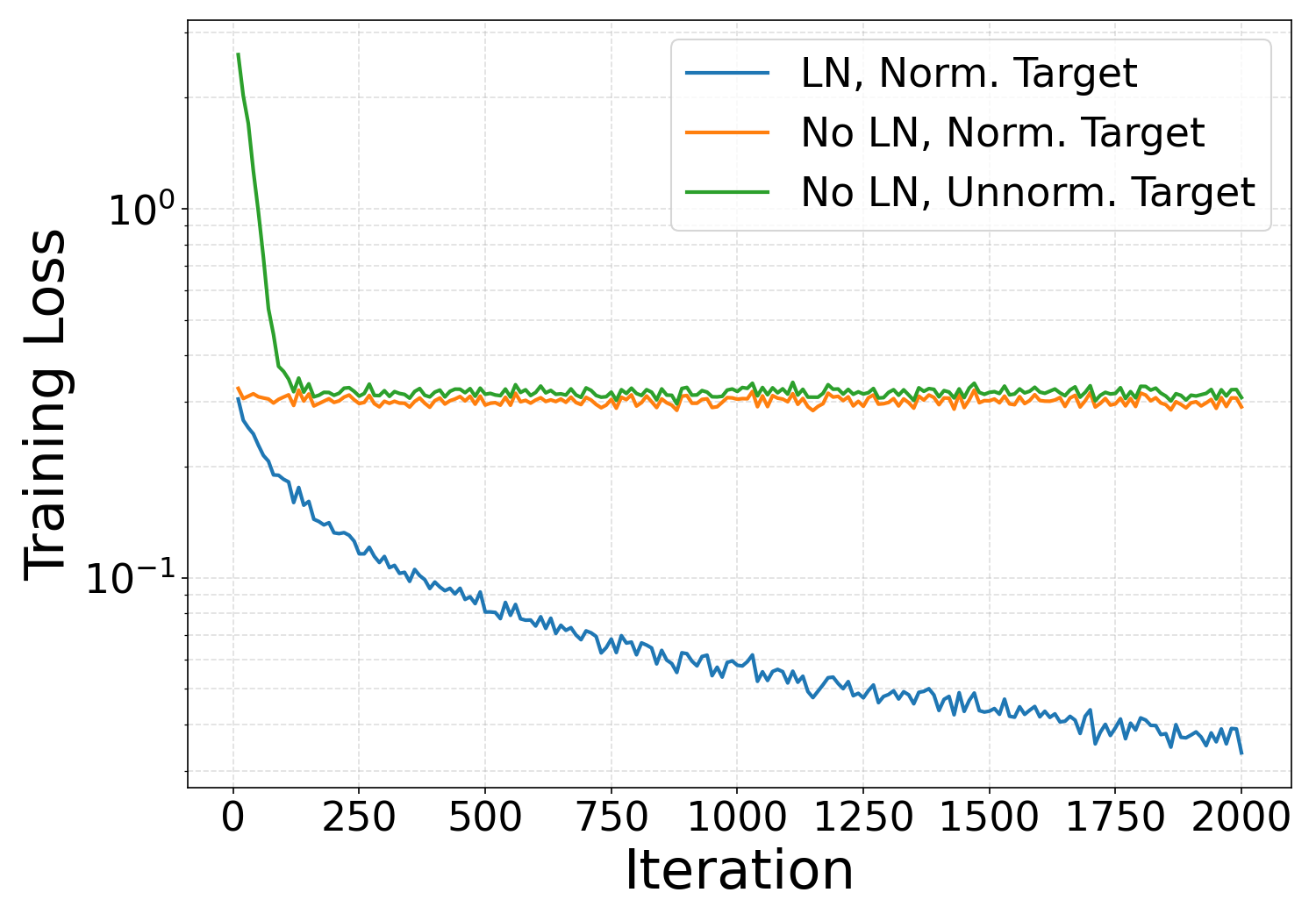}}
\subfigure[Test loss, Task 1]{\includegraphics[width=0.32\textwidth]{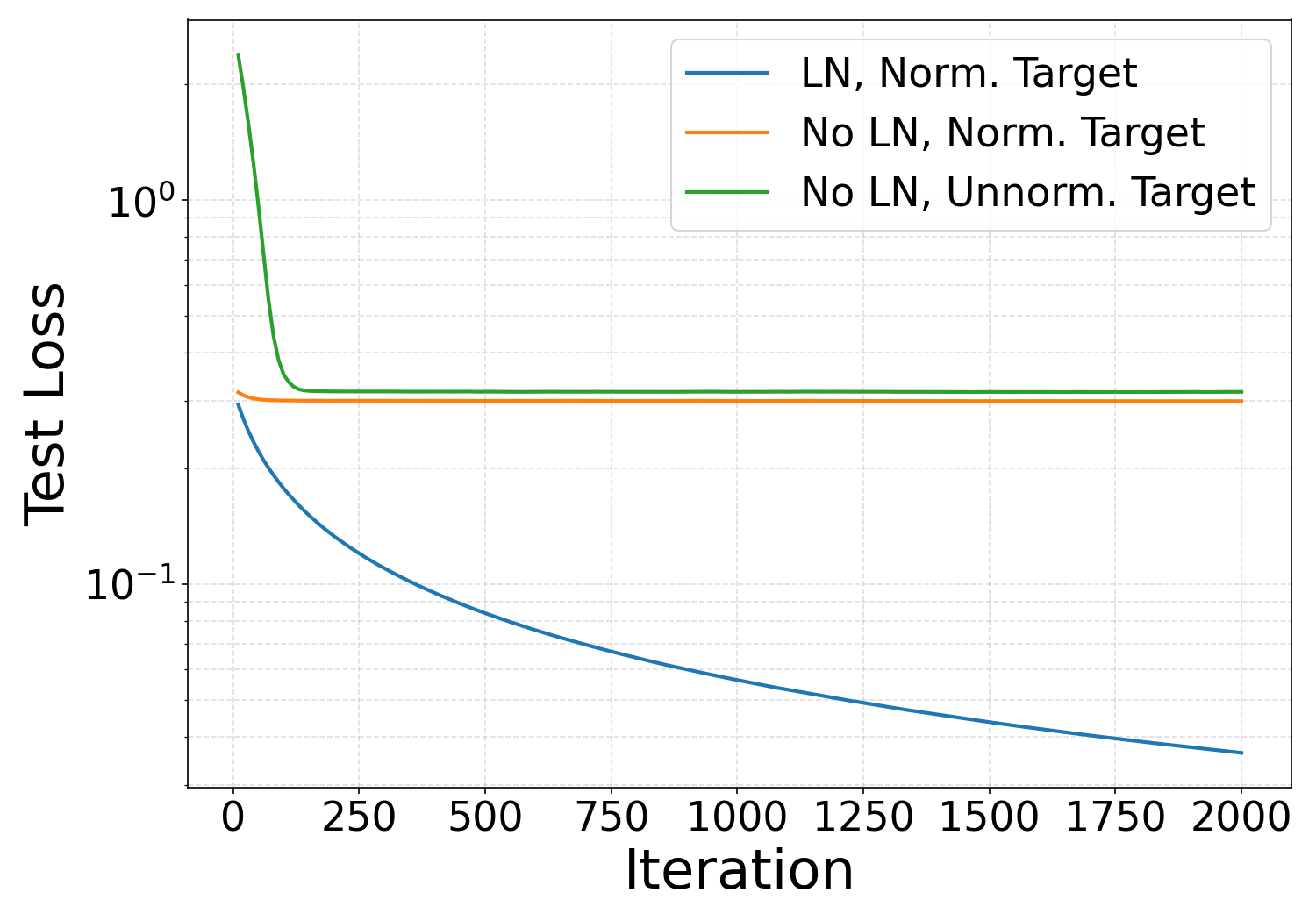}}
\subfigure[Looped transformer, Task 1]{\includegraphics[width=0.32\textwidth]{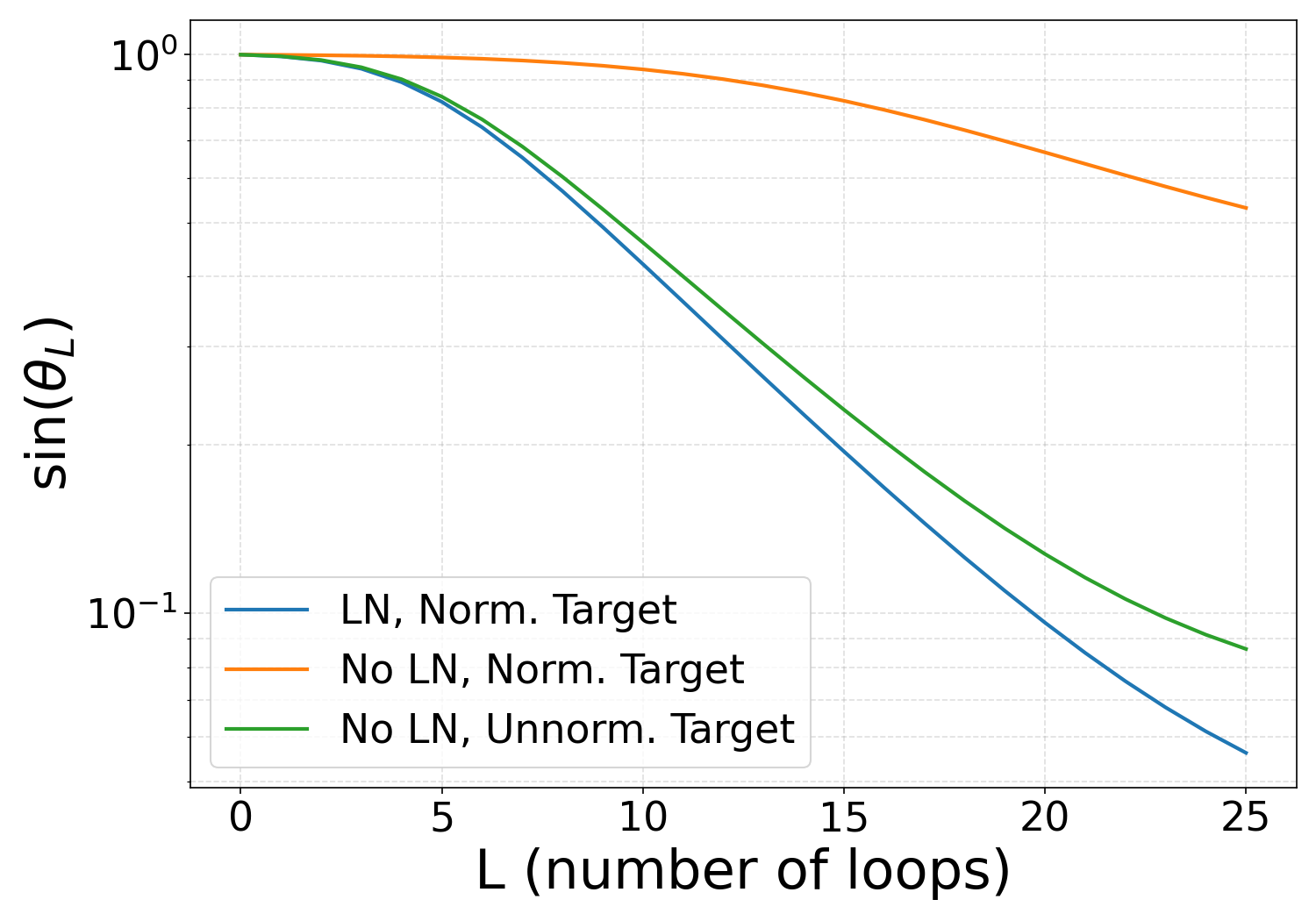}}\\
\subfigure[Training loss, Task 2]{\includegraphics[width=0.32\textwidth]{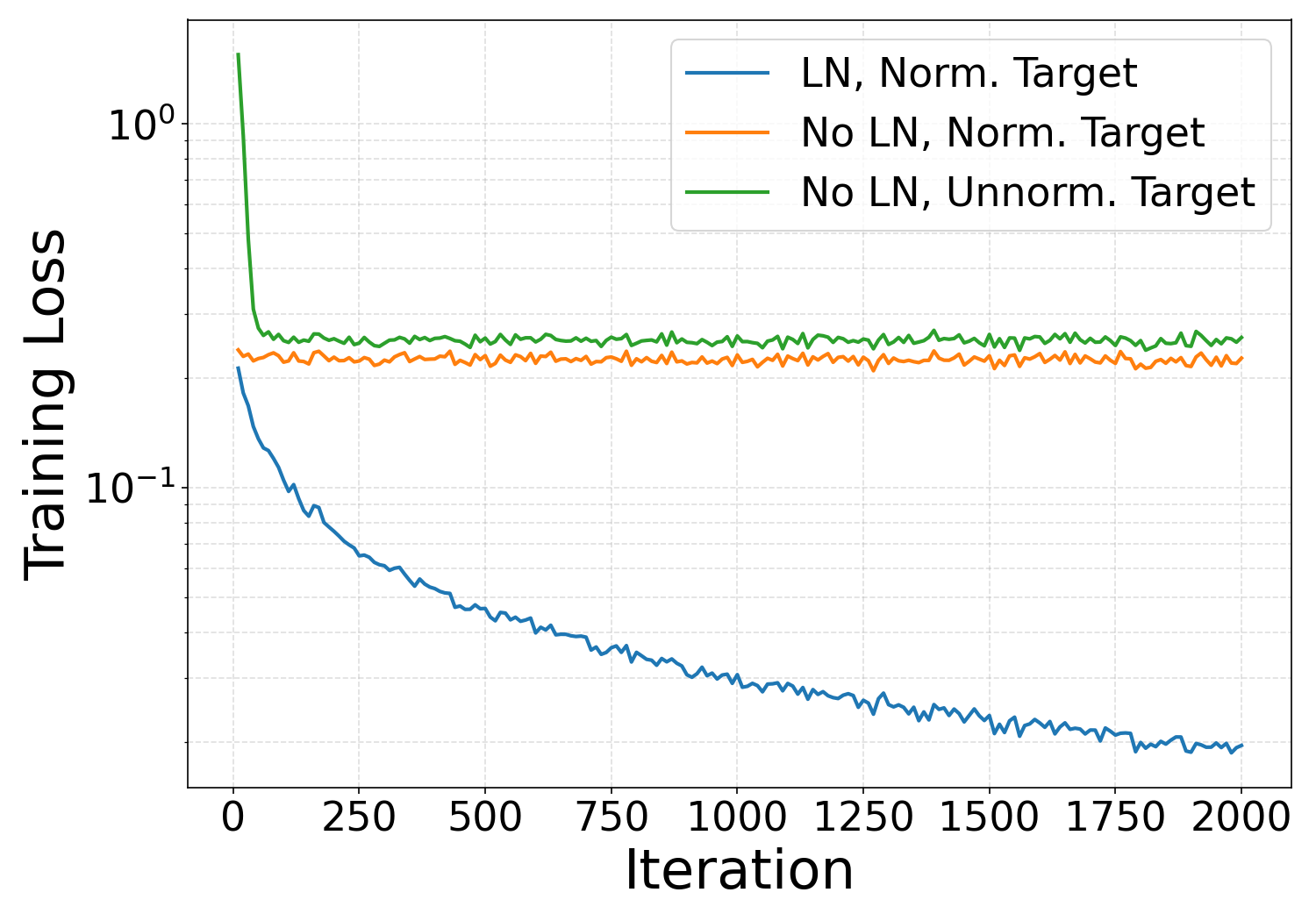}}
\subfigure[Test loss, Task 2]{\includegraphics[width=0.32\textwidth]{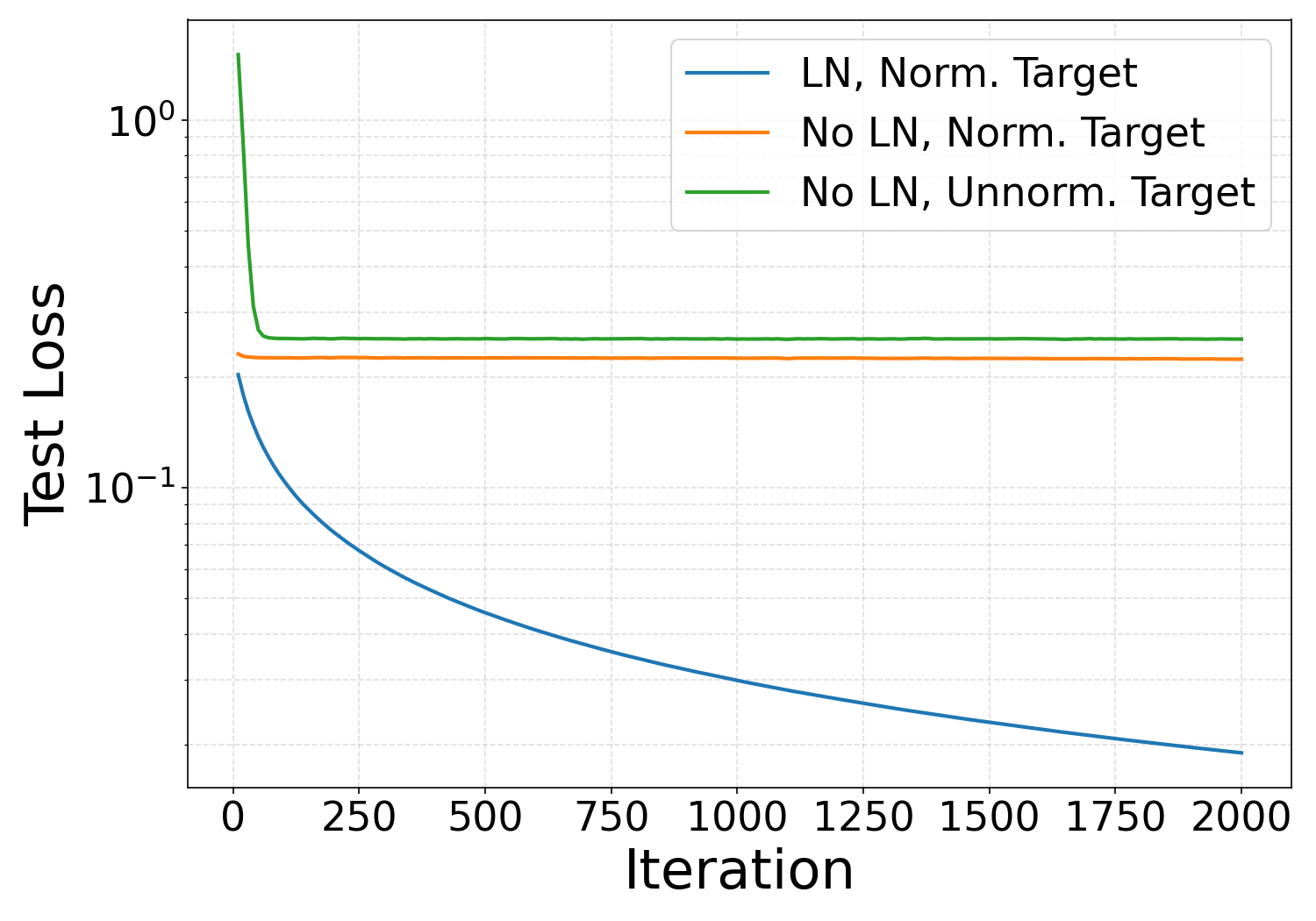}}
\subfigure[Looped transformer, Task 2]{\includegraphics[width=0.32\textwidth]{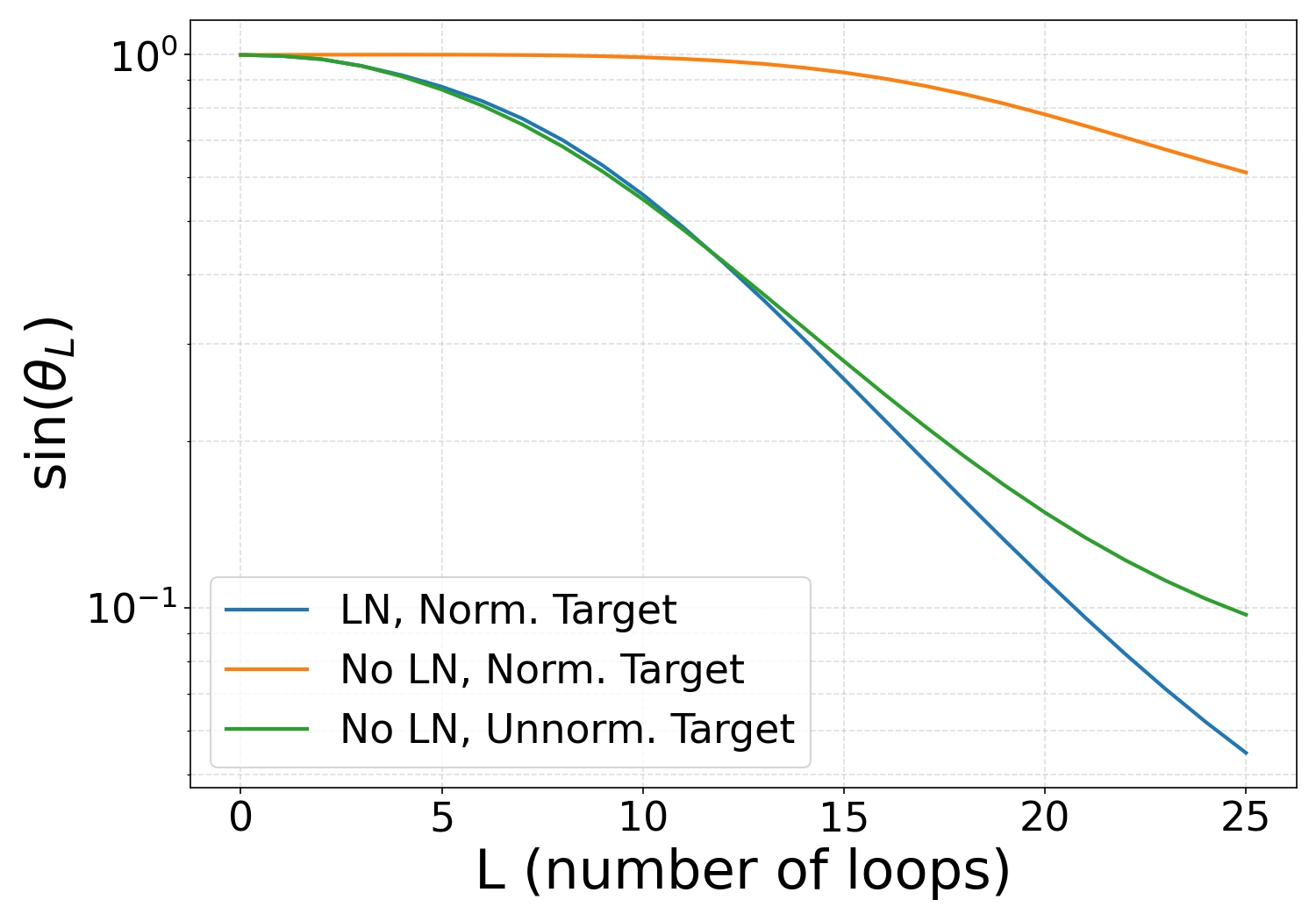}}
\caption{Semi-log plots of training loss, test loss, and the performance of the looped transformer for three models on tasks 1 and 2.}
\label{fig:LN-critical}
\end{figure}
% \subsection{The Critical Role of Layer Normalization}\label{unnormalized comparsion experiment}

\noindent\textbf{The Critical Role of Layer Normalization.} We also compare finite-sample SGD-trained models in order to test whether the qualitative separation predicted by Theorems~\ref{thm:unnorm convergence}, \ref{thm:convergence}, and~\ref{thm:looped error} is visible in practice. Specifically, we compare the following three models:
% \noindent\textbf{The Critical Role of Layer Normalization.} In this subsection, we empirically validate our theoretical claims in Theorems~\ref{thm:unnorm convergence}, \ref{thm:convergence} and~\ref{thm:looped error}, demonstrating that LN modifies training dynamics to facilitate improved algorithmic learning. Specifically, we compare the following three models: 
\begin{enumerate}[leftmargin = *, nosep]
    \item A transformer layer with LN \eqref{transformer prediction} trained via SGD on a normalized (norm.) target.
    \item A transformer layer without LN \eqref{unnormalized transformer prediction} trained via SGD on a normalized target $\yb_{\text{target}}^{(1)}$.
    \item A transformer layer without LN \eqref{unnormalized transformer prediction} trained via SGD on an unnormalized (unnorm.) target  $\yb_{\text{target}}^{(2)}$.
\end{enumerate}
% We set $d=16$, $n=32$. For evaluating the performance of the looped transformer, we take model snapshots at training step $T=2000$, sample a fixed test point $(\mathbf{X}_{\text{test}}, \mathbf{a}_{\text{test}})$, and iteratively feed the model output back as input to the transformer layer for $L\in{1,2,\ldots,25}$ iterations. Figure~\ref{fig:LN-critical} reports the training loss, test loss, and angular error under iterative inference for these three models. For the model with LN, the training and test losses converge to near zero. By contrast, both models without LN fail to converge effectively, with the training loss plateauing strictly above zero. This matches our conclusions in Theorem~\ref{thm:unnorm convergence} that the losses for models without LN will converge to a strictly positive limit. Regarding the performance of the looped transformer, Theorem~\ref{thm:looped error} predicts that the incomplete convergence of unnormalized models during training leads to a slower exponential decay of the angular error. The empirical curves in Figure~\ref{fig:LN-critical} indeed exhibit this slower convergence, thereby confirming the theorem.
We set $d=16$ and $n=32$. To evaluate the performance of the looped transformer, we take model snapshots at training step $T=2000$, sample a fixed test point $(\mathbf{X}_{\text{test}}, \mathbf{a}_{\text{test}})$, and iteratively feed the model output back as input to the transformer layer for $L\in\{1,2,\ldots,25\}$ iterations. Figure~\ref{fig:LN-critical} reports the training loss, test loss, and angular error under iterative inference for these three models. For the model with LN, the training and test losses converge to near zero. By contrast, both models without LN plateau at a strictly positive training loss, qualitatively consistent with the positive limiting loss predicted by Theorem~\ref{thm:unnorm convergence}. Regarding iterative inference, Theorem~\ref{thm:looped error} predicts a slower exponential decay of angular error for the unnormalized models, and the empirical curves in Figure~\ref{fig:LN-critical} show the same qualitative separation.

%\scriptsize

\section{Conclusions and Limitations}
This work studies principal component prediction as a concrete testbed for understanding the training dynamics of transformers with LN. Our analysis shows that a looped linear transformer with LN, trained by gradient descent, converges in the training limit to the $L$-step power method. We further provide a comparison between transformers with and without LN, showing that even when a transformer without LN is trained with layerwise guidance from power iterations, it still cannot exactly learn the power method, whereas the corresponding transformer with LN can. This difference leads to a concrete performance gap in predicting leading principal components. To the best of our knowledge, this paper provides the first theoretical analysis of transformer training dynamics that explicitly incorporates LN, highlighting its structural role in enabling algorithmic learning. Although our analysis focuses on a stylized setting with a simplified transformer architecture and population-loss training, it provides new theoretical insights into the role of LN. A natural next step is to extend the analysis to more realistic transformer models by incorporating components such as MLP layers, as well as to move beyond the population-loss setting toward finite-sample guarantees.

\appendix

\section{Proof of Theorem \ref{thm:looped_end_to_end}}\label{appendix:end-to-end training}

In this section, we prove Theorem~\ref{thm:looped_end_to_end}. The proof proceeds in four main steps:
\begin{enumerate}
    \item We first show that on the two-dimensional manifold
    \[
    \Wb_{12}=w\Ib_d,\qquad \Vb_{21}=v\Ib_d,
    \]
    with all other blocks equal to zero, the looped transformer reduces to the normalized power iteration
    \[
    \yb_L(\Eb;\btheta)
    =
    \frac{(\Ib_d+\rho\Xb\Xb^\top)^L\ab}
    {\|(\Ib_d+\rho\Xb\Xb^\top)^L\ab\|_2},
    \qquad
    \rho=wv.
    \]

    \item We then prove that this manifold is invariant under full-block population gradient descent. More precisely, all inactive blocks have zero population gradient, while the active-block gradients remain scalar multiples of the identity. Hence the full dynamics reduce exactly to scalar updates for \(w_t\) and \(v_t\).

    \item We analyze the resulting scalar recursion. The restricted population loss is a one-dimensional function \(R_L(\rho)\), and its derivative satisfies \(R_L'(\rho)<0\) and
    \[
    R_L'(\rho)
    =
    -\frac{\Upsilon_1}{\rho^2}
    +
    o(\rho^{-2}).
    \]
    This implies
    \[
    \rho_t=w_t v_t
    =
    \left(\frac{4\eta \Upsilon_1}{d}t\right)^{1/2}(1+o(1)),
    \qquad
    w_t,v_t
    =
    \left(\frac{4\eta \Upsilon_1}{d}t\right)^{1/4}(1+o(1)).
    \]

    \item Finally, we use the asymptotics of \(R_L\) to control the population loss and obtain
    \[
    \mathcal L_L(\btheta^{(t)})-\mathcal L_L^{(\infty)}
    =
    \left(\frac{d\Upsilon_1}{4\eta}\right)^{1/2}
    t^{-1/2}(1+o(1)).
    \]
\end{enumerate}

Throughout the proof, write
$
\Sb:=\Xb\Xb^\top
$, $
\boldsymbol{\Pi}_{\Xb}
:=
\sum_{i=1}^{r(\Xb)}
\vb_i(\Xb)\vb_i(\Xb)^\top .
$
Thus
\[
%\min_{\ub\in\operatorname{span}\{\vb_1(\Xb),\ldots,\vb_{r(\Xb)}(\Xb)\}}\|\yb-\ub\|_2^2
\left\|
\yb- \mathrm{Proj}_{\cV(\Xb)}(\yb)
\right\|_2^2
=
\|(\Ib_d-\boldsymbol{\Pi}_{\Xb})\yb\|_2^2.
\]
For \(w,v\in\RR\), define
\[
\Wb(w):=
\begin{pmatrix}
\mathbf 0&w\Ib_d\\
\mathbf 0&\mathbf 0
\end{pmatrix},
\qquad
\Vb(v):=
\begin{pmatrix}
\mathbf 0&\mathbf 0\\
v\Ib_d&\mathbf 0
\end{pmatrix},
\qquad
\rho=wv,
\]
and
\[
\mathcal M_+
:=
\left\{
(\Wb(w),\Vb(v)):w,v\in\RR,\ wv\ge0
\right\}.
\]
For \(\rho\ge0\), define the restricted scalar loss
\[
R_L(\rho)
:=
\EE_{\Xb,\ab}
\left[
\left\|
(\Ib_d-\boldsymbol{\Pi}_{\Xb})
\frac{(\Ib_d+\rho\Xb\Xb^\top)^L\ab}
{\|(\Ib_d+\rho\Xb\Xb^\top)^L\ab\|_2}
\right\|_2^2
\right].
\]

\begin{lemma}%[Forward dynamics on the scalar manifold] 
\label{lem:forward_scalar_manifold}
If \((\Wb,\Vb)=(\Wb(w),\Vb(v))\in\mathcal M_+\), then
\[
\yb_L(\Eb;\btheta)
=
\frac{(\Ib_d+\rho\Xb\Xb^\top)^L\ab}
{\|(\Ib_d+\rho\Xb\Xb^\top)^L\ab\|_2},
\qquad
\rho=wv.
\]
Consequently,
\[
\mathcal L_L(\Wb(w),\Vb(v))
=
R_L(wv).
\]
\end{lemma}

\begin{proof}
Suppose that at layer \(\ell\),
\[
\Eb^{(\ell)}
=
\begin{pmatrix}
\Xb&\mathbf 0\\
\mathbf 0&\ab_\ell
\end{pmatrix},
\qquad
\|\ab_\ell\|_2=1.
\]
For a data token \(\binom{\xb_j}{0}\),
\[
\Wb(w)\binom{\xb_j}{0}=0.
\]
Hence its attention update is zero, and since \(\|\xb_j\|_2=1\), column-wise normalization leaves it unchanged.

For the last token,
\[
\Wb(w)\binom{0}{\ab_\ell}
=
\binom{w\ab_\ell}{0}.
\]
Moreover,
\[
\Eb^{(\ell)}\Eb^{(\ell)\top}
=
\begin{pmatrix}
\Sb&\mathbf 0\\
\mathbf 0&\ab_\ell\ab_\ell^\top
\end{pmatrix},
\]
so
\[
\Eb^{(\ell)}\Eb^{(\ell)\top}
\Wb(w)\binom{0}{\ab_\ell}
=
\binom{w\Sb\ab_\ell}{0}.
\]
Multiplying by \(\Vb(v)\) gives
\[
\Vb(v)\Eb^{(\ell)}\Eb^{(\ell)\top}
\Wb(w)\binom{0}{\ab_\ell}
=
\binom{0}{wv\Sb\ab_\ell}
=
\binom{0}{\rho\Sb\ab_\ell}.
\]
Thus the last pre-normalization column is
\[
\binom{0}{(\Ib_d+\rho\Sb)\ab_\ell}.
\]
Since \(\rho\ge0\) and \(\Sb\succeq0\),
\[
\|(\Ib_d+\rho\Sb)\ab_\ell\|_2\ge \|\ab_\ell\|_2=1.
\]
Therefore
\[
\ab_{\ell+1}
=
\frac{(\Ib_d+\rho\Sb)\ab_\ell}
{\|(\Ib_d+\rho\Sb)\ab_\ell\|_2}.
\]
Induction over \(\ell=0,\ldots,L-1\) gives
\[
\yb_L(\Eb;\btheta)
=
\frac{(\Ib_d+\rho\Sb)^L\ab}
{\|(\Ib_d+\rho\Sb)^L\ab\|_2}.
\]
Substituting this expression into the squared-distance loss gives
\[
\mathcal L_L(\Wb(w),\Vb(v))
=
R_L(wv).
\]
\end{proof}

\begin{lemma}%[Inactive blocks have zero population gradient] 
\label{lem:inactive_zero_gradient}
Fix \((\Wb,\Vb)=(\Wb(w),\Vb(v))\in\mathcal M_+\). At this point,
\[
\nabla_{\Wb_{11}}\mathcal L_L
=
\nabla_{\Wb_{21}}\mathcal L_L
=
\nabla_{\Wb_{22}}\mathcal L_L
=
\mathbf 0,
\]
and
\[
\nabla_{\Vb_{11}}\mathcal L_L
=
\nabla_{\Vb_{12}}\mathcal L_L
=
\nabla_{\Vb_{22}}\mathcal L_L
=
\mathbf 0.
\]
\end{lemma}

\begin{proof}
Consider an arbitrary perturbation
\[
\dot{\Wb}
=
\begin{pmatrix}
\dot{\Wb}_{11}&\dot{\Wb}_{12}\\
\dot{\Wb}_{21}&\dot{\Wb}_{22}
\end{pmatrix},
\qquad
\dot{\Vb}
=
\begin{pmatrix}
\dot{\Vb}_{11}&\dot{\Vb}_{12}\\
\dot{\Vb}_{21}&\dot{\Vb}_{22}
\end{pmatrix}
\]
satisfying
\[
\dot{\Wb}_{12}=\mathbf 0,
\qquad
\dot{\Vb}_{21}=\mathbf 0.
\]
It suffices to show that the first variation of \(\yb_L\) is zero.

Let
\[
\Wb_\varepsilon=\Wb+\varepsilon\dot{\Wb},
\qquad
\Vb_\varepsilon=\Vb+\varepsilon\dot{\Vb},
\]
and let \(\Eb_\varepsilon^{(\ell)}\) be the corresponding iterates. Denote the baseline iterates by
\[
\bar{\Eb}^{(\ell)}
=
\begin{pmatrix}
\Xb&\mathbf 0\\
\mathbf 0&\ab_\ell
\end{pmatrix},
\]
and write the first variation as
\[
\dot{\Eb}^{(\ell)}
=
\left.\frac{d}{d\varepsilon}\Eb_\varepsilon^{(\ell)}\right|_{\varepsilon=0}
=
\begin{pmatrix}
\Pb_\ell&\pb_\ell\\
\Qb_\ell&\qb_\ell
\end{pmatrix}.
\]
We prove by induction that
\[
\Pb_\ell=\mathbf 0,
\qquad
\qb_\ell=\mathbf 0
\]
for all \(\ell\). This is true at \(\ell=0\), since \(\Eb^{(0)}\) is independent of \(\Wb,\Vb\).

Let
\[
\Ub(\Eb;\Wb,\Vb)
=
\Eb+\Vb\Eb\Eb^\top\Wb\Eb .
\]
At the baseline point,
\[
\dot{\Ub}
=
\dot{\Eb}
+
\dot{\Vb}\bar{\Eb}\bar{\Eb}^{\top}\Wb\bar{\Eb}
+
\Vb(\dot{\Eb}\bar{\Eb}^{\top}+\bar{\Eb}\dot{\Eb}^{\top})\Wb\bar{\Eb}
+
\Vb\bar{\Eb}\bar{\Eb}^{\top}\dot{\Wb}\bar{\Eb}
+
\Vb\bar{\Eb}\bar{\Eb}^{\top}\Wb\dot{\Eb}.
\]

For a data token \(\bar{\eb}_j=\binom{\xb_j}{0}\), we have
\[
\Wb\bar{\eb}_j=0.
\]
Thus the terms
\[
\dot{\Vb}\bar{\Eb}\bar{\Eb}^{\top}\Wb\bar{\eb}_j
\qquad\text{and}\qquad
\Vb(\dot{\Eb}\bar{\Eb}^{\top}+\bar{\Eb}\dot{\Eb}^{\top})\Wb\bar{\eb}_j
\]
vanish. The remaining terms multiplied by \(\Vb\) have zero top block because the top block row of \(\Vb\) is zero. Hence
\[
[\dot{\Ub}_j]_{\rm top}=(\Pb_\ell)_{:,j}.
\]
Since the baseline pre-normalization data token is \(\binom{\xb_j}{0}\), the derivative of column-wise normalization gives
\[
(\Pb_{\ell+1})_{:,j}
=
(\Ib_d-\xb_j\xb_j^\top)(\Pb_\ell)_{:,j}.
\]
By the induction hypothesis, this is zero. Hence
\[
\Pb_{\ell+1}=\mathbf 0.
\]

For the last token \(\bar{\eb}_q=\binom{0}{\ab_\ell}\), we have
\[
\Wb\bar{\eb}_q=\binom{w\ab_\ell}{0},
\qquad
\bar{\Eb}\bar{\Eb}^{\top}\Wb\bar{\eb}_q
=
\binom{w\Sb\ab_\ell}{0}.
\]
The bottom block of the pre-normalization first variation is
\[
[\dot{\Ub}_q]_{\rm bottom}
=
\qb_\ell
+
w\dot{\Vb}_{21}\Sb\ab_\ell
+
wv(\Pb_\ell\Xb^\top+\Xb\Pb_\ell^\top)\ab_\ell
+
v\Sb\dot{\Wb}_{12}\ab_\ell
+
wv\Sb\qb_\ell.
\]
Since \(\dot{\Vb}_{21}=0\), \(\dot{\Wb}_{12}=0\), and by induction
\[
\Pb_\ell=0,\qquad \qb_\ell=0,
\]
we obtain
\[
[\dot{\Ub}_q]_{\rm bottom}=0.
\]
The baseline last pre-normalization column is
\[
\binom{0}{(\Ib_d+\rho\Sb)\ab_\ell},
\]
whose bottom block has norm at least \(1\). Therefore the derivative of column-wise normalization gives
\[
\qb_{\ell+1}=0.
\]
This completes the induction. Hence
\[
\dot{\yb}_L=\qb_L=0.
\]

Therefore the sample directional derivative of
\[
\ell_{\Xb}(\yb)
=
\|(\Ib_d-\boldsymbol{\Pi}_{\Xb})\yb\|_2^2
\]
is zero for every inactive perturbation.

It remains to justify differentiating under the expectation. Since the sample space
\[
(\SSS^{d-1})^n\times\SSS^{d-1}
\]
is compact, and all baseline pre-normalization columns have norm at least \(1\), there exists a deterministic \(\varepsilon_0>0\) such that, for all \(|\varepsilon|\le\varepsilon_0\), all pre-normalization columns along the perturbed trajectory have norm at least \(1/2\), uniformly over all samples. Therefore all layer-normalization derivatives are uniformly bounded in this neighborhood.

Moreover,
\[
\nabla_\yb \ell_{\Xb}(\yb)
=
2(\Ib_d-\boldsymbol{\Pi}_{\Xb})\yb,
\qquad
\|\boldsymbol{\Pi}_{\Xb}\|_2\le1.
\]
Since every output column after layer normalization has norm \(1\), the bottom prediction has norm at most \(1\). A layer-by-layer chain-rule bound therefore gives a deterministic bound on the directional derivative of the sample loss, depending only on
\[
d,n,L,w,v,\dot{\Wb},\dot{\Vb}.
\]
Dominated convergence allows differentiation under the expectation. Hence the population directional derivative is zero in every inactive direction, proving the claim.
\end{proof}

\begin{lemma}%[Active gradients reduce to scalar gradients] 
\label{lem:active_scalar_gradient}
At every point \((\Wb(w),\Vb(v))\in\mathcal M_+\), there exist scalars
\(c_A,c_B\) such that
\[
\nabla_{\Wb}\mathcal L_L
=
\begin{pmatrix}
\mathbf 0&c_A\Ib_d\\
\mathbf 0&\mathbf 0
\end{pmatrix},
\qquad
\nabla_{\Vb}\mathcal L_L
=
\begin{pmatrix}
\mathbf 0&\mathbf 0\\
c_B\Ib_d&\mathbf 0
\end{pmatrix}.
\]
Moreover,
\[
c_A=\frac{v}{d}R_L'(wv),
\qquad
c_B=\frac{w}{d}R_L'(wv).
\]
\end{lemma}

\begin{proof}
By Lemma~\ref{lem:inactive_zero_gradient}, all inactive block gradients vanish. It remains to treat the active blocks. Write
\[
\Ab=\Wb_{12},
\qquad
\Bb=\Vb_{21}.
\]
Let \(\mathcal F(\Ab,\Bb)\) denote the population loss restricted to the subspace where all inactive blocks vanish. In a sufficiently small neighborhood of
\[
(\Ab,\Bb)=(w\Ib_d,v\Ib_d),
\]
all layer-normalization denominators are uniformly bounded away from zero, as in Lemma~\ref{lem:inactive_zero_gradient}; hence \(\mathcal F\) is differentiable there.

For any orthogonal matrix \(\Qb\), define
\[
\Tb_\Qb=
\begin{pmatrix}
\Qb&\mathbf 0\\
\mathbf 0&\Qb
\end{pmatrix}.
\]
If
\[
\Xb'=\Qb\Xb,
\qquad
\ab'=\Qb\ab,
\]
then
\[
\Xb'\Xb'^\top
=
\Qb\Xb\Xb^\top\Qb^\top,
\]
and the principal eigenspace is rotated by \(\Qb\). Hence
\[
\boldsymbol{\Pi}_{\Xb'}=\Qb\boldsymbol{\Pi}_{\Xb}\Qb^\top.
\]
The transformed active blocks are
\[
\Ab'=\Qb\Ab\Qb^\top,
\qquad
\Bb'=\Qb\Bb\Qb^\top.
\]
By induction over layers,
\[
\Eb'^{(\ell)}=\Tb_\Qb\Eb^{(\ell)}
\]
for all \(\ell\), because the pre-normalization update is equivariant under \(\Tb_\Qb\) and column-wise normalization commutes with orthogonal transformations. Therefore
\[
\yb_L(\Eb';\Ab',\Bb')
=
\Qb\yb_L(\Eb;\Ab,\Bb).
\]
Consequently,
\[
\|(\Ib_d-\boldsymbol{\Pi}_{\Xb'})
\yb_L(\Eb';\Ab',\Bb')\|_2^2
=
\|(\Ib_d-\boldsymbol{\Pi}_{\Xb})
\yb_L(\Eb;\Ab,\Bb)\|_2^2.
\]
Since \((\Xb,\ab)\) and \((\Qb\Xb,\Qb\ab)\) have the same distribution, we get
\[
\mathcal F(\Ab,\Bb)
=
\mathcal F(\Qb\Ab\Qb^\top,\Qb\Bb\Qb^\top).
\]

Evaluate at
\[
\Ab=w\Ib_d,
\qquad
\Bb=v\Ib_d.
\]
If
\[
\Gb_A=\nabla_{\Ab}\mathcal F(w\Ib_d,v\Ib_d),
\]
then for every \(\Hb\),
\[
\langle \Gb_A,\Hb\rangle_F
=
\langle \Gb_A,\Qb\Hb\Qb^\top\rangle_F
=
\langle \Qb^\top\Gb_A\Qb,\Hb\rangle_F.
\]
Thus
\[
\Gb_A=\Qb^\top\Gb_A\Qb
\qquad
\text{for every orthogonal }\Qb,
\]
by Lemma \ref{Schur's lemma} we have
\[
\Gb_A=c_A\Ib_d
\]
for some scalar \(c_A\). Similarly,
\[
\nabla_{\Bb}\mathcal F(w\Ib_d,v\Ib_d)
=
c_B\Ib_d
\]
for some scalar \(c_B\).

By Lemma~\ref{lem:forward_scalar_manifold},
\[
\mathcal L_L(\Wb(w),\Vb(v))
=
R_L(wv).
\]
Therefore
\[
\frac{\partial}{\partial w}\mathcal L_L(\Wb(w),\Vb(v))
=
v R_L'(wv),
\]
while also
\[
\frac{\partial}{\partial w}\mathcal L_L(\Wb(w),\Vb(v))
=
\langle c_A\Ib_d,\Ib_d\rangle_F
=
dc_A.
\]
Hence
\[
c_A=\frac{v}{d}R_L'(wv).
\]
The same argument gives
\[
c_B=\frac{w}{d}R_L'(wv).
\]
Combining this with Lemma~\ref{lem:inactive_zero_gradient} proves the claim.
\end{proof}

\begin{lemma}%[Derivative of the scalar objective]
\label{lem:scalar_derivative}
Let \(\Sb=\sum_{i=1}^d\lambda_i\vb_i\vb_i^\top\), 
\(\ab=\sum_{i=1}^d c_i\vb_i\), and \(r=r(\Xb)\). 
For \(\rho\ge0\), define
\(p_i(\rho):=c_i^2(1+\rho\lambda_i)^{2L}/\sum_{k=1}^d c_k^2(1+\rho\lambda_k)^{2L}\). 
Then the sample scalar loss is \(\ell_L(\rho)=1-\sum_{i\le r}p_i(\rho)\), and
\[
\ell_L'(\rho)
=
-2L\sum_{i\le r}\sum_{j>r}
p_i(\rho)p_j(\rho)
\frac{\lambda_1-\lambda_j}
{(1+\rho\lambda_1)(1+\rho\lambda_j)}.
\]
Consequently, \(R_L\in C^1([0,\infty))\), \(R_L'(\rho)<0\) for all \(\rho\ge0\), and 
\(\sup_{\rho\ge0}|R_L'(\rho)|\le 2Ln\). 
Moreover, under Assumption~\ref{assump:distribution}(A3),
\[
R_L'(\rho)=-\frac{\Upsilon_1}{\rho^2}+o(\rho^{-2})
\qquad \text{as } \rho\to\infty,
\]
where \(0<\Upsilon_1<\infty\) is given by
\[
\Upsilon_1
:=
2L\,\EE_{\Xb,\ab}
\left[
\sum_{i\le r(\Xb)}\sum_{j>r(\Xb)}
p_i^{(\infty)}p_j^{(\infty)}
\frac{\lambda_1(\Xb)-\lambda_j(\Xb)}
{\lambda_1(\Xb)\lambda_j(\Xb)}
\right],
\qquad
p_i^{(\infty)}
:=
\frac{c_i^2\lambda_i(\Xb)^{2L}}
{\sum_{k=1}^d c_k^2\lambda_k(\Xb)^{2L}}.
\]
\end{lemma}

\begin{proof}
On \(\mathcal M_+\),
\[
\yb_L
=
\frac{(\Ib_d+\rho\Sb)^L\ab}
{\|(\Ib_d+\rho\Sb)^L\ab\|_2}.
\]
Since
\[
(\Ib_d+\rho\Sb)^L\ab
=
\sum_{i=1}^d c_i(1+\rho\lambda_i)^L\vb_i,
\]
the squared distance to the principal eigenspace is
\[
\ell_L(\rho)
=
1-
\frac{
\sum_{i\le r}c_i^2(1+\rho\lambda_i)^{2L}
}{
\sum_{k=1}^d c_k^2(1+\rho\lambda_k)^{2L}
}
=
1-\sum_{i\le r}p_i(\rho).
\]
The quotient rule gives
\[
p_i'(\rho)
=
p_i(\rho)
\left(
\frac{2L\lambda_i}{1+\rho\lambda_i}
-
\sum_{k=1}^d p_k(\rho)
\frac{2L\lambda_k}{1+\rho\lambda_k}
\right).
\]
Since \(\lambda_i=\lambda_1\) for \(i\le r\),
\[
\ell_L'(\rho)
=
-\sum_{i\le r}p_i'(\rho)
=
-2L
\sum_{i\le r}\sum_{j>r}
p_i(\rho)p_j(\rho)
\frac{\lambda_1-\lambda_j}
{(1+\rho\lambda_1)(1+\rho\lambda_j)}.
\]
Because \(r<d\), \(\lambda_1>\lambda_j\) for \(j>r\), and a uniform \(\ab\) has nonzero projections onto both the principal eigenspace and its orthogonal complement almost surely, we have
\[
\ell_L'(\rho)<0
\]
almost surely. The bound
\[
|\ell_L'(\rho)|\le 2Ln
\]
proved below is uniform in \(\rho\), so dominated convergence permits differentiation under the expectation and yields
\[
R_L'(\rho)=\EE_{\Xb,\ab}[\ell_L'(\rho)]<0.
\]
The same bound and pointwise continuity of \(\ell_L'(\rho)\) imply, again by dominated convergence, that \(R_L'\) is continuous on \([0,\infty)\).

It remains to prove the uniform bound. Since
\[
\frac{\lambda_1-\lambda_j}
{(1+\rho\lambda_1)(1+\rho\lambda_j)}
\le \lambda_1\le \operatorname{tr}(\Sb)=n,
\]
and
\[
\sum_{i\le r}\sum_{j>r}p_i(\rho)p_j(\rho)\le1,
\]
we obtain
\[
|R_L'(\rho)|\le2Ln.
\]

For the asymptotics, as \(\rho\to\infty\),
\[
p_i(\rho)\to p_i^{(\infty)}
\]
and
\[
\rho^2
\frac{\lambda_1-\lambda_j}
{(1+\rho\lambda_1)(1+\rho\lambda_j)}
\to
\frac{\lambda_1-\lambda_j}{\lambda_1\lambda_j}.
\]
Moreover,
\[
0\le
\rho^2
\frac{\lambda_1-\lambda_j}
{(1+\rho\lambda_1)(1+\rho\lambda_j)}
\le
\frac{\lambda_1-\lambda_j}{\lambda_1\lambda_j}
\le
\frac1{\lambda_j}
\le
\frac1{\lambda_d}.
\]
Since
\[
\sum_{i\le r}\sum_{j>r}p_i(\rho)p_j(\rho)\le1
\]
and by Assumption~\ref{assump:distribution}(A3) \(\EE[\lambda_d(\Xb)^{-1}]<\infty\), dominated convergence gives
\[
\rho^2R_L'(\rho)\to -\Upsilon_1.
\]
The same domination gives \(\Upsilon_1<\infty\). Positivity follows from \(r(\Xb)<d\), the strict spectral gap between \(\lambda_1\) and non-principal eigenvalues, and the almost sure nonzero projections of \(\ab\) onto both subspaces.
\end{proof}

\begin{lemma}%[Scalar gradient dynamics]
\label{lem:scalar_dynamics}
Let \(w_0=0\), \(v_0=1\), and suppose
\[
w_{t+1}
=
w_t-\frac{\eta}{d}v_tR_L'(\rho_t),
\qquad
v_{t+1}
=
v_t-\frac{\eta}{d}w_tR_L'(\rho_t),
\qquad
\rho_t=w_t v_t.
\]
If \(\eta\le d/(2Ln)\), then
\[
w_t\ge0,\qquad v_t\ge1,\qquad \rho_t\uparrow\infty,
\qquad
\frac{w_t}{v_t}\to1.
\]
Moreover,
\[
\rho_t
=
\left(\frac{4\eta \Upsilon_1}{d}t\right)^{1/2}(1+o(1)),
\qquad
w_t,v_t
=
\left(\frac{4\eta \Upsilon_1}{d}t\right)^{1/4}(1+o(1)).
\]
\end{lemma}

\begin{proof}
Let
\[
\kappa=\frac{\eta}{d},
\qquad
G_t=R_L'(\rho_t).
\]
By Lemma~\ref{lem:scalar_derivative},
\[
G_t<0.
\]
Since
\[
w_0=0,\qquad v_0=1,
\]
induction gives
\[
w_t\ge0,
\qquad
v_t\ge1,
\qquad
\rho_t\ge0.
\]
Expanding
\[
\rho_{t+1}
=
(w_t-\kappa v_tG_t)(v_t-\kappa w_tG_t),
\]
we obtain
\[
\rho_{t+1}-\rho_t
=
-\kappa G_t(w_t^2+v_t^2)
+
\kappa^2\rho_tG_t^2>0.
\]
Hence \(\rho_t\) is increasing. If it were bounded by \(B\), continuity and strict negativity of \(R_L'\) on \([0,B]\) would give \(c_B>0\) such that
\[
-R_L'(\rho)\ge c_B
\qquad
\text{on }[0,B].
\]
Then
\[
\rho_{t+1}-\rho_t
\ge
\kappa c_B(w_t^2+v_t^2)
\ge
\kappa c_B,
\]
contradicting boundedness. Thus
\[
\rho_t\to\infty.
\]

Next define
\[
\Delta_t=v_t^2-w_t^2.
\]
Then
\[
\Delta_{t+1}
=
\Delta_t(1-\kappa^2G_t^2).
\]
Since \(\eta\le d/(2Ln)\) and
\[
\sup_{\rho\ge0}|R_L'(\rho)|\le2Ln,
\]
we have
\[
0\le\kappa|G_t|\le1.
\]
Thus
\[
0\le\Delta_t\le\Delta_0=1.
\]
Since \(\rho_t=w_t v_t\to\infty\), both \(w_t,v_t\to\infty\), and
\[
1-\left(\frac{w_t}{v_t}\right)^2
=
\frac{\Delta_t}{v_t^2}\to0.
\]
Hence
\[
\frac{w_t}{v_t}\to1.
\]

Finally, Lemma~\ref{lem:scalar_derivative} gives
\[
G_t
=
-\frac{\Upsilon_1}{\rho_t^2}
+
o(\rho_t^{-2}).
\]
Since \(w_t/v_t\to1\) and \(w_t v_t=\rho_t\),
\[
w_t^2+v_t^2
=
2\rho_t+o(\rho_t).
\]
Therefore
\[
\rho_{t+1}-\rho_t
=
\frac{2\kappa \Upsilon_1}{\rho_t}
+
o(\rho_t^{-1}).
\]
Hence
\[
\rho_{t+1}^2-\rho_t^2
=
2\rho_t(\rho_{t+1}-\rho_t)
+
(\rho_{t+1}-\rho_t)^2
=
4\kappa \Upsilon_1+o(1).
\]
Summing over \(t\) yields
\[
\rho_t^2
=
4\kappa \Upsilon_1t(1+o(1)).
\]
Thus
\[
\rho_t
=
\left(\frac{4\eta \Upsilon_1}{d}t\right)^{1/2}(1+o(1)).
\]
Since \(w_t/v_t\to1\) and \(w_t v_t=\rho_t\),
\[
w_t,v_t
=
\rho_t^{1/2}(1+o(1))
=
\left(\frac{4\eta \Upsilon_1}{d}t\right)^{1/4}(1+o(1)).
\]
\end{proof}

\begin{proof}[Proof of Theorem~\ref{thm:looped_end_to_end}]
The initialization satisfies
\[
\Wb^{(0)}=\Wb(0),
\qquad
\Vb^{(0)}=\Vb(1).
\]
We prove by induction that the iterates remain in \(\mathcal M_+\). Suppose
\[
\Wb^{(t)}=\Wb(w_t),
\qquad
\Vb^{(t)}=\Vb(v_t),
\qquad
w_t\ge0,\quad v_t\ge1.
\]
By Lemma~\ref{lem:inactive_zero_gradient} and Lemma~\ref{lem:active_scalar_gradient}, the full population gradient is tangent to \(\mathcal M_+\), and the scalar parameters obey
\[
w_{t+1}
=
w_t-\frac{\eta}{d}v_tR_L'(\rho_t),
\qquad
v_{t+1}
=
v_t-\frac{\eta}{d}w_tR_L'(\rho_t),
\qquad
\rho_t=w_t v_t.
\]
Since \(R_L'(\rho_t)<0\), we have
\[
w_{t+1}\ge0,
\qquad
v_{t+1}\ge1.
\]
Thus the induction closes. Consequently, for every \(t\ge0\),
\[
\Wb^{(t)}
=
\begin{pmatrix}
\mathbf 0&w_t\Ib_d\\
\mathbf 0&\mathbf 0
\end{pmatrix},
\qquad
\Vb^{(t)}
=
\begin{pmatrix}
\mathbf 0&\mathbf 0\\
v_t\Ib_d&\mathbf 0
\end{pmatrix}.
\]
The claims on \(w_t,v_t,\rho_t\) now follow from Lemma~\ref{lem:scalar_dynamics}.

It remains to identify the limiting loss. For almost every \((\Xb,\ab)\),
\[
\frac{(\Ib_d+\rho\Sb)^L\ab}{\rho^L}
=
(\rho^{-1}\Ib_d+\Sb)^L\ab
\to
\Sb^L\ab.
\]
Since by Assumption~\ref{assump:distribution}(A3) \(\EE[\lambda_d(\Xb)^{-1}]<\infty\), we have \(\lambda_d(\Xb)>0\) almost surely, so
\[
\Sb^L\ab\neq0
\]
almost surely. Therefore, by dominated convergence,
\[
R_L(\rho)
\to
\mathcal L_L^{(\infty)},
\]
where
\[
\mathcal L_L^{(\infty)}
=
\EE_{\Xb,\ab}\!\left[
\mathrm{Dist}^2\left(
\frac{(\Xb\Xb^\top)^L\ab}{\|(\Xb\Xb^\top)^L\ab\|_2},
\cV(\Xb)
\right)
\right].
\]
Equivalently,
\[
\mathcal L_L^{(\infty)}
=
\EE_{\Xb,\ab}
\left[
1-
\frac{
((\Xb\Xb^\top)^L\ab)^\top
\boldsymbol{\Pi}_{\Xb}
((\Xb\Xb^\top)^L\ab)
}{
\|(\Xb\Xb^\top)^L\ab\|_2^2
}
\right].
\]

Since Lemma~\ref{lem:scalar_derivative} gives \(R_L\in C^1([0,\infty))\) and
\[
R_L'(\rho)
=
-\frac{\Upsilon_1}{\rho^2}
+
o(\rho^{-2}),
\]
we have
\[
R_L(\rho)-\mathcal L_L^{(\infty)}
=
\int_\rho^\infty -R_L'(s)\,ds
=
\frac{\Upsilon_1}{\rho}
+
o(\rho^{-1}).
\]
Taking \(\rho=\rho_t\) and using Lemma~\ref{lem:scalar_dynamics},
\[
\mathcal L_L(\btheta^{(t)})
-
\mathcal L_L^{(\infty)}
=
\frac{\Upsilon_1}{\rho_t}(1+o(1))
=
\left(\frac{d\Upsilon_1}{4\eta}\right)^{1/2}
t^{-1/2}(1+o(1)).
\]

Moreover, \(\mathcal L_L^{(\infty)}>0\). Indeed, for almost every \(\Xb\), we have
\(r(\Xb)<d\) and \(\lambda_j(\Xb)>0\) for every \(j\). Writing
\[
\ab=\sum_{i=1}^d c_i\vb_i(\Xb),
\]
a uniformly random \(\ab\in\SSS^{d-1}\) has nonzero projection onto the orthogonal complement of the principal eigenspace almost surely. Hence
\[
\mathrm{Dist}^2\left(
\frac{(\Xb\Xb^\top)^L\ab}{\|(\Xb\Xb^\top)^L\ab\|_2},
\cV(\Xb)
\right)
>0
\]
almost surely, and therefore \(\mathcal L_L^{(\infty)}>0\). This completes the proof.
\end{proof}

\section{Proof of Theorem \ref{thm:unnorm convergence}}
\label{sec:appendix:noLN-full}

To highlight the critical role of normalization, we present a detailed analysis of the training dynamics when LN is omitted. This ``unnormalized model'' provides a baseline that exhibits qualitatively different—and provably suboptimal—learning dynamics. As we will demonstrate, its training dynamics converge not to the ideal power-iteration limit, but to a fixed point, regardless of whether the learning target is normalized or not.

\noindent We adopt the previous notation:
\begin{equation*}
\Eb=\begin{pmatrix}\Xb & \mathbf{0}\\ \mathbf{0} & \ab\end{pmatrix},\qquad
\Sb= \Xb\Xb^\top.
\end{equation*}
The model's output, i.e., the lower-$d$ block of the last column of the unnormalized one-layer forward pass, is denoted by $\tilde{\yb}$:
\begin{equation}
\label{eq:tilde-y-def-full}
\tilde{\yb}(\Eb;\btheta)
= \big(\Eb+\Vb\Eb\Eb^\top\Wb\Eb\big)_{d+1:2d,n+1}
= \ab+\Vb_{21}\Sb\Wb_{12}\ab+\Vb_{22}\ab\ab^\top\Wb_{22}\ab.
\end{equation}
We consider two possible learning targets to analyze the system's behavior comprehensively:
\begin{enumerate}
    \item \textbf{Normalized target:} $\yb_{\text{target}}^{(1)} \coloneqq \frac{\Sb\ab}{\|\Sb\ab\|_2}$.
    \item \textbf{Unnormalized target:} $\yb_{\text{target}}^{(2)} \coloneqq \Sb\ab$.
\end{enumerate}
The training objective is the MSE loss for a given target $\yb_{\text{target}}$:
\begin{equation}
\label{eq:tilde-loss-full}
\widetilde{\mathcal L}(\btheta)
\coloneqq
\EE_{\Xb,\ab}\!\left[\left\|\tilde{\yb}(\Eb;\btheta)-\yb_{\text{target}}\right\|_2^2\right].
\end{equation}

\subsection{Reduction to a One-Dimensional Scalar Dynamic}

Our first step is to show that, due to the initialization and the symmetries of the data distribution, the matrix-valued gradient-descent dynamics collapse to a one-dimensional recursion governed by a single scalar parameter. This reduction holds for both choices of $\yb_{\text{target}}$.

\begin{lemma}
\label{lem:noLN:block-full}
%Under Assumption \ref{assump:distribution}, for every gradient descent step $t\ge0$ on the loss $\widetilde{\mathcal L}$ (with either $\yb_{\text{target}}^{(1)}$ or $\yb_{\text{target}}^{(2)}$), the parameter matrices maintain a sparse block structure:
Under Assumption \ref{assump:distribution}, for every gradient descent step $t\ge0$ on the loss $\widetilde{\mathcal L}$ (with either $\yb_{\text{target}}^{(1)}$ or $\yb_{\text{target}}^{(2)}$), the parameter matrices preserve the following block structure:
\[
\Wb_{12}^{(t)}=\tilde{\alpha}_t\Ib_d,\qquad
\Vb_{21}^{(t)}=\tilde{\beta}_t\Ib_d,
\]
while all other blocks remain zero. Consequently, the model's output simplifies to a linear form governed by the scalar product 
$\tilde{\gamma}_t\coloneqq \tilde{\alpha}_t\tilde{\beta}_t$:
\begin{equation*}
\tilde{\yb}^{(t)}=\ab+\tilde{\gamma}_t\Sb\ab.
\end{equation*}
\end{lemma}

\begin{proof}
Direct computation from \eqref{eq:tilde-y-def-full} and \eqref{eq:tilde-loss-full} yields the gradients:
\begin{align*}
\frac{\partial \widetilde{\mathcal L}}{\partial \Wb_{12}}
&= 2\EE\big[\Sb\Vb_{21}^\top(\tilde{\yb}-\yb_{\text{target}})\ab^\top\big], &
\frac{\partial \widetilde{\mathcal L}}{\partial \Vb_{21}}
&= 2\EE\big[(\tilde{\yb}-\yb_{\text{target}})(\Sb\Wb_{12}\ab)^\top\big],\\
\frac{\partial \widetilde{\mathcal L}}{\partial \Wb_{22}}
&= 2\EE\big[\ab\ab^\top\Vb_{22}^\top(\tilde{\yb}-\yb_{\text{target}})\ab^\top\big], &
\frac{\partial \widetilde{\mathcal L}}{\partial \Vb_{22}}
&= 2\EE\big[(\tilde{\yb}-\yb_{\text{target}})(\ab\ab^\top\Wb_{22}\ab)^\top\big].
\end{align*}

\emph{Vanishing blocks.} %At initialization ($t=0$), $\Wb_{22}^{(0)}=\Vb_{22}^{(0)}=\mathbf{0}$. The gradients for $\Wb_{22}$ and $\Vb_{22}$ are linear in $\Wb_{22}$ and $\Vb_{22}$, respectively. Hence the gradients vanish at $t=0$, and the gradient-descent updates satisfy $\Wb_{22}^{(1)}=\Wb_{22}^{(0)}-\eta\cdot\mathbf{0}=\mathbf{0}$ and likewise for $\Vb_{22}^{(1)}$. By induction, these blocks remain zero for all $t\ge 0$. Consequently, the model output simplifies to $\tilde{\yb}=\ab+\Vb_{21}\Sb\Wb_{12}\ab$. Since $\tilde{\yb}$ and $\yb_{\text{target}}$ are irrelevant to other blocks in $\Wb$ and $\Vb$, all other blocks remain zero.
At initialization ($t=0$), we have $\Wb_{22}^{(0)}=\Vb_{22}^{(0)}=\mathbf{0}$. Observe that the gradient $\frac{\partial \widetilde{\mathcal L}}{\partial \Wb_{22}}$ depends linearly on $\Vb_{22}$, while $\frac{\partial \widetilde{\mathcal L}}{\partial \Vb_{22}}$ depends linearly on $\Wb_{22}$. Since both blocks are initialized to zero, their gradients vanish at $t=0$. Consequently, the updates satisfy $\Wb_{22}^{(1)}=\Wb_{22}^{(0)}-\eta\cdot\mathbf{0}=\mathbf{0}$ and similarly $\Vb_{22}^{(1)}=\mathbf{0}$. By induction, these blocks remain zero for all $t\ge 0$. Consequently, the model output simplifies to $\tilde{\yb}=\ab+\Vb_{21}\Sb\Wb_{12}\ab$.
Furthermore, since $\tilde{\yb}$ and $\yb_{\text{target}}$ are independent of the remaining blocks (e.g., $\Wb_{11}, \Vb_{11}$), the gradients with respect to these blocks are identically zero. Thus, all blocks other than $\Wb_{12}$ and $\Vb_{21}$ remain zero throughout training.

\emph{Isotropic structure.} With the vanishing blocks established, we focus on the gradients with respect to $\Wb_{12}$ and $\Vb_{21}$. Let us define:
\begin{equation*}
    \Fb_W(\Xb,\ab)\coloneqq \Sb\Vb_{21}^\top(\tilde{\yb}-\mathbf{y}_{\text{target}})\ab^\top \quad \text{and} \quad \Fb_V(\Xb,\ab)\coloneqq (\tilde{\yb}-\mathbf{y}_{\text{target}})(\Sb\Wb_{12}\ab)^\top.
\end{equation*}
 Consider an arbitrary rotation matrix $\Rb\in \mathrm{SO}(d)$. Under the transformation $(\Xb, \ab) \to (\Rb\Xb, \Rb\ab)$, the components transform as $\Xb \to \Rb\Xb$, $\ab \to \Rb\ab$, $\Sb \to \Rb\Sb\Rb^\top$, $\tilde{\yb} \to \Rb\tilde{\yb}$, and $\yb_{\text{target}} \to \Rb\yb_{\text{target}}$ (for both cases). Under the inductive hypothesis that $\Vb_{21}$ and $\Wb_{12}$ are scalar multiples of the identity, it follows that $\Fb_W$ satisfies the equivariance property $\Fb_W(\Rb\Xb,\Rb\ab)=\Rb\Fb_W(\Xb,\ab)\Rb^\top$, and similarly for $\Fb_V$. 
Consequently, by Lemma~\ref{Appendix Symmetry Proposition}, the expected gradients are scalar multiples of $\Ib_d$. Since the parameters are initialized as scalar multiples of $\Ib_d$ (or zero), and the gradient updates preserve this structure, the claim holds for all $t \ge 0$ by induction.
\end{proof}

\subsection{Analysis of Loss Landscape and Convergence}
\label{sec:appendix:noLN-gamma-conv-refined}
With the dynamics reduced to the scalar parameter $\tilde{\gamma}_t$, we rewrite the population loss \eqref{eq:tilde-loss-full} as
$L(\gamma) \coloneqq \EE\bigl[\|\ab + \gamma \Sb\ab - \yb_{\text{target}}\|_2^2\bigr]$,
which is a quadratic function of $\gamma$. We analyze this objective separately for each choice of target. Define
$\kappa_1 \coloneqq \EE[\ab^\top\Sb\ab]$, $\kappa_2 \coloneqq \EE[\|\Sb\ab\|_2^2]$, and $\kappa_3 \coloneqq \EE[\|\Sb\ab\|_2]$.

\subsubsection{Case 1: Normalized Target }

\paragraph{Loss landscape.} The population loss is given by:
\[
L_1(\gamma) = \EE\!\left[\left\|\ab+\gamma\Sb\ab-\tfrac{\Sb\ab}{\|\Sb\ab\|_2}\right\|_2^2\right] = \EE\!\left[\|\ab\|_2^2 + 2\gamma\ab^\top\Sb\ab - 2\frac{\ab^\top\Sb\ab}{\|\Sb\ab\|_2} + \gamma^2\|\Sb\ab\|_2^2 - 2\gamma\frac{\|\Sb\ab\|_2^2}{\|\Sb\ab\|_2} + 1\right].
\]
Grouping terms by powers of $\gamma$:
\[
L_1(\gamma) = \EE[2 - 2\frac{\ab^\top\Sb\ab}{\|\Sb\ab\|_2}] + 2\gamma(\kappa_1 - \kappa_3) + \gamma^2\kappa_2.
\]
Thus $L_1$ is $2\kappa_2$-strongly convex. The unique global minimizer $\gamma_1^\star$ is found by setting $L_1'(\gamma) = 2(\kappa_1 - \kappa_3) + 2\gamma\kappa_2 = 0$:
\begin{equation*}
\gamma_1^\star = \frac{\kappa_3 - \kappa_1}{\kappa_2}= \frac{\EE[\|\Sb\ab\|_2] - \EE[\ab^\top\Sb\ab]}{\EE[\|\Sb\ab\|_2^2]} > 0.
\end{equation*}
The strict inequality follows from Assumption~\ref{assump:distribution} and the independence of $\ab$ and $\Sb$.

\paragraph{Irreducible loss.} The minimum achievable loss is strictly positive:
\begin{align*}
    L_1(\gamma_1^\star) &= L_1(0) + 2\gamma_1^\star(\kappa_1 - \kappa_3) + (\gamma_1^\star)^2\kappa_2 = L_1(0) - \frac{(\kappa_3-\kappa_1)^2}{\kappa_2} \\
    &= \EE\!\left[2 - 2\frac{\ab^\top\Sb\ab}{\|\Sb\ab\|_2}\right] - \frac{(\EE[\|\Sb\ab\|_2 - \ab^\top\Sb\ab])^2}{\EE[\|\Sb\ab\|_2^2]} \\
    &\ge \EE\!\left[2 - 2\frac{\ab^\top\Sb\ab}{\|\Sb\ab\|_2}\right] - \EE\!\left[\frac{(\|\Sb\ab\|_2 - \ab^\top\Sb\ab)^2}{\|\Sb\ab\|_2^2}\right] \quad \text{(Cauchy–Schwarz)} \\
    &= \EE\!\left[2 - 2\frac{\ab^\top\Sb\ab}{\|\Sb\ab\|_2} - \left(1 - 2\frac{\ab^\top\Sb\ab}{\|\Sb\ab\|_2} + \frac{(\ab^\top\Sb\ab)^2}{\|\Sb\ab\|_2^2}\right)\right] \\&= \EE\!\left[1 - \frac{(\ab^\top\Sb\ab)^2}{\|\Sb\ab\|_2^2}\right] > 0.
\end{align*}
The final inequality follows from Assumption~\ref{assump:distribution} and the independence of $\ab$ and $\Sb$.

\subsubsection{Case 2: Unnormalized Target}

\paragraph{Loss landscape.} The population loss is given by:
\begin{align*}
L_2(\gamma) &= \EE\big[\|\ab + \gamma\Sb\ab - \Sb\ab\|_2^2\big] = \EE\big[\|\ab + (\gamma-1)\Sb\ab\|_2^2\big] \\
&= \EE\big[ \|\ab\|_2^2 + 2(\gamma-1)\ab^\top\Sb\ab + (\gamma-1)^2\|\Sb\ab\|_2^2 \big] \\
&= 1 + 2(\gamma-1)\kappa_1 + (\gamma-1)^2\kappa_2.
\end{align*}
Hence $L_2$ is also $2\kappa_2$-strongly convex. Setting $L_2'(\gamma) = 2\kappa_1 + 2(\gamma-1)\kappa_2 = 0$ yields the unique global minimizer $\gamma_2^\star$:
\begin{equation*}
\gamma_2^\star = 1 - \frac{\kappa_1}{\kappa_2}=\frac{\EE[\|\Sb\ab\|_2^2] - \EE[\ab^\top\Sb\ab]}{\EE[\|\Sb\ab\|_2^2]} \in [0, 1).
\end{equation*}
The minimizer's range follows from Lemma \ref{lem: bounds for Sa}.
\begin{remark}[Degenerate case]
The equality $\kappa_2=\kappa_1$ holds iff $\EE[\operatorname{tr}(\Sb^2)]=\EE[\operatorname{tr}(\Sb)]$; for $\Sb=\Xb\Xb^\top$ with unit-norm columns, this means the columns of $\Xb$ are mutually orthogonal and $n\le d$. In this case $\gamma_2^\star=0$. Under the standard initialization $\gamma_0=0$, we have $F(\gamma_0)=0$ by \eqref{F update}, and hence the update \eqref{gamma update} yields $\gamma_t\equiv 0$ for all $t$. Thus the unnormalized dynamics are trivially trapped at a suboptimal fixed point. 
For clarity, we therefore exclude this degenerate case and focus on the non-degenerate regime where $0<\gamma_2^\star<1$ (equivalently, $\kappa_2>\kappa_1$). Note that the degenerate case is still covered by Theorem~\ref{unnormalized theorem} as a boundary instance with a trivial fixed point.
\end{remark}

\paragraph{Irreducible loss.} The minimum achievable loss is again strictly positive:
\begin{align*}
    L_2(\gamma_2^\star) &= 1 + 2(\gamma_2^\star-1)\kappa_1 + (\gamma_2^\star-1)^2\kappa_2 \\
    &= 1 - 2\frac{\kappa_1^2}{\kappa_2} + \frac{\kappa_1^2}{\kappa_2} \\& = 1 - \frac{(\EE[\ab^\top\Sb\ab])^2}{\EE[\|\Sb\ab\|_2^2]} 
    \geq1 - \frac{\EE[(\ab^\top\Sb\ab)^2]}{\EE[\|\Sb\ab\|_2^2]} >0.
\end{align*}
The final inequality follows from Assumption~\ref{assump:distribution} and the independence of $\ab$ and $\Sb$.

\subsubsection{Unified Convergence Theorem}

In both cases, $L(\gamma)$ is a strongly convex quadratic. We can unify the analysis by writing the loss in terms of its distance to the respective minimizer $\gamma^\star$:
\begin{equation}
\label{eq:noLN-L-distance}
L(\gamma) = L(\gamma^\star)+\kappa_2(\gamma-\gamma^\star)^2.
\end{equation}
We also define a scaled and shifted gradient, which we term the linear statistic $F(\gamma)$:
\begin{equation}\label{F update}
F(\gamma) \coloneqq \frac{1}{d}L'(\gamma) = \frac{2\kappa_2}{d}(\gamma-\gamma^\star) \coloneqq \kappa_4(\gamma-\gamma^\star),
\end{equation}
where $\kappa_4 = 2\kappa_2/d > 0$. Thus $F(\gamma)$ is an affine function of $\gamma$ that is zero only at the minimum $\gamma^\star$.

\paragraph{Exact scalar recursions under gradient descent.}
The gradient descent updates for the scalar parameters $\tilde{\alpha}_t$ and $\tilde{\beta}_t$ are given by:
\begin{equation*}
\tilde{\alpha}_{t+1}=\tilde{\alpha}_t-\eta\tilde{\beta}_tF(\tilde{\gamma}_t),
\qquad
\tilde{\beta}_{t+1}=\tilde{\beta}_t-\eta\tilde{\alpha}_tF(\tilde{\gamma}_t).
\end{equation*}
To analyze the dynamics of $\tilde{\gamma}_t = \tilde{\alpha}_t\tilde{\beta}_t$, we introduce two auxiliary quantities:
\begin{itemize}
    \item $S_t \coloneqq \tilde{\alpha}_t^2+\tilde{\beta}_t^2$.
    \item $D_t \coloneqq \tilde{\beta}_t^2-\tilde{\alpha}_t^2$.
\end{itemize}
A direct expansion of the update rules yields the exact discrete dynamical system for these three scalar quantities. Let $F_t \equiv F(\tilde{\gamma}_t)$ for brevity.

\textbf{Derivation for $\tilde{\gamma}_{t+1}$:}
\begin{align}
\tilde{\gamma}_{t+1} &= \tilde{\alpha}_{t+1}\tilde{\beta}_{t+1} = (\tilde{\alpha}_t-\eta\tilde{\beta}_t F_t)(\tilde{\beta}_t-\eta\tilde{\alpha}_t F_t) \notag\\
&= \tilde{\alpha}_t\tilde{\beta}_t - \eta\tilde{\alpha}_t^2 F_t - \eta\tilde{\beta}_t^2 F_t + \eta^2\tilde{\alpha}_t\tilde{\beta}_t F_t^2 \notag\\
&= \tilde{\gamma}_t - \eta(\tilde{\alpha}_t^2+\tilde{\beta}_t^2)F_t + \eta^2\tilde{\gamma}_t F_t^2 \notag\\
&= \tilde{\gamma}_t - \eta S_t F_t + \eta^2\tilde{\gamma}_t F_t^2.\label{gamma update}
\end{align}

\textbf{Derivation for $S_{t+1}$:}
\begin{align*}
S_{t+1} &= \tilde{\alpha}_{t+1}^2 + \tilde{\beta}_{t+1}^2 = (\tilde{\alpha}_t-\eta\tilde{\beta}_t F_t)^2 + (\tilde{\beta}_t-\eta\tilde{\alpha}_t F_t)^2 \\
&= (\tilde{\alpha}_t^2 - 2\eta\tilde{\alpha}_t\tilde{\beta}_t F_t + \eta^2\tilde{\beta}_t^2 F_t^2) + (\tilde{\beta}_t^2 - 2\eta\tilde{\alpha}_t\tilde{\beta}_t F_t + \eta^2\tilde{\alpha}_t^2 F_t^2) \\
&= (\tilde{\alpha}_t^2+\tilde{\beta}_t^2) - 4\eta\tilde{\alpha}_t\tilde{\beta}_t F_t + \eta^2(\tilde{\alpha}_t^2+\tilde{\beta}_t^2)F_t^2 \\
&= S_t - 4\eta\tilde{\gamma}_t F_t + \eta^2 S_t F_t^2.
\end{align*}

\textbf{Derivation for $D_{t+1}$:}
\begin{align*}
D_{t+1} &= \tilde{\beta}_{t+1}^2 - \tilde{\alpha}_{t+1}^2 = (\tilde{\beta}_t-\eta\tilde{\alpha}_t F_t)^2 - (\tilde{\alpha}_t-\eta\tilde{\beta}_t F_t)^2 \\
&= (\tilde{\beta}_t^2 - 2\eta\tilde{\alpha}_t\tilde{\beta}_t F_t + \eta^2\tilde{\alpha}_t^2 F_t^2) - (\tilde{\alpha}_t^2 - 2\eta\tilde{\alpha}_t\tilde{\beta}_t F_t + \eta^2\tilde{\beta}_t^2 F_t^2) \\
&= (\tilde{\beta}_t^2-\tilde{\alpha}_t^2) - \eta^2(\tilde{\beta}_t^2-\tilde{\alpha}_t^2)F_t^2 \\
&= D_t(1-\eta^2 F_t^2).
\end{align*}

This gives the complete dynamical system:
\begin{equation*}
\begin{aligned}
\tilde{\gamma}_{t+1} &= \tilde{\gamma}_t - \eta S_t F(\tilde{\gamma}_t) + \eta^2\tilde{\gamma}_t F(\tilde{\gamma}_t)^2, \\
S_{t+1} &= (1+\eta^2 F(\tilde{\gamma}_t)^2)S_t - 4\eta\tilde{\gamma}_t F(\tilde{\gamma}_t), \\
D_{t+1} &= (1-\eta^2 F(\tilde{\gamma}_t)^2)D_t.
\end{aligned}
\end{equation*}
A key algebraic identity is $S_t^2=D_t^2+4\tilde{\gamma}_t^2$, which implies $S_t\ge 2|\tilde{\gamma}_t|$.
Since the loss landscape is structurally identical in both cases (a strongly convex quadratic), the convergence analysis is also identical, differing only in the specific values of the minimizer $\gamma^\star$ and the associated constants. We can therefore state and prove a unified theorem.

\begin{theorem}\label{unnormalized theorem}
For either target $\yb_{\text{target}}^{(1)}$ or $\yb_{\text{target}}^{(2)}$, let $L(\gamma)$ be the corresponding loss function. Let $\gamma^\star$ be its unique, finite minimizer. Assume the initialization $\tilde{\alpha}_0=0, \tilde{\beta}_0>0$, which implies $\tilde{\gamma}_0=0$, $S_0 = \tilde{\beta}_0^2 > 0$, and $D_0 = \tilde{\beta}_0^2 - \tilde{\alpha}_0^2 = \tilde{\beta}_0^2$.

Let the following constants be defined based on the initial level set $L(\gamma) \le L(\gamma_0)$, which by \eqref{eq:noLN-L-distance} is $[\gamma^\star-R, \gamma^\star+R]$ where $R \coloneqq \sqrt{(L(\gamma_0)-L(\gamma^\star))/\kappa_2} = \gamma^\star$:
\begin{itemize}
    \item $\Gamma \coloneqq \gamma^\star+R$.
    \item $S_{\max} \coloneqq \sqrt{D_0^2+4\Gamma^2}$.
    \item $F_{\max} \coloneqq \frac{2\kappa_2}{d}R$.
\end{itemize}
There exists a maximum stepsize $\eta_{\max} > 0$, defined as:
\begin{align*}
\eta_{\max} :&= \min\left\{
\frac{1}{4F_{\max}},
\frac{d}{8 \kappa_2 S_{\max}},
\sqrt{\frac{d}{16 \kappa_2 \Gamma F_{\max}}},
\sqrt[3]{\frac{d}{16 \kappa_2 \Gamma F_{\max}^2}},
\frac{d}{2\kappa_2 S_0},
\frac{d}{2\kappa_2\sqrt{2S_0\gamma^\star}}
\right\}\\&=\min\left\{
\frac{d}{8 \kappa_2\sqrt{1+16(\gamma^\star)^2}},
\frac{d}{2\kappa_2 S_0},
\frac{d}{2\kappa_2\sqrt{2S_0\gamma^\star}}
\right\}.
\end{align*}
If the stepsize $\eta$ satisfies $0 < \eta \le \eta_{\max}$, then the iterates of the gradient descent algorithm exhibit the following properties:
\begin{enumerate}
    \item \textbf{Invariant set}: All iterates remain within the initial level set, ensuring they are uniformly bounded:
    \[
    \tilde{\gamma}_t \in [0, 2\gamma^\star] \quad \text{for all } t \ge 0.
    \]

    \item \textbf{Convergence}: The sequence of iterates converges to the unique optimal solution:
    \[
    \lim_{t\to\infty} L(\tilde{\gamma}_t) = L(\gamma^\star) \quad \text{and} \quad \lim_{t\to\infty} \tilde{\gamma}_t = \gamma^\star.
    \]

    \item \textbf{Linear convergence rate}: After the first step, the loss and the parameter converge geometrically to their optimal values. Specifically, for all $t \ge 1$:
    \begin{align*}
    L(\tilde{\gamma}_t) - L(\gamma^\star) &\le \rho^{t-1} \left( L(\tilde{\gamma}_1) - L(\gamma^\star) \right), \\
    |\tilde{\gamma}_t - \gamma^\star| &\le \rho^{\frac{t-1}{2}} \sqrt{\frac{L(\tilde{\gamma}_1) - L(\gamma^\star)}{\kappa_2}},
    \end{align*}
    where the contraction factor $\rho \in [0, 1)$ is given by:
    \[
    \rho \coloneqq 1 - \eta^2 \frac{8\kappa_2^2}{d^2} S_0 \gamma^\star.
    \]
\end{enumerate}
\end{theorem}

\begin{proof}
The proof is presented in three parts. Part 1 establishes a descent lemma that guarantees the loss decreases at each step, which in turn proves the invariance of the level set. Part 2 analyzes the first step of the iteration to establish a uniform lower bound for $S_t$. Part 3 uses these results to derive the global and linear convergence rates.

\paragraph{Part 1: Descent lemma and invariant set}
Our first objective is to show that for a sufficiently small stepsize $\eta$, the loss function $L(\tilde{\gamma}_t)$ is guaranteed to decrease at every step, provided $\tilde{\gamma}_t \neq \gamma^\star$.

The change in the parameter $\tilde{\gamma}_t$ at each step is given by $\Delta\tilde{\gamma}_t = \tilde{\gamma}_{t+1} - \tilde{\gamma}_t = -\eta S_t F(\tilde{\gamma}_t) + \eta^2 \tilde{\gamma}_t F(\tilde{\gamma}_t)^2$. Since the loss function $L(\gamma)$ is quadratic, its change can be expressed exactly using a second-order Taylor expansion. Let $F_t \equiv F(\tilde{\gamma}_t)$ for brevity.
\begin{align*}
    L(\tilde{\gamma}_{t+1}) - L(\tilde{\gamma}_t)
    &= L'(\tilde{\gamma}_t)\Delta\tilde{\gamma}_t + \frac{1}{2}L''(\tilde{\gamma}_t)(\Delta\tilde{\gamma}_t)^2 \\
    &= (dF_t)\Delta\tilde{\gamma}_t + \kappa_2(\Delta\tilde{\gamma}_t)^2 \\
    &= (dF_t)\left(-\eta S_t F_t + \eta^2 \tilde{\gamma}_t F_t^2\right) + \kappa_2\left(-\eta S_t F_t + \eta^2 \tilde{\gamma}_t F_t^2\right)^2 \\
    &= -\eta d S_t F_t^2 + \eta^2 d \tilde{\gamma}_t F_t^3 + \kappa_2\left( \eta^2 S_t^2 F_t^2 - 2\eta^3 \tilde{\gamma}_t S_t F_t^3 + \eta^4 \tilde{\gamma}_t^2 F_t^4 \right) \\
    &= -\eta d S_t F_t^2 + \left( \eta^2 d \tilde{\gamma}_t F_t^3 + \kappa_2\eta^2 S_t^2 F_t^2 - 2\kappa_2\eta^3 \tilde{\gamma}_t S_t F_t^3 + \kappa_2\eta^4 \tilde{\gamma}_t^2 F_t^4 \right).
\end{align*}
The change in loss is exactly $L(\tilde{\gamma}_{t+1})-L(\tilde{\gamma}_t) = -\eta d S_t F_t^2 + \mathcal{E}_t$, where $\mathcal{E}_t$ is the sum of all higher-order terms in $\eta$:
\[
\mathcal{E}_t = \underbrace{\eta^2 d \tilde{\gamma}_t F_t^3}_{T_1} + \underbrace{\kappa_2\eta^2 S_t^2 F_t^2}_{T_2} \underbrace{- 2\kappa_2\eta^3 \tilde{\gamma}_t S_t F_t^3}_{T_3} + \underbrace{\kappa_2\eta^4 \tilde{\gamma}_t^2 F_t^4}_{T_4}.
\]
We now show that for a sufficiently small $\eta$, the magnitude of $\mathcal{E}_t$ is smaller than the main negative term. Assume the inductive hypothesis that $\tilde{\gamma}_t$ is in the initial level set, so $|\tilde{\gamma}_t|\le\Gamma$, $|F_t|\le F_{\max}$, and $S_t\le S_{\max}$. We bound the ratio of the absolute value of each term in $\mathcal{E}_t$ to the main descent term $|\eta d S_t F_t^2|$:
\begin{align*}
\frac{|T_1|}{\eta d S_t F_t^2} &= \frac{\eta^2 d |\tilde{\gamma}_t| |F_t|^3}{\eta d S_t F_t^2} = \eta\frac{|\tilde{\gamma}_t|}{S_t}|F_t| \le \frac{\eta F_{\max}}{2}. \\
\frac{|T_2|}{\eta d S_t F_t^2} &= \frac{\kappa_2\eta^2 S_t^2 F_t^2}{\eta d S_t F_t^2} = \eta\frac{\kappa_2 S_t}{d} \le \eta\frac{\kappa_2 S_{\max}}{d}. \\
\frac{|T_3|}{\eta d S_t F_t^2} &= \frac{2\kappa_2\eta^3 |\tilde{\gamma}_t| S_t |F_t|^3}{\eta d S_t F_t^2} = \eta^2\frac{2 \kappa_2 |\tilde{\gamma}_t||F_t|}{d} \le \eta^2\frac{2 \kappa_2 \Gamma F_{\max}}{d}. \\
\frac{|T_4|}{\eta d S_t F_t^2} &= \frac{\kappa_2\eta^4 \tilde{\gamma}_t^2 F_t^4}{\eta d S_t F_t^2} = \eta^3\frac{\kappa_2 \tilde{\gamma}_t^2 F_t^2}{d S_t} \le \eta^3\frac{\kappa_2 \Gamma^2 F_{\max}^2}{d (2|\tilde{\gamma}_t|)} \le \eta^3\frac{\kappa_2 \Gamma F_{\max}^2}{d}.
\end{align*}
For the total error $|\mathcal{E}_t|$ to be less than or equal to $\frac{1}{2}\eta d S_t F_t^2$, it is sufficient that the sum of these ratios is less than or equal to $1/2$. We can ensure this by requiring each ratio to be less than or equal to $1/8$. This condition is met if $\eta \le \eta_{\max}$, due to the definition of $\eta_{\max}$.
Therefore, for any $\eta \le \eta_{\max}$, we have $|\mathcal{E}_t| \le \frac{1}{2}\eta d S_t F_t^2$, which implies:
\[
L(\tilde{\gamma}_{t+1})-L(\tilde{\gamma}_t) = -\eta d S_t F_t^2 + \mathcal{E}_t \le -\eta d S_t F_t^2 + \frac{1}{2}\eta d S_t F_t^2 = -\frac{1}{2}\eta d S_t F_t^2 \le 0.
\]
This establishes the descent property. By mathematical induction, since $\tilde{\gamma}_0$ is in the level set and the loss never increases, all subsequent iterates $\tilde{\gamma}_t$ must also remain in the initial level set. This proves statement (1) of the theorem.

\paragraph{Part 2: Analysis of the first iteration}
We analyze the transition from $t=0$ to $t=1$. At $t=0$, we have $\tilde{\gamma}_0=0$, $S_0=\tilde{\beta}_0^2$, and $F(\tilde{\gamma}_0) = \frac{2\kappa_2}{d}(0-\gamma^\star) = -\frac{2\kappa_2}{d}\gamma^\star$. The updates for $\tilde{\alpha}$ and $\tilde{\beta}$ are:
\begin{align*}
\tilde{\alpha}_1 &= \tilde{\alpha}_0 - \eta\tilde{\beta}_0 F(\tilde{\gamma}_0) = 0 - \eta\tilde{\beta}_0 \left(-\frac{2\kappa_2}{d}\gamma^\star\right) = \eta\tilde{\beta}_0\frac{2\kappa_2}{d}\gamma^\star > 0, \\
\tilde{\beta}_1 &= \tilde{\beta}_0 - \eta\tilde{\alpha}_0 F(\tilde{\gamma}_0) = \tilde{\beta}_0 - 0 = \tilde{\beta}_0.
\end{align*}
Consequently, the new parameter $\tilde{\gamma}_1$ is:
\[
\tilde{\gamma}_1 = \tilde{\alpha}_1\tilde{\beta}_1 = \left(\eta\tilde{\beta}_0\frac{2\kappa_2}{d}\gamma^\star\right)\tilde{\beta}_0 = \eta \frac{2\kappa_2}{d} S_0 \gamma^\star.
\]
The condition $\eta \le \eta_{\max} \le \frac{d}{2\kappa_2 S_0}$ ensures that $\tilde{\gamma}_1 \le \gamma^\star$. Since $\tilde{\gamma}_1 > 0$, we have established that the first step moves the parameter into the interval $(0, \gamma^\star]$.
From Part 1, we know the loss is monotonically decreasing, which for a quadratic means the iterates get progressively closer to $\gamma^\star$. Therefore, for all $t \ge 1$, $|\tilde{\gamma}_t - \gamma^\star| \le |\tilde{\gamma}_1 - \gamma^\star|$, which implies $\tilde{\gamma}_t \ge \tilde{\gamma}_1$. This gives a uniform positive lower bound on the parameter and on $S_t$ for all subsequent steps:
\[
S_t = \sqrt{(\tilde{\alpha}_t^2-\tilde{\beta}_t^2)^2+4\tilde{\gamma}_t^2} \ge 2|\tilde{\gamma}_t| \ge 2\tilde{\gamma}_1 \quad \text{for all } t \ge 1.
\]

\paragraph{Part 3: Derivation of linear convergence}
From the descent property in Part 1 and the properties of the quadratic loss, we have:
\[
L(\tilde{\gamma}_{t+1})-L(\tilde{\gamma}_t) \le -\frac{1}{2}\eta d S_t F(\tilde{\gamma}_t)^2.
\]
Adding $L(\tilde{\gamma}_t) - L(\gamma^\star)$ to both sides and rearranging gives:
\[
L(\tilde{\gamma}_{t+1}) - L(\gamma^\star) \le L(\tilde{\gamma}_t) - L(\gamma^\star) - \frac{1}{2}\eta d S_t F(\tilde{\gamma}_t)^2.
\]
We substitute $F(\tilde{\gamma}_t) = \kappa_4(\tilde{\gamma}_t-\gamma^\star)$, which gives $F(\tilde{\gamma}_t)^2 = \kappa_4^2(\tilde{\gamma}_t-\gamma^\star)^2$. Also, $L(\tilde{\gamma}_t)-L(\gamma^\star) = \kappa_2(\tilde{\gamma}_t-\gamma^\star)^2$, which implies $(\tilde{\gamma}_t-\gamma^\star)^2 = \frac{1}{\kappa_2}(L(\tilde{\gamma}_t)-L(\gamma^\star))$. Substituting these into the inequality:
\begin{align*}
L(\tilde{\gamma}_{t+1}) - L(\gamma^\star) &\le L(\tilde{\gamma}_t) - L(\gamma^\star) - \frac{1}{2}\eta d S_t \left(\kappa_4^2\right) \frac{1}{\kappa_2}(L(\tilde{\gamma}_t)-L(\gamma^\star)) \\
&= \left(1 - \frac{\eta d \kappa_4^2 S_t}{2\kappa_2}\right) (L(\tilde{\gamma}_t) - L(\gamma^\star)) \\
&= \left(1 - \frac{\eta d (4\kappa_2^2/d^2) S_t}{2\kappa_2}\right) (L(\tilde{\gamma}_t) - L(\gamma^\star)) \\
&= \left(1 - \frac{2\eta \kappa_2 S_t}{d}\right) (L(\tilde{\gamma}_t) - L(\gamma^\star)).
\end{align*}
Using the lower bound $S_t \ge 2\tilde{\gamma}_1$ established in Part 2 for all $t \ge 1$:
\[
1 - \frac{2\eta \kappa_2 S_t}{d} \le 1 - \frac{2\eta \kappa_2 (2\tilde{\gamma}_1)}{d} = 1 - \frac{4\eta \kappa_2}{d}\tilde{\gamma}_1 = 1 - \frac{4\eta \kappa_2}{d}\left(\eta \frac{2\kappa_2}{d} S_0 \gamma^\star\right) = 1 - \eta^2 \frac{8\kappa_2^2}{d^2} S_0 \gamma^\star\coloneqq \rho.
\]
The condition $\eta \le \eta_{\max} \le \frac{d}{2\kappa_2\sqrt{2S_0\gamma^\star}}$ ensures that $\eta^2 \frac{8\kappa_2^2}{d^2} S_0 \gamma^\star \le 1$, so $\rho \in [0, 1)$.
Thus, for all $t \ge 1$, we have the contraction:
\[
L(\tilde{\gamma}_{t+1}) - L(\gamma^\star) \le \rho (L(\tilde{\gamma}_t) - L(\gamma^\star)).
\]
By recursively applying this inequality from $t$ down to $1$, we get:
\[
L(\tilde{\gamma}_t) - L(\gamma^\star) \le \rho^{t-1} (L(\tilde{\gamma}_1) - L(\gamma^\star)).
\]
Since $\rho \in [0, 1)$, this proves that the loss converges to $L(\gamma^\star)$ at a linear (geometric) rate.
Finally, to find the convergence rate for the parameter $\tilde{\gamma}_t$, we use the relation $L(\tilde{\gamma}_t) - L(\gamma^\star) = \kappa_2(\tilde{\gamma}_t - \gamma^\star)^2$:
\[
\kappa_2(\tilde{\gamma}_t - \gamma^\star)^2 \le \rho^{t-1} (L(\tilde{\gamma}_1) - L(\gamma^\star)).
\]
Taking the square root of both sides yields:
\[
|\tilde{\gamma}_t - \gamma^\star| \le \sqrt{\frac{L(\tilde{\gamma}_1) - L(\gamma^\star)}{\kappa_2}} \cdot (\sqrt{\rho})^{t-1} = \sqrt{\frac{L(\tilde{\gamma}_1) - L(\gamma^\star)}{\kappa_2}} \cdot \rho^{\frac{t-1}{2}}.
\]
This establishes the convergence of the parameter $\tilde{\gamma}_t$ to $\gamma^\star$ and completes the proof of all statements in the theorem.
\end{proof}
\begin{corollary}
    Noting that $\kappa_2 = \Upsilon_2$, $\tilde{\alpha}_0 = 0$, and $\tilde{\beta}_0 = 1$ under the conditions of Theorem~\ref{thm:unnorm convergence}, we obtain that for any stepsize $\eta$ satisfying
\begin{equation*}
    0 < \eta \leq\frac{d}{8 \Upsilon_2\sqrt{1+16(\gamma^\star)^2}},
\end{equation*}
the following statements hold:
    \begin{enumerate}[leftmargin = *,itemsep=4pt]
    \item \textbf{Convergence:} The sequence of iterates converges to the unique optimal solution:
    \[
    \lim_{t\to\infty} \tilde{\mathcal L}(\btheta^{(t)}) = \tilde{\mathcal L}(\btheta^\star) > 0
    \quad \text{and} \quad
    \lim_{t\to\infty} \tilde{\gamma}_t = \gamma^\star.
    \]

    \item \textbf{Linear convergence rate:} After the first step, both the loss and the parameter converge geometrically to their optimal values. Specifically, for all $t \ge 1$,
    \begin{align*}
     \tilde{\mathcal L}(\btheta^{(t)}) - \tilde{\mathcal L}^\star
    &\le \rho^{t-1} \bigl(  \tilde{\mathcal L}(\btheta^{(1)}) - \tilde{\mathcal L}^\star \bigr),\\
    \bigl|\tilde{\gamma}_t - \gamma^\star\bigr|
    &\le \sqrt{\frac{\rho^{t-1}}{\Upsilon_2}}\,
    \sqrt{\tilde{\mathcal L}(\btheta^{(1)}) - \tilde{\mathcal L}^\star},
    \end{align*}
    where     $\rho \coloneqq 1 - 8\eta^2 {\Upsilon_2^2\gamma^\star}/{d^2}  \in [0, 1)$.
    \end{enumerate}
\end{corollary}

%\begin{remark}[Structural Inevitability of Suboptimal Convergence]
%The unified analysis yields a central conclusion: the failure of the unnormalized model is \emph{structural} and does not depend on the choice of learning target. In both cases, the training dynamics converge to a unique, finite fixed point $\gamma^\star$ that is strictly suboptimal.

%This behavior contrasts sharply with the normalized model analyzed in the main text, where the corresponding parameter $\gamma_t$ diverges. This divergence is not pathological but essential: as $\gamma_t\to\infty$, the pre-normalization update $(\Ib_d+\gamma_t\Sb)\ab$ becomes asymptotically collinear with $\Sb\ab$, thereby recovering the one-step Power Method and driving the normalized loss to zero.

%By comparison, the unnormalized model is confined to an optimization landscape whose minimizer corresponds to a fixed affine combination of the identity and the power update. Training therefore settles at $(\Ib_d+\gamma^\star\Sb)\ab$, which cannot exactly isolate the principal eigenvector, yielding a strictly positive irreducible loss $L(\gamma^\star)>0$.

%Hence, normalization is not merely a numerical stabilizer or a representational choice; it is a \textbf{structural requirement} that reshapes the optimization landscape, enabling escape from a suboptimal fixed point and convergence toward the correct algorithmic limit.
%\end{remark}

\section{Proof of Theorem \ref{thm:convergence}}\label{proof of theorem 1}
In this section, we detail the proof of Theorem \ref{thm:convergence}. The proof proceeds in three main steps:
\begin{enumerate}
    \item By induction, we show that under gradient descent, the matrices $\Wb^{(t)}$ and $\Vb^{(t)}$ maintain the following block structure:
\begin{equation*}
    \Wb^{(t)}_{12}=\alpha_t \Ib_d,\qquad \Vb^{(t)}_{21}=\beta_t \Ib_d,
    \end{equation*}
where $\alpha_t$ and $\beta_t$ are scalar parameters, and all other blocks remain zero.
    \item By analyzing the update dynamics of $\alpha_t$ and $\beta_t$, we demonstrate that their product $\gamma_t = \alpha_t\beta_t$ diverges, and more precisely, that both $\alpha_t$ and $\beta_t$ grow at a rate of $\Theta((\eta\Upsilon_3t/d)^{1/6})$.
    \item With the behavior of $\alpha_t$ and $\beta_t$ determined, we further control the loss function $\mathcal{L}(\btheta^{(t)})$ and establish the convergence of the loss function.
\end{enumerate}
\subsection{Dynamics of Gradient Descent}
We begin by computing the gradients of the loss with respect to $\Wb$ and $\Vb$. Subsequently, we prove by induction that the block structure is preserved:
$\Wb^{(t)}_{12}=\alpha_t \Ib_d, \Vb^{(t)}_{21}=\beta_t \Ib_d$,
and all other blocks remain zero.
To facilitate the computation of gradients and for notational convenience, we first derive a simplified expression for the transformer's output $\yb$, as defined in \eqref{transformer prediction}. Let $\Sb=\Xb\Xb^\top$. We introduce the auxiliary vectors $\pb, \hb$ and the scalar $r$ in the following lemma.
\begin{lemma}
   The prediction of the transformer, $\yb$, as defined in \eqref{transformer prediction}, can be expressed as:
\begin{equation*}
  \yb(\Eb; \btheta)= \mathrm{LN}\Big(\Eb + \Vb\Eb\Eb^{\top}\Wb\Eb\Big)_{d+1:2d, n+1}=\frac{\hb}{r},
\end{equation*}
where
\begin{align*}
        \pb &= \Vb_{11}\Sb\Wb_{12}\ab + \Vb_{12}\ab\ab^\top\Wb_{22}\ab, \\
        \hb &= \ab + \Vb_{21}\Sb\Wb_{12}\ab + \Vb_{22}\ab\ab^\top\Wb_{22}\ab,\\
        r &= \sqrt{\|\pb\|_2^2 + \|\hb\|_2^2}.
\end{align*}
\begin{proof}
   The proof begins by computing the last column of the matrix before applying LN. Let this pre-normalization vector be $\mathbf{c} \in \RR^{2d}$. It is given by:
\begin{equation*}
\mathbf{c} = \left( \Eb + \Vb\Eb\Eb^{\top}\Wb\Eb \right)_{:, n+1} = \Eb_{:, n+1} + \Vb(\Eb\Eb^\top)(\Wb\Eb_{:, n+1}).
\end{equation*}
The input term is $\Eb_{:, n+1} = [\mathbf{0}_{d}^\top, \ab^\top]^\top$. The product term is computed sequentially. First, the matrix product $\Eb\Eb^{\top}$ is:
\begin{align*}
\Eb\Eb^{\top} &= \begin{pmatrix} \Xb & \mathbf{0} \\ \mathbf{0} & \ab \end{pmatrix} \begin{pmatrix} \Xb^\top & \mathbf{0} \\ \mathbf{0} & \ab^\top \end{pmatrix} \\&= \begin{pmatrix} \Xb\Xb^\top & \mathbf{0} \\ \mathbf{0} & \ab\ab^\top \end{pmatrix} = \begin{pmatrix} \Sb & \mathbf{0} \\ \mathbf{0} & \ab\ab^\top \end{pmatrix}.
\end{align*}
The product of $\Wb$ with the last column of $\Eb$ yields:
\begin{equation*}
\Wb\Eb_{:, n+1} = \begin{pmatrix} \Wb_{11} & \Wb_{12} \\ \Wb_{21} & \Wb_{22} \end{pmatrix} \begin{pmatrix} \mathbf{0} \\ \ab \end{pmatrix} = \begin{pmatrix} \Wb_{12}\ab \\ \Wb_{22}\ab \end{pmatrix}.
\end{equation*}
Pre-multiplying this result by $\Eb\Eb^{\top}$ and subsequently by $\Vb$ gives:
\begin{align*}
\Vb \left( \Eb\Eb^{\top} \Wb\Eb_{:, n+1} \right) &= \begin{pmatrix} \Vb_{11} & \Vb_{12} \\ \Vb_{21} & \Vb_{22} \end{pmatrix} \begin{pmatrix} \Sb\Wb_{12}\ab \\ \ab\ab^\top\Wb_{22}\ab \end{pmatrix} \\&= \begin{pmatrix} \Vb_{11}\Sb\Wb_{12}\ab + \Vb_{12}\ab\ab^\top\Wb_{22}\ab \\ \Vb_{21}\Sb\Wb_{12}\ab + \Vb_{22}\ab\ab^\top\Wb_{22}\ab \end{pmatrix}.
\end{align*}
The full expression for $\mathbf{c}$ is the sum of the input and product terms. Its upper and lower blocks define $\pb$ and $\hb$, respectively:
\begin{align*}
\mathbf{c} &= \begin{pmatrix} \mathbf{0} \\ \ab \end{pmatrix} + \begin{pmatrix} \Vb_{11}\Sb\Wb_{12}\ab + \Vb_{12}\ab\ab^\top\Wb_{22}\ab \\ \Vb_{21}\Sb\Wb_{12}\ab + \Vb_{22}\ab\ab^\top\Wb_{22}\ab \end{pmatrix} \\&=\begin{pmatrix} \Vb_{11}\Sb\Wb_{12}\ab + \Vb_{12}\ab\ab^\top\Wb_{22}\ab \\ \ab+\Vb_{21}\Sb\Wb_{12}\ab + \Vb_{22}\ab\ab^\top\Wb_{22}\ab \end{pmatrix}= \begin{pmatrix} \pb \\ \hb \end{pmatrix}.
\end{align*}
Finally, applying the LN to $\mathbf{c}$ and extracting the lower $d$-dimensional block yields the prediction:
\begin{equation*}
    \yb = \frac{\mathbf{c}_{d+1:2d}}{\|\mathbf{c}\|_2} = \frac{\hb}{\sqrt{\|\pb\|_2^2 + \|\hb\|_2^2}} = \frac{\hb}{r}.
\end{equation*}
This completes the proof.
\end{proof}
\end{lemma}

To analyze the training dynamics, we compute the gradient of the MSE loss function, $\mathcal{L}(\btheta) = \mathbb{E}[\|\yb - \mathbf{y}_{\text{target}}\|_2^2]$, with respect to the parameters $\btheta = (\Vb, \Wb)$. The target vector $\mathbf{y}_{\text{target}}$ is the ideal output of the power method's single-step update:
\begin{equation*}
\mathbf{y}_{\text{target}} \coloneqq \frac{\Xb\Xb^\top\ab}{\|\Xb\Xb^\top\ab\|_2} = \frac{\Sb\ab}{\|\Sb\ab\|_2}.
\end{equation*}
\subsubsection{Gradient with Respect to \texorpdfstring{$\Wb$}{W} and \texorpdfstring{$\Vb$}{V}}
We calculate the gradients $ \frac{\partial \mathcal{L}(\btheta)}{\partial \Wb_{ij}}$ and $\frac{\partial \mathcal{L}(\btheta)}{\partial \Vb_{ij}}$, respectively, using the chain rule. The core idea is to first compute the gradient of the loss with respect to the intermediate vectors $\pb$ and $\hb$, and then propagate these gradients backward to the parameters.
Finally, we have the following closed-form formulas:
\begin{align*}
 \frac{\partial \mathcal{L}(\btheta)}{\partial \Wb_{11}} &= \mathbf{0},\\
    \frac{\partial \mathcal{L}(\btheta)}{\partial \Wb_{12}} &= \mathbb{E}_{\Xb, \ab} \left[\left(\frac{\partial \pb}{\partial \Wb_{12}}\right)^\top \left(-\frac{\pb\hb^\top}{r^3} 2(\mathbf{y} - \mathbf{y}_{\text{target}})\right) + \left(\frac{\partial \hb}{\partial \Wb_{12}}\right)^\top \frac{1}{r}\left(\Ib - \frac{\hb\hb^\top}{r^2}\right) 2(\mathbf{y} - \mathbf{y}_{\text{target}}) \right]\\
    &= \mathbb{E}_{\Xb, \ab} \left[(\Vb_{11}\Sb)^\top \left(-\frac{\pb\hb^\top}{r^3} 2(\mathbf{y} - \mathbf{y}_{\text{target}})\right) \ab^\top + (\Vb_{21}\Sb)^\top \frac{1}{r}\left(\Ib - \frac{\hb\hb^\top}{r^2}\right) 2(\mathbf{y} - \mathbf{y}_{\text{target}}) \ab^\top \right]\\
    &= \mathbb{E}_{\Xb, \ab} \left[\Sb^\top \left(\Vb_{11}^\top \left(-\frac{\pb\hb^\top}{r^3} 2(\mathbf{y} - \mathbf{y}_{\text{target}})\right) + \Vb_{21}^\top \frac{1}{r}\left(\Ib - \frac{\hb\hb^\top}{r^2}\right) 2(\mathbf{y} - \mathbf{y}_{\text{target}})\right) \ab^\top\right]\\
    &=\mathbb{E}_{\Xb, \ab} \left[ \Sb^\top \left( \Vb_{21}^\top \frac{1}{r}\left(\Ib - \frac{\hb\hb^\top}{r^2}\right) - \Vb_{11}^\top \frac{\pb\hb^\top}{r^3} \right) 2(\mathbf{y} - \mathbf{y}_{\text{target}}) \ab^\top \right],\\
    \frac{\partial \mathcal{L}(\btheta)}{\partial \Wb_{21}} &= \mathbf{0},\\
    \frac{\partial \mathcal{L}(\btheta)}{\partial \Wb_{22}} &= \mathbb{E}_{\Xb, \ab} \left[\left(\frac{\partial \pb}{\partial \Wb_{22}}\right)^\top \left(-\frac{\pb\hb^\top}{r^3} 2(\mathbf{y} - \mathbf{y}_{\text{target}})\right) + \left(\frac{\partial \hb}{\partial \Wb_{22}}\right)^\top \frac{1}{r}\left(\Ib - \frac{\hb\hb^\top}{r^2}\right) 2(\mathbf{y} - \mathbf{y}_{\text{target}})\right] \\
    &= \mathbb{E}_{\Xb, \ab} \left[(\Vb_{12}\ab\ab^\top)^\top \left(-\frac{\pb\hb^\top}{r^3} 2(\mathbf{y} - \mathbf{y}_{\text{target}})\right) \ab^\top + (\Vb_{22}\ab\ab^\top)^\top \frac{1}{r}\left(\Ib - \frac{\hb\hb^\top}{r^2}\right) 2(\mathbf{y} - \mathbf{y}_{\text{target}}) \ab^\top \right]\\
    &=\mathbb{E}_{\Xb, \ab} \left[ (\ab\ab^\top)^\top (\Vb_{12}^\top \left(-\frac{\pb\hb^\top}{r^3} 2(\mathbf{y} - \mathbf{y}_{\text{target}})\right) + \Vb_{22}^\top \frac{1}{r}\left(\Ib - \frac{\hb\hb^\top}{r^2}\right) 2(\mathbf{y} - \mathbf{y}_{\text{target}})) \ab^\top\right]\\
    &= \mathbb{E}_{\Xb, \ab} \left[ \ab\ab^\top \left( \Vb_{22}^\top \frac{1}{r}\left(\Ib - \frac{\hb\hb^\top}{r^2}\right) - \Vb_{12}^\top \frac{\pb\hb^\top}{r^3} \right) 2(\mathbf{y} - \mathbf{y}_{\text{target}}) \ab^\top \right],\\
\frac{\partial \mathcal{L}(\btheta)}{\partial \Vb_{11}} &=  \mathbb{E}_{\Xb, \ab} \left[\left(\frac{\partial \pb}{\partial \Vb_{11}}\right)^\top \left(-\frac{\pb\hb^\top}{r^3} 2(\mathbf{y} - \mathbf{y}_{\text{target}})\right) + \left(\frac{\partial \hb}{\partial \Vb_{11}}\right)^\top \frac{1}{r}\left(\Ib - \frac{\hb\hb^\top}{r^2}\right) 2(\mathbf{y} - \mathbf{y}_{\text{target}}) \right]\\&=\mathbb{E}_{\Xb, \ab} \left[ \Ib_d^\top\left(-\frac{\pb\hb^\top}{r^3} 2(\mathbf{y} - \mathbf{y}_{\text{target}})\right) (\Sb\Wb_{12}\ab)^\top \right]
\\&=\mathbb{E}_{\Xb, \ab} \left[ \left(-\frac{\pb\hb^\top}{r^3} 2(\mathbf{y} - \mathbf{y}_{\text{target}})\right) (\Sb\Wb_{12}\ab)^\top \right], \\
\frac{\partial \mathcal{L}(\btheta)}{\partial \Vb_{12}} &=  \mathbb{E}_{\Xb, \ab} \left[\left(\frac{\partial \pb}{\partial \Vb_{12}}\right)^\top \left(-\frac{\pb\hb^\top}{r^3} 2(\mathbf{y} - \mathbf{y}_{\text{target}})\right) + \left(\frac{\partial \hb}{\partial \Vb_{12}}\right)^\top \frac{1}{r}\left(\Ib - \frac{\hb\hb^\top}{r^2}\right) 2(\mathbf{y} - \mathbf{y}_{\text{target}}) \right]\\&=\mathbb{E}_{\Xb, \ab} \left[ \Ib_d^\top\left(-\frac{\pb\hb^\top}{r^3} 2(\mathbf{y} - \mathbf{y}_{\text{target}})\right) (\ab\ab^\top\Wb_{22}\ab)^\top \right]\\&=\mathbb{E}_{\Xb, \ab} \left[ \left(-\frac{\pb\hb^\top}{r^3} 2(\mathbf{y} - \mathbf{y}_{\text{target}})\right) (\ab\ab^\top\Wb_{22}\ab)^\top \right], \\
\frac{\partial \mathcal{L}(\btheta)}{\partial \Vb_{21}} &=  \mathbb{E}_{\Xb, \ab} \left[\left(\frac{\partial \pb}{\partial \Vb_{21}}\right)^\top \left(-\frac{\pb\hb^\top}{r^3} 2(\mathbf{y} - \mathbf{y}_{\text{target}})\right) + \left(\frac{\partial \hb}{\partial \Vb_{21}}\right)^\top \frac{1}{r}\left(\Ib - \frac{\hb\hb^\top}{r^2}\right) 2(\mathbf{y} - \mathbf{y}_{\text{target}}) \right]\\&=\mathbb{E}_{\Xb, \ab} \left[ \Ib_d^\top\left(\frac{1}{r}\left(\Ib - \frac{\hb\hb^\top}{r^2}\right) 2(\mathbf{y} - \mathbf{y}_{\text{target}})\right) (\Sb\Wb_{12}\ab)^\top \right]\\&=\mathbb{E}_{\Xb, \ab} \left[ \left(\frac{1}{r}\left(\Ib - \frac{\hb\hb^\top}{r^2}\right) 2(\mathbf{y} - \mathbf{y}_{\text{target}})\right) (\Sb\Wb_{12}\ab)^\top \right], \\
\frac{\partial \mathcal{L}(\btheta)}{\partial \Vb_{22}} &=  \mathbb{E}_{\Xb, \ab} \left[\left(\frac{\partial \pb}{\partial \Vb_{22}}\right)^\top \left(-\frac{\pb\hb^\top}{r^3} 2(\mathbf{y} - \mathbf{y}_{\text{target}})\right) + \left(\frac{\partial \hb}{\partial \Vb_{22}}\right)^\top \frac{1}{r}\left(\Ib - \frac{\hb\hb^\top}{r^2}\right) 2(\mathbf{y} - \mathbf{y}_{\text{target}}) \right]\\&=\mathbb{E}_{\Xb, \ab} \left[ \Ib_d^\top\left(\frac{1}{r}\left(\Ib - \frac{\hb\hb^\top}{r^2}\right)2(\mathbf{y} - \mathbf{y}_{\text{target}})\right) (\ab\ab^\top\Wb_{22}\ab)^\top \right]\\&=\mathbb{E}_{\Xb, \ab} \left[ \left(\frac{1}{r}\left(\Ib - \frac{\hb\hb^\top}{r^2}\right) 2(\mathbf{y} - \mathbf{y}_{\text{target}})\right) (\ab\ab^\top\Wb_{22}\ab)^\top \right].
\end{align*}

\subsubsection{Gradient Descent Analysis}

Before analyzing the gradient descent steps, we prove a general lemma that forms the backbone of our entire argument. It shows why the expectation of certain matrix-valued functions becomes isotropic. 

\begin{lemma}\label{Appendix Symmetry Proposition}
Let $\Fb(\Xb, \ab)$ be a matrix-valued function of random variables $\Xb$ that satisfies Assumption \ref{assump:distribution}, where $\ab$ is uniformly distributed on the unit sphere $\mathbb{S}^{d-1}$ and independent of $\Xb$. If $\Fb$ exhibits rotational equivariance, meaning that, for any rotation matrix $\Rb$,
\begin{equation*}
\Fb(\Rb\Xb, \Rb\ab) = \Rb \Fb(\Xb, \ab) \Rb^\top,
\end{equation*}
then its expectation $\Mb = \EE_{\Xb, \ab}[\Fb(\Xb, \ab)]$ is a scalar multiple of the identity matrix.
\end{lemma}

\begin{proof}
We will show that $\Mb$ satisfies the condition of Lemma \ref{Schur's lemma}. Let $\Rb$ be an arbitrary rotation matrix.
\begin{enumerate}
    \item \textbf{Transform the expectation}: By the linearity of expectation,
    \begin{equation*}
    \Rb\Mb\Rb^\top = \Rb \left( \EE_{\Xb, \ab}[\Fb(\Xb, \ab)] \right) \Rb^\top = \EE_{\Xb, \ab}[\Rb\Fb(\Xb, \ab)\Rb^\top].
    \end{equation*}
    \item \textbf{Change of variables}: Let us define new random variables $\Xb' = \Rb\Xb$ and $\ab' = \Rb\ab$. Due to the rotational invariance of the distributions, the expectation over the original variables is equal to the expectation over the new (primed) variables:
    \begin{equation*}
    \Mb = \EE_{\Xb, \ab}[\Fb(\Xb, \ab)] = \EE_{\Xb', \ab'}[\Fb(\Xb', \ab')].
    \end{equation*}
    \item \textbf{Apply the equivariance property}: By the hypothesis of this proposition, we know $\Fb(\Xb', \ab') = \Fb(\Rb\Xb, \Rb\ab) = \Rb\Fb(\Xb, \ab)\Rb^\top$. Substituting this into the equation above:
    \begin{equation*}
    \Mb = \EE_{\Xb', \ab'}[\Rb\Fb(\Xb, \ab)\Rb^\top].
    \end{equation*}
    Since the distributions of $(\Xb', \ab')$ and $(\Xb, \ab)$ are identical, we can switch the expectation back to be over the original variables:
    \begin{equation*}
    \Mb = \EE_{\Xb, \ab}[\Rb\Fb(\Xb, \ab)\Rb^\top].
    \end{equation*}
    \item \textbf{Conclude}: By comparing the first and last equations, we have established that $\Mb = \Rb\Mb\Rb^\top$. Since this holds for any rotation matrix $\Rb$, we conclude by Lemma \ref{Schur's lemma} that $\Mb = c\Ib_d$ for some scalar $c$.
\end{enumerate}
\end{proof}
This lemma is our main tool. The rest of the proof involves showing that the relevant gradient expressions satisfy the rotational equivariance condition.

\begin{lemma}
Under Assumption \ref{assump:distribution}, for any step $t \ge 0$ of gradient descent, the parameter matrices will have the following structure for some scalar coefficients $\alpha_t$ and $\beta_t$:
\begin{itemize}
    \item $\Wb_{12}^{(t)} = \alpha_t \Ib$ and $\Vb_{21}^{(t)} = \beta_t \Ib$.
    \item All other blocks are zero matrices.
\end{itemize}
\end{lemma}

\begin{proof}
We use mathematical induction.

\paragraph{Base case ($t=0 \to t=1$):}
At $t=0$, the parameters are initialized as $\alpha_0=0, \beta_0=1$. The forward pass gives $\pb^{(0)}=\mathbf{0}$ and $\hb^{(0)}=\ab$. 
\begin{align*}
\pb^{(0)} &= \Vb_{11}^{(0)}\Sb\Wb_{12}^{(0)}\ab + \Vb_{12}^{(0)}\ab\ab^\top\Wb_{22}^{(0)}\ab = \mathbf{0}, \\
\hb^{(0)} &= \ab + \Vb_{21}^{(0)}\Sb\Wb_{12}^{(0)}\ab + \Vb_{22}^{(0)}\ab\ab^\top\Wb_{22}^{(0)}\ab = \ab + (\beta_0 \Ib)\Sb(\alpha_0 \Ib)\ab = \ab.
\end{align*}
Below, we calculate the gradients for $\Wb$ and  $\Vb$ at $t=0$, respectively:
\begin{align*}
 \frac{\partial \mathcal{L}(\btheta)}{\partial \Wb_{11}}\bigg|_{t=0} &= \mathbf{0},\\
    \frac{\partial \mathcal{L}(\btheta)}{\partial \Wb_{12}}\bigg|_{t=0}
    &=\mathbb{E}_{\Xb, \ab} \left[ \Sb^\top \left( \Vb_{21}^{(0)\top} \frac{1}{r^{(0)}}\left(\Ib - \frac{\hb^{(0)}\hb^{(0)^\top}}{(r^{(0)})^2}\right) - \Vb_{11}^{(0)\top} \frac{\pb^{(0)}\hb^{(0)\top}}{(r^{(0)})^3} \right) 2(\mathbf{y}^{(0)} - \mathbf{y}_{\text{target}}) \ab^\top \right]\\
    &=\EE_{\Xb, \ab} \left[ \frac{2}{\|\Sb\ab\|_2} \Sb \left( (\ab^\top\Sb\ab)\ab - \Sb\ab \right) \ab^\top \right],\\
    \frac{\partial \mathcal{L}(\btheta)}{\partial \Wb_{21}} \bigg|_{t=0}&= \mathbf{0},\\
    \frac{\partial \mathcal{L}(\btheta)}{\partial \Wb_{22}} \bigg|_{t=0}
    &= \mathbb{E}_{\Xb, \ab} \left[ \ab\ab^\top \left( \Vb_{22}^{(0)\top} \frac{1}{r^{(0)}}\left(\Ib - \frac{\hb^{(0)}\hb^{(0)\top}}{(r^{(0)})^2}\right) - \Vb_{12}^{(0)\top} \frac{\pb^{(0)}\hb^{(0)\top}}{(r^{(0)})^3} \right) 2(\mathbf{y}^{(0)} - \mathbf{y}_{\text{target}}) \ab^\top \right]\\&= \mathbf{0},\\
\frac{\partial \mathcal{L}(\btheta)}{\partial \Vb_{11}} \bigg|_{t=0}&=\mathbb{E}_{\Xb, \ab} \left[ \left(-\frac{\pb^{(0)}\hb^{(0)\top}}{(r^{(0)})^3} 2(\mathbf{y}^{(0)} - \mathbf{y}_{\text{target}})\right) (\Sb\Wb_{12}^{(0)}\ab)^\top \right]\\&= \mathbf{0}, \\
\frac{\partial \mathcal{L}(\btheta)}{\partial \Vb_{12}}\bigg|_{t=0} &=\mathbb{E}_{\Xb, \ab} \left[ \left(-\frac{\pb^{(0)}\hb^{(0)\top}}{(r^{(0)})^3} 2(\mathbf{y}^{(0)} - \mathbf{y}_{\text{target}})\right) (\ab\ab^\top\Wb_{22}^{(0)}\ab)^\top \right]\\&= \mathbf{0}, \\
\frac{\partial \mathcal{L}(\btheta)}{\partial \Vb_{21}} \bigg|_{t=0}&=\mathbb{E}_{\Xb, \ab} \left[ \left(\frac{1}{r^{(0)}}\left(\Ib - \frac{\hb^{(0)}\hb^{(0)\top}}{(r^{(0)})^2}\right) 2(\mathbf{y}^{(0)} - \mathbf{y}_{\text{target}})\right) (\Sb\Wb_{12}^{(0)}\ab)^\top \right]\\&= \mathbf{0}, \\
\frac{\partial \mathcal{L}(\btheta)}{\partial \Vb_{22}} \bigg|_{t=0}&=\mathbb{E}_{\Xb, \ab} \left[ \left(\frac{1}{r^{(0)}}\left(\Ib - \frac{\hb^{(0)}\hb^{(0)\top}}{(r^{(0)})^2}\right) 2(\mathbf{y}^{(0)} - \mathbf{y}_{\text{target}})\right) (\ab\ab^\top\Wb_{22}^{(0)}\ab)^\top \right]\\&= \mathbf{0}.
\end{align*}
It follows that $\Wb_{12}$ is the only block with a non-zero gradient at $t=0$. Explicitly, this gradient is given by:
\begin{equation*}
\frac{\partial \mathcal{L}(\btheta)}{\partial \Wb_{12}} \bigg|_{t=0} = \EE_{\Xb, \ab} \left[ \frac{2}{\|\Sb\ab\|_2} \Sb \left( (\ab^\top\Sb\ab)\ab - \Sb\ab \right) \ab^\top \right].
\end{equation*}
We now verify that the term inside the expectation satisfies the conditions of Lemma \ref{Appendix Symmetry Proposition} regarding rotational equivariance. Let us define: 
\begin{equation*}  
\Fb_0(\Xb, \ab) = \frac{2}{\|\Sb\ab\|_2} \Sb ( (\ab^\top\Sb\ab)\ab - \Sb\ab ) \ab^\top. \end{equation*}
%We define $\Xb'=\Rb\Xb, \ab'=\Rb\ab$. We use the facts that $\Sb'=\Rb\Sb\Rb^\top$, $\|\Sb'\ab'\|_2=\|\Sb\ab\|_2$, and $(\ab')^\top\Sb'\ab' = \ab^\top\Sb\ab$.
Consider an arbitrary rotation matrix $\Rb$ and define the transformed inputs $\Xb'=\Rb\Xb$ and $\ab'=\Rb\ab$. Utilizing the identities $\Sb'=\Rb\Sb\Rb^\top$, $\|\Sb'\ab'\|_2=\|\Sb\ab\|_2$, and $(\ab')^\top\Sb'\ab' = \ab^\top\Sb\ab$, we have:
%\begin{align*}
    %\Fb_0(\Xb', \ab') &= \frac{2}{\|\Sb'\ab'\|_2} \Sb' ( (\ab'^\top\Sb'\ab')\ab' - \Sb'\ab' ) \ab'^\top \\
    %&= \frac{2}{\|\Sb\ab\|_2} (\Rb\Sb\Rb^\top) ( (\ab^\top\Sb\ab)(\Rb\ab) - \Rb\Sb\ab ) (\Rb\ab)^\top \\
    %&= \frac{2}{\|\Sb\ab\|_2} (\Rb\Sb\Rb^\top) \Rb ( (\ab^\top\Sb\ab)\ab - \Sb\ab ) \ab^\top\Rb^\top \\
    %&= \Rb \left( \frac{2}{\|\Sb\ab\|_2} \Sb ( (\ab^\top\Sb\ab)\ab - \Sb\ab ) \ab^\top \right) \Rb^\top \\&= \Rb\Fb_0(\Xb, \ab)\Rb^\top.
%\end{align*}
\begin{align*}
    \Fb_0(\Xb', \ab') &= \frac{2}{\|\Sb'\ab'\|_2} \Sb' \Big( (\ab'^\top\Sb'\ab')\ab' - \Sb'\ab' \Big) \ab'^\top \\
    &= \frac{2}{\|\Sb\ab\|_2} (\Rb\Sb\Rb^\top) \Big( (\ab^\top\Sb\ab)(\Rb\ab) - \Rb\Sb\ab \Big) (\Rb\ab)^\top \\
    &= \frac{2}{\|\Sb\ab\|_2} (\Rb\Sb\Rb^\top) \Rb \Big( (\ab^\top\Sb\ab)\ab - \Sb\ab \Big) \ab^\top\Rb^\top \\
    &= \Rb \left[ \frac{2}{\|\Sb\ab\|_2} \Sb \Big( (\ab^\top\Sb\ab)\ab - \Sb\ab \Big) \ab^\top \right] \Rb^\top \\
    &= \Rb\Fb_0(\Xb, \ab)\Rb^\top.
\end{align*}
This confirms that $\Fb_0$ satisfies the rotational equivariance condition of Lemma \ref{Appendix Symmetry Proposition}. Consequently, the gradient is isotropic, meaning $\frac{\partial \mathcal{L}}{\partial \Wb_{12}} \big|_{t=0} = c_0 \Ib$ for some scalar constant $c_0$. Thus, one step of gradient descent yields the following updates for $t=1$:
%\begin{itemize}
    %\item $\Wb_{12}^{(1)} = \mathbf{0} - \eta c_0 \Ib = \alpha_1 \Ib$.
    %\item $\Vb_{21}^{(1)} = \Ib - \eta \cdot \mathbf{0} = \beta_1 \Ib$ (with $\beta_1=1$).
    %\item All other blocks remain zero.
%\end{itemize}
\begin{itemize}
    \item $\Wb_{12}^{(1)} = \Wb_{12}^{(0)} - \eta c_0 \Ib = -\eta c_0 \Ib = \alpha_1 \Ib$.
    \item $\Vb_{21}^{(1)} = \Vb_{21}^{(0)} - \eta \cdot \mathbf{0} = \Ib = \beta_1 \Ib$ (where $\beta_1=1$).
    \item All other blocks remain zero.
\end{itemize}
Thus, the base case for the induction holds at $t=1$.

\paragraph{Inductive hypothesis:} %Assume that at the end of step $t$, the lemma's statement holds for scalars $\alpha_t, \beta_t$. 
Assume that at the end of step $t$, the block structure holds, i.e., $\Wb^{(t)}_{12}=\alpha_t \Ib_d$ and $\Vb^{(t)}_{21}=\beta_t \Ib_d$, while all other blocks are zero.

\paragraph{Inductive step ($t \to t+1$):}
We now demonstrate that this structure is preserved at step $t+1$.
First, we compute the forward pass variables at step $t$ under the inductive hypothesis:
\begin{align*}
\pb^{(t)} &= \Vb_{11}^{(t)}\Sb\Wb_{12}^{(t)}\ab + \Vb_{12}^{(t)}\ab\ab^\top\Wb_{22}^{(t)}\ab = \mathbf{0}, \\
\hb^{(t)} &= \ab + \Vb_{21}^{(t)}\Sb\Wb_{12}^{(t)}\ab + \Vb_{22}^{(t)}\ab\ab^\top\Wb_{22}^{(t)}\ab = \ab + (\beta_t \Ib)\Sb(\alpha_t \Ib)\ab = \ab + \alpha_t \beta_t \Sb\ab.
\end{align*}
%We now show that the statement also holds for step $t+1$.
Below, we calculate the gradients for $\Wb$ and  $\Vb$ at $t$, respectively:
\begin{align*}
 \frac{\partial \mathcal{L}(\btheta)}{\partial \Wb_{11}}\bigg|_t &= \mathbf{0},\\
    \frac{\partial \mathcal{L}(\btheta)}{\partial \Wb_{12}}\bigg|_t
    &=\mathbb{E}_{\Xb, \ab} \left[ \Sb^\top \left( \Vb_{21}^{(t)\top} \frac{1}{r^{(t)}}\left(\Ib - \frac{\hb^{(t)}\hb^{(t)^\top}}{(r^{(t)})^2}\right) - \Vb_{11}^{(t)\top} \frac{\pb^{(t)}\hb^{(t)\top}}{(r^{(t)})^3} \right) 2(\mathbf{y}^{(t)} - \mathbf{y}_{\text{target}}) \ab^\top \right]\\
    &=\beta_t \cdot \EE \left[ \Sb \frac{2}{r^{(t)}}\left(\Ib - \frac{\hb^{(t)}\hb^{(t)\top}}{(r^{(t)})^2}\right) (\mathbf{y}^{(t)} - \mathbf{y}_{\text{target}}) \ab^\top \right],\\
    \frac{\partial \mathcal{L}(\btheta)}{\partial \Wb_{21}} \bigg|_t&= \mathbf{0},\\
    \frac{\partial \mathcal{L}(\btheta)}{\partial \Wb_{22}} \bigg|_t
    &= \mathbb{E}_{\Xb, \ab} \left[ \ab\ab^\top \left( \Vb_{22}^{(t)\top} \frac{1}{r^{(t)}}\left(\Ib - \frac{\hb^{(t)}\hb^{(t)\top}}{(r^{(t)})^2}\right) - \Vb_{12}^{(t)\top} \frac{\pb^{(t)}\hb^{(t)\top}}{(r^{(t)})^3} \right) 2(\mathbf{y}^{(t)} - \mathbf{y}_{\text{target}}) \ab^\top \right]\\&= \mathbf{0},
    \end{align*}
    \begin{align*}
\frac{\partial \mathcal{L}(\btheta)}{\partial \Vb_{11}} \bigg|_t&=\mathbb{E}_{\Xb, \ab} \left[ \left(-\frac{\pb^{(t)}\hb^{(t)\top}}{(r^{(t)})^3} 2(\mathbf{y}^{(t)} - \mathbf{y}_{\text{target}})\right) (\Sb\Wb_{12}^{(t)}\ab)^\top \right]\\&= \mathbf{0}, \\
\frac{\partial \mathcal{L}(\btheta)}{\partial \Vb_{12}}\bigg|_t &=\mathbb{E}_{\Xb, \ab} \left[ \left(-\frac{\pb^{(t)}\hb^{(t)\top}}{(r^{(t)})^3} 2(\mathbf{y}^{(t)} - \mathbf{y}_{\text{target}})\right) (\ab\ab^\top\Wb_{22}^{(t)}\ab)^\top \right]\\&= \mathbf{0}, \\
\frac{\partial \mathcal{L}(\btheta)}{\partial \Vb_{21}} \bigg|_t&=\mathbb{E}_{\Xb, \ab} \left[ \left(\frac{1}{r^{(t)}}\left(\Ib - \frac{\hb^{(t)}\hb^{(t)\top}}{(r^{(t)})^2}\right) 2(\mathbf{y}^{(t)} - \mathbf{y}_{\text{target}})\right) (\Sb\Wb_{12}^{(t)}\ab)^\top \right]\\&= \alpha_t \cdot \EE \left[ \frac{2}{r^{(t)}}\left(\Ib - \frac{\hb^{(t)}\hb^{(t)\top}}{(r^{(t)})^2}\right) (\mathbf{y}^{(t)} - \mathbf{y}_{\text{target}}) (\Sb\ab)^\top \right], \\
\frac{\partial \mathcal{L}(\btheta)}{\partial \Vb_{22}} \bigg|_t&=\mathbb{E}_{\Xb, \ab} \left[ \left(\frac{1}{r^{(t)}}\left(\Ib - \frac{\hb^{(t)}\hb^{(t)\top}}{(r^{(t)})^2}\right) 2(\mathbf{y}^{(t)} - \mathbf{y}_{\text{target}})\right) (\ab\ab^\top\Wb_{22}^{(t)}\ab)^\top \right]\\&= \mathbf{0}.
\end{align*}

It is evident that the only non-zero gradients are those for $\Wb_{12}$ and $\Vb_{21}$. Consequently, all other blocks remain zero after the gradient update.

To simplify the gradient expressions, let us define: \begin{equation*}
\mathbf{g}_{\hb}^{(t)}=\frac{2}{r^{(t)}}\left(\Ib - \frac{\hb^{(t)}\hb^{(t)\top}}{(r^{(t)})^2}\right) (\mathbf{y}^{(t)} - \mathbf{y}_{\text{target}}).\end{equation*}
Now, consider the gradients for $\Vb_{21}$ and $\Wb_{12}$:
\begin{align*}
    \frac{\partial \mathcal{L}(\btheta)}{\partial \Vb_{21}} \bigg|_{t} &= \EE \left[ \mathbf{g}_{\hb}^{(t)} (\Sb\Wb_{12}^{(t)}\ab)^\top \right] = \alpha_t \cdot \EE \left[ \mathbf{g}_{\hb}^{(t)} (\Sb\ab)^\top \right], \\
    \frac{\partial \mathcal{L}(\btheta)}{\partial \Wb_{12}} \bigg|_{t} &= \EE \left[ \Sb \Vb_{21}^{(t)\top} \mathbf{g}_{\hb}^{(t)} \ab^\top \right] = \beta_t \cdot \EE \left[ \Sb \mathbf{g}_{\hb}^{(t)} \ab^\top \right].
\end{align*}
Let the expressions inside the expectations be $\Fb_{vt}(\Xb, \ab)$ and $\Fb_{wt}(\Xb, \ab)$. 
%We first show that $\mathbf{g}_{\hb}^{(t)}$ has rotational covariance, i.e., $\mathbf{g}_{\hb}^{(t)}(\Xb', \ab') = \Rb \mathbf{g}_{\hb}^{(t)}(\Xb, \ab)$. We have shown that $\hb^{(t)'} = \Rb\hb^{(t)}$. It follows that $r^{(t)'}=r^{(t)}$, $\mathbf{y}^{(t)'} = \Rb\mathbf{y}^{(t)}$, $\mathbf{y}_{\text{target}}'=\Rb\mathbf{y}_{\text{target}}$.
%\begin{align*}
   % \mathbf{g}_{\hb}^{(t)'} &= \frac{2}{r^{(t)'}}\left(\Ib - \frac{\hb^{(t)'}\hb^{(t)'\top}}{(r^{(t)'})^2}\right)(\mathbf{y}^{(t)'}-\mathbf{y}_{\text{target}}') \\
   % &= \frac{2}{r^{(t)}}\left(\Ib - \Rb\frac{\hb^{(t)}\hb^{(t)\top}}%{r^{(t)2}}\Rb^\top\right)\Rb(\mathbf{y}^{(t)}-\mathbf{y}_{\text{target}}) \\&= \Rb \left( \frac{2}{r^{(t)}}\left(\Ib - \frac{\hb^{(t)}\hb^{(t)\top}}{r^{(t)2}}\right) (\mathbf{y}^{(t)}-\mathbf{y}_{\text{target}}) \right) \\&= \Rb \mathbf{g}_{\hb}^{(t)}.
%\end{align*}
%Now we can check the equivariance of $\Fb_{wt}(\Xb, \ab) = \beta_t \Sb \mathbf{g}_{\hb}^{(t)} \ab^\top$.
%\begin{equation*}
%\Fb_{wt}(\Xb', \ab') = \beta_t \Sb' \mathbf{g}_{\hb}^{(t)'} (\ab')^\top = \beta_t (\Rb\Sb\Rb^\top)(\Rb\mathbf{g}_{\hb}^{(t)})(\Rb\ab)^\top = \Rb(\beta_t \Sb \mathbf{g}_{\hb}^{(t)} \ab^\top)\Rb^\top = \Rb\Fb_{wt}(\Xb, \ab)\Rb^\top.
%\end{equation*}
%The proof for $\Fb_{vt}$ is analogous. 
First, we establish the equivariance of $\mathbf{g}_{\hb}^{(t)}$. Consider an arbitrary rotation matrix $\Rb$. Since $\hb^{(t)'} = \Rb\hb^{(t)}$, it follows that $r^{(t)'}=r^{(t)}$, $\mathbf{y}^{(t)'} = \Rb\mathbf{y}^{(t)}$, and $\mathbf{y}_{\text{target}}'=\Rb\mathbf{y}_{\text{target}}$. Thus:
\begin{align*}
    \mathbf{g}_{\hb}^{(t)'} &= \frac{2}{r^{(t)'}}\left(\Ib - \frac{\hb^{(t)'}\hb^{(t)'\top}}{(r^{(t)'})^2}\right)(\mathbf{y}^{(t)'}-\mathbf{y}_{\text{target}}') \\
    &= \frac{2}{r^{(t)}}\left(\Ib - \Rb\frac{\hb^{(t)}\hb^{(t)\top}}{(r^{(t)})^2}\Rb^\top\right)\Rb(\mathbf{y}^{(t)}-\mathbf{y}_{\text{target}}) \\
    &= \Rb \left[ \frac{2}{r^{(t)}}\left(\Ib - \frac{\hb^{(t)}\hb^{(t)\top}}{(r^{(t)})^2}\right) (\mathbf{y}^{(t)}-\mathbf{y}_{\text{target}}) \right] \\&= \Rb \mathbf{g}_{\hb}^{(t)}.
\end{align*}
Using this property, we check $\Fb_{wt}(\Xb, \ab)$:
\begin{equation*}
\Fb_{wt}(\Xb', \ab') = \beta_t \Sb' \mathbf{g}_{\hb}^{(t)'} (\ab')^\top = \beta_t (\Rb\Sb\Rb^\top)(\Rb\mathbf{g}_{\hb}^{(t)})(\Rb\ab)^\top = \Rb(\beta_t \Sb \mathbf{g}_{\hb}^{(t)} \ab^\top)\Rb^\top = \Rb\Fb_{wt}(\Xb, \ab)\Rb^\top.
\end{equation*}
A similar argument applies to $\Fb_{vt}(\Xb, \ab)$.
Therefore, by Lemma \ref{Appendix Symmetry Proposition}, both gradients are isotropic:
\begin{equation*}
\frac{\partial \mathcal{L}(\btheta)}{\partial \Vb_{21}} \bigg|_{t} = c_{vt} \cdot \Ib_d \quad \text{and} \quad \frac{\partial \mathcal{L}(\btheta)}{\partial \Wb_{12}} \bigg|_{t} = c_{wt} \cdot \Ib_d,
\end{equation*}
for some scalars $c_{vt}$ and $c_{wt}$. The parameter updates for step $t+1$ are:
\begin{align*}
\Vb_{21}^{(t+1)} &= \Vb_{21}^{(t)} - \eta (c_{vt} \Ib) = (\beta_t - \eta c_{vt}) \Ib \equiv \beta_{t+1} \Ib, \\
\Wb_{12}^{(t+1)} &= \Wb_{12}^{(t)} - \eta (c_{wt} \Ib) = (\alpha_t - \eta c_{wt}) \Ib \equiv \alpha_{t+1} \Ib.
\end{align*}
This confirms that the block structure is preserved at step $t+1$, completing the inductive proof.
\end{proof}

\subsection{Evolution of \texorpdfstring{$\gamma_t=\alpha_t\beta_t$}{gammat=alphat betat}}
This section characterizes the training dynamics of the product term $\gamma_t = \alpha_t\beta_t$ under gradient descent. The analysis proceeds in two main steps:
\begin{enumerate}
   \item First, we establish the qualitative behavior of the sequence $\{\gamma_t\}$. We prove that it is monotonically increasing and, by contradiction, that it diverges to positive infinity.
    \item Second, we determine the precise growth rates. We show that the sequence of cubic differences, $\{\gamma_{t+1}^3 - \gamma_t^3\}$, converges to a positive constant. This key result implies that $\gamma_t = \Theta((\eta\Upsilon_3t/d)^{1/3})$. Furthermore, by proving that the ratio $\alpha_t/\beta_t$ converges to 1, we deduce the individual growth rates for the original parameters: $\alpha_t = \Theta((\eta\Upsilon_3t/d)^{1/6})$ and $\beta_t = \Theta((\eta\Upsilon_3t/d)^{1/6})$.
\end{enumerate}
\subsubsection{Divergence of \texorpdfstring{$\gamma_t$}{gammat}}\label{first dominated theorem}

The proof will proceed in four sequential, detailed steps:
\begin{enumerate}
    \item Derive symmetric iterative formulas for $\alpha_t$ and $\beta_t$ by defining a common gradient factor, $F_t$, that captures the essential dynamics.
    \item Prove that this common factor $F_t$ is non-positive.
    \item Show that the sequence $\{\gamma_t\}$ is monotonically increasing.
    \item Finally, by contradiction, we demonstrate that $\{\gamma_t\}$ cannot be bounded from above and must therefore diverge to infinity.
\end{enumerate}

\paragraph{Step 1: Derivation of the symmetric iterative formulas}

Applying the trace to the update rules for the scalar parameters $\alpha_t$ and $\beta_t$, we have:
\begin{align*}
    \alpha_{t+1} &= \alpha_t - \frac{\eta}{d} \tr\left( \grad{\mathcal{L}(\btheta)}{\Wb_{12}}\bigg|_t \right), \\
    \beta_{t+1} &= \beta_t - \frac{\eta}{d} \tr\left( \grad{\mathcal{L}(\btheta)}{\Vb_{21}}\bigg|_t \right).
\end{align*}
Let us analyze the trace terms.
For the $\alpha_t$ update, the trace term is:
\begin{align*}
    \tr\left( \grad{\mathcal{L}(\btheta)}{\Wb_{12}}\bigg|_t \right) &= \tr\left( \EE \left[ \Sb \Vb_{21}^{(t)\top} \mathbf{g}_{\hb}^{(t)} \ab^\top \right] \right) \\
    &= \tr\left( \EE \left[ \Sb (\beta_t \Ib) \mathbf{g}_{\hb}^{(t)} \ab^\top \right] \right) \\
    &= \beta_t \cdot \EE \left[ \tr \left( \Sb \mathbf{g}_{\hb}^{(t)} \ab^\top \right) \right] && \text{(Linearity of trace and expectation)} \\
    &= \beta_t \cdot \EE \left[ \ab^\top \Sb \mathbf{g}_{\hb}^{(t)} \right] && \text{(Cyclic property of trace)}.
\end{align*}

For the $\beta_t$ update, the trace term is:
\begin{align*}
    \tr\left( \grad{\mathcal{L}(\btheta)}{\Vb_{21}}\bigg|_t \right) &= \tr\left( \EE \left[ \mathbf{g}_{\hb}^{(t)} (\Sb\Wb_{12}^{(t)}\ab)^\top \right] \right) \\
    &= \tr\left( \EE \left[ \mathbf{g}_{\hb}^{(t)} (\alpha_t \Sb \ab)^\top \right] \right) \\
    &= \alpha_t \cdot \EE \left[ \tr \left( \mathbf{g}_{\hb}^{(t)} \ab^\top \Sb^\top \right) \right] \\
    &= \alpha_t \cdot \EE \left[ \tr \left( \Sb \mathbf{g}_{\hb}^{(t)} \ab^\top \right) \right] && \text{(Symmetry of } \Sb \text{ and cyclic property)} \\
    &= \alpha_t \cdot \EE \left[ \ab^\top \Sb \mathbf{g}_{\hb}^{(t)} \right] && \text{(Cyclic property again)}.
\end{align*}

We see that both trace calculations hinge on the exact same scalar expectation, $\EE \left[ \ab^\top \Sb \mathbf{g}_{\hb}^{(t)} \right]$. This confirms the underlying symmetry of the dynamics. We can now define this common factor.

Let $F_t$ be the scalar defined as:
\begin{equation*} F_t \equiv \frac{1}{d} \EE \left[ \ab^\top \Sb \mathbf{g}_{\hb}^{(t)} \right]. \end{equation*}

Substituting this definition back into the update rules yields:
\begin{align*}
    \alpha_{t+1} &= \alpha_t - \frac{\eta}{d} \left( \beta_t \cdot \EE \left[ \ab^\top \Sb \mathbf{g}_{\hb}^{(t)} \right] \right) = \alpha_t - \frac{\eta}{d} \left( \beta_t \cdot (d F_t) \right) = \alpha_t - \eta \beta_t F_t, \\
    \beta_{t+1} &= \beta_t - \frac{\eta}{d} \left( \alpha_t \cdot \EE \left[ \ab^\top \Sb \mathbf{g}_{\hb}^{(t)} \right] \right) = \beta_t - \frac{\eta}{d} \left( \alpha_t \cdot (d F_t) \right) = \beta_t - \eta \alpha_t F_t.
\end{align*}

This yields the desired symmetric system of iterative equations:
\begin{align}
\alpha_{t+1} &= \alpha_t - \eta \beta_t F_t \label{eq:alpha_update}, \\
\beta_{t+1} &= \beta_t - \eta \alpha_t F_t \label{eq:beta_update}.
\end{align}

\paragraph{Step 2: Common factor \texorpdfstring{$F_t \leq 0$}{Ft <= 0}}
From Step 1, we know
\begin{align*}
d \cdot F_t = \EE \left[ \ab^\top \Sb \mathbf{g}_{\hb}^{(t)} \right].
\end{align*}
Let us analyze the expression inside the expectation. Substituting the definition of $\mathbf{g}_{\hb}^{(t)}$:
\begin{equation*}
\mathbf{g}_{\hb}^{(t)} = \frac{2}{r^{(t)}}\left(\Ib - \frac{\hb^{(t)}\hb^{(t)\top}}{(r^{(t)})^2}\right) (\mathbf{y}^{(t)} - \mathbf{y}_{\text{target}}) = \frac{2}{r^{(t)}}\left(\Ib - \mathbf{y}^{(t)}\mathbf{y}^{(t)\top}\right) (\mathbf{y}^{(t)} - \mathbf{y}_{\text{target}}).
\end{equation*}
Note that the orthogonal projection operator $(\Ib - \mathbf{y}^{(t)}\mathbf{y}^{(t)\top})$ maps any vector collinear with $\mathbf{y}^{(t)}$ to zero. Specifically:
\begin{equation*}
\left(\Ib - \mathbf{y}^{(t)}\mathbf{y}^{(t)\top}\right) \mathbf{y}^{(t)} = \mathbf{y}^{(t)} - \mathbf{y}^{(t)}(\|\mathbf{y}^{(t)}\|_2^2) = \mathbf{y}^{(t)} - \mathbf{y}^{(t)} = \mathbf{0}.
\end{equation*}
Thus, the expression for $\mathbf{g}_{\hb}^{(t)}$ can be simplified:
\begin{equation*}
\mathbf{g}_{\hb}^{(t)} = -\frac{2}{r^{(t)}}\left(\Ib - \mathbf{y}^{(t)}\mathbf{y}^{(t)\top}\right) \mathbf{y}_{\text{target}}.
\end{equation*}
Substituting this simplified expression back into the term we are analyzing:
\begin{equation*}
\ab^\top \Sb \mathbf{g}_{\hb}^{(t)} = -\frac{2}{r^{(t)}} \ab^\top \Sb \left(\Ib - \mathbf{y}^{(t)}\mathbf{y}^{(t)\top}\right) \mathbf{y}_{\text{target}}.
\end{equation*}
Expanding the term in the parenthesis:
\begin{equation*}
\ab^\top \Sb \mathbf{g}_{\hb}^{(t)} = -\frac{2}{r^{(t)}} \left( \ab^\top \Sb \mathbf{y}_{\text{target}} - \ab^\top \Sb \mathbf{y}^{(t)} \mathbf{y}^{(t)\top} \mathbf{y}_{\text{target}} \right).
\end{equation*}
To simplify the analysis, we define an auxiliary vector $\ub = \Sb\ab$. By definition, the target vector is 
\begin{equation*}
    \mathbf{y}_{\text{target}} = \frac{\Sb\ab}{\|\Sb\ab\|_2} = \frac{\ub}{\|\ub\|_2}.
\end{equation*}
Substituting $\ub$ and $\mathbf{y}_{\text{target}}$ into the equation above:
\begin{align*}
\ab^\top \Sb \mathbf{g}_{\hb}^{(t)} &= -\frac{2}{r^{(t)}} \left( \ub^\top \frac{\ub}{\|\ub\|_2} - \ub^\top \mathbf{y}^{(t)} \mathbf{y}^{(t)\top} \frac{\ub}{\|\ub\|_2} \right) \\
&= -\frac{2}{r^{(t)}\|\ub\|_2} \left( \|\ub\|_2^2 - (\ub^\top \mathbf{y}^{(t)}) (\mathbf{y}^{(t)\top} \ub) \right) \\
&= -\frac{2}{r^{(t)}\|\ub\|_2} \left( \|\ub\|_2^2 - (\ub^\top \mathbf{y}^{(t)})^2 \right).
\end{align*}
Now we apply the Cauchy-Schwarz inequality, which states that for any vectors $\ub$ and $\mathbf{y}^{(t)}$:
\begin{equation*}
(\ub^\top \mathbf{y}^{(t)})^2 \le \|\ub\|_2^2 \|\mathbf{y}^{(t)}\|_2^2.
\end{equation*}
Since $\|\mathbf{y}^{(t)}\|_2^2 = 1$, the inequality becomes:
\begin{equation*}
(\ub^\top \mathbf{y}^{(t)})^2 \le \|\ub\|_2^2.
\end{equation*}
Thus, the expression $\left( \|\ub\|_2^2 - (\ub^\top \mathbf{y}^{(t)})^2 \right)$ is necessarily non-negative.
Equality holds if and only if $\mathbf{y}^{(t)}$ is collinear with $\ub$ (i.e., with $\Sb\ab$). Since $\|\mathbf{y}^{(t)}\|_2=1$, this means that $\mathbf{y}^{(t)}$ must be equal to $+ \mathbf{y}_{\text{target}}$ or $- \mathbf{y}_{\text{target}}$.
%\begin{itemize}
   % \item If $\mathbf{y}^{(t)} = \mathbf{y}_{\text{target}}$, the loss $L'$ is 0, and the gradient is 0.
   % \item If $\mathbf{y}^{(t)} = -\mathbf{y}_{\text{target}}$, the loss is maximal, and the gradient is also 0.
%\end{itemize}
Since $r^{(t)} > 0$ and $\|\ub\|_2 = \|\Sb\ab\|_2 > 0$, we can conclude that for any $t$:
\begin{equation*}
\ab^\top \Sb \mathbf{g}_{\hb}^{(t)} = -\frac{2}{r^{(t)}\|\ub\|_2} \underbrace{\left( \|\ub\|_2^2 - (\ub^\top \mathbf{y}^{(t)})^2 \right)}_{\geq 0} \leq 0.
\end{equation*}
Because this expression is non-positive for any $\Xb$ and $\ab$, its expectation must also be non-positive:
\begin{equation*}
\EE \left[ \ab^\top \Sb \mathbf{g}_{\hb}^{(t)} \right] \leq 0.
\end{equation*}

\paragraph{Step 3: \texorpdfstring{$\gamma_t$}{gammat} is monotonically increasing}

We analyze the change in $\gamma_t$ from one step to the next, $\Delta \gamma_t = \gamma_{t+1} - \gamma_t$. Using the symmetric update rules from (\ref{eq:alpha_update}) and (\ref{eq:beta_update}):
\begin{align*}
\gamma_{t+1} &= \alpha_{t+1}\beta_{t+1} = \left( \alpha_t - \eta \beta_t F_t \right) \left( \beta_t - \eta \alpha_t F_t \right) \\
&= \alpha_t\beta_t - \eta \alpha_t^2 F_t - \eta \beta_t^2 F_t + \eta^2 \alpha_t \beta_t F_t^2 \\
&= \gamma_t - \eta F_t (\alpha_t^2 + \beta_t^2) + \eta^2 \gamma_t F_t^2.
\end{align*}
The increment is therefore:
\begin{equation*}
\Delta \gamma_t = \eta \left[ -(\alpha_t^2 + \beta_t^2) F_t + \eta \gamma_t F_t^2 \right].
\end{equation*}
For $\Delta \gamma_t \geq 0$, we need the term in the brackets to be non-negative. We know $F_t \leq 0$, so the term $-(\alpha_t^2 + \beta_t^2) F_t$ is non-negative. The second term, $\eta \gamma_t F_t^2$, is also non-negative. Thus, their sum is non-negative for any $\eta>0$.
\begin{equation*}
\Delta \gamma_t \geq 0.
\end{equation*}
Therefore, $\{\gamma_t\}$ is a monotonically increasing sequence.

\paragraph{Step 4: \texorpdfstring{$\gamma_t \to \infty$}{gammat diverges}}

\begin{lemma}\label{lemma divergence}
The sequence $\{\gamma_t\}$ diverges to positive infinity.
\end{lemma}
\begin{proof}
We proceed by contradiction.

Assume that the sequence $\{\gamma_t\}$ does not diverge to infinity. Since we have established that $\{\gamma_t\}$ is monotonically increasing, if it does not diverge, it must be bounded from above and therefore must converge to a finite limit. Let this limit be $\gamma_\infty$.
\begin{equation*}
\lim_{t \to \infty} \gamma_t = \gamma_\infty < \infty.
\end{equation*}
For the sequence $\{\gamma_t\}$ to converge, the updates must eventually cease. This means the difference between consecutive terms must converge to zero:
\begin{equation*}
\lim_{t \to \infty} \Delta \gamma_t = \lim_{t \to \infty} (\gamma_{t+1} - \gamma_t) = 0.
\end{equation*}
Looking at the expression for $\Delta \gamma_t$:
\begin{equation*}
\lim_{t \to \infty} \left( - \eta F_t (\alpha_t^2 + \beta_t^2) + \eta^2 \gamma_t F_t^2 \right) = 0.
\end{equation*}
So
\begin{equation*}
\lim_{t \to \infty}  - \eta F_t (\alpha_t^2 + \beta_t^2)  = 0.
\end{equation*}
Since $\eta>0$ and $\alpha_t^2 + \beta_t^2\geq 1$, the only way for this limit to be zero is if the common gradient factor itself converges to zero.
\begin{equation*}
\lim_{t \to \infty} F_t = 0.
\end{equation*}
This implies that
\begin{equation*}
    \lim_{t\to \infty}\EE \left[ \ab^\top \Sb \mathbf{g}_{\hb}^{(t)} \right]=0.
\end{equation*}

Let $Z_t(\ab, \Xb) = -\ab^\top \Sb \mathbf{g}_{\hb}^{(t)}$ be the non-negative random variable inside the expectation. To infer the properties of the limit system, we must justify interchanging the limit and the expectation.

\paragraph{Justification for exchanging limit and expectation.}
We use the Dominated Convergence Theorem, which requires two conditions to be met.
\begin{enumerate}
    \item \textbf{Pointwise convergence:} For a fixed data sample $(\ab, \Xb)$, the term $Z_t$ is a continuous function of $\gamma_t$. Since we assumed $\gamma_t \to \gamma_\infty$, it follows that $Z_t$ converges pointwise to a limit $Z_\infty(\ab, \Xb) = -\ab^\top \Sb \mathbf{g}_{\hb}^{(\infty)}$.
    \item \textbf{Domination:} We need to find an integrable function $Y$, independent of $t$, such that $|Z_t| \le Y$. We have $|Z_t| \le \|\Sb\ab\|_2 \|\mathbf{g}_{\hb}^{(t)}\|_2$. The norm of the gradient is $||\mathbf{g}_{\hb}^{(t)}||_2 = ||-\frac{2}{r^{(t)}}(\Ib - \mathbf{y}^{(t)}\mathbf{y}^{(t)\top})\mathbf{y}_{\text{target}}||_2 \le \frac{2}{r^{(t)}}$, where $(r^{(t)})^2 = ||\ab + \gamma_t \Sb\ab||_2^2$. Since $\gamma_t \ge 0$ and $\Sb$ is positive semi-definite, $(r^{(t)})^2 \ge ||\ab||_2^2$, so $r^{(t)} \ge ||\ab||_2$. This gives a bound independent of $t$: $||\mathbf{g}_{\hb}^{(t)}||_2 \le \frac{2}{||\ab||_2}$. Therefore, $|Z_t| \le \|\Sb\ab\|_2 \frac{2}{||\ab||_2}$. We can set the dominating function $Y = 2\frac{\|\Sb\ab\|_2}{\|\ab\|_2}$. By Assumption \ref{assump:distribution} and Lemma \ref{lem: bounds for Sa}, we have
    \begin{equation*}
        \EE[Y]=2\EE\left[\frac{\|\Sb\ab\|_2}{\|\ab\|_2}\right]=2\EE\left[\|\Sb\ab\|_2\right]\leq\frac{2n}{\sqrt{d}}<\infty,
    \end{equation*}
    which ensures that $\EE[Y]$ is finite.
\end{enumerate}
With both conditions satisfied, we can interchange the limit and expectation:
\begin{equation*} \EE \left[ \lim_{t \to \infty} Z_t \right] = \lim_{t \to \infty} \EE [Z_t] = 0. \end{equation*}
%Since $\lim_{t\to \infty} Z_t$ is a non-negative function whose expectation is zero, it must be zero for almost all data samples. This means that for almost every sample, 
%\begin{equation*}\lim_{t \to \infty} \ab^\top \Sb\mathbf{g}_{\hb}^{(t)} = -\ab^\top \Sb \mathbf{g}_{\hb}^{(\infty)}=0.\end{equation*}    
%From the proof of Step 2, the result implies that almost surely
%\begin{equation*}
    %\mathbf{y}_\infty = \frac{\ab + \gamma_\infty \Sb\ab}{\|\ab + \gamma_\infty \Sb\ab\|_2}=\frac{\Sb\ab}{\|\Sb\ab\|_2},
%\end{equation*}
%or
%\begin{equation*}
    %\mathbf{y}_\infty = \frac{\ab + \gamma_\infty \Sb\ab}{\|\ab + \gamma_\infty \Sb\ab\|_2}=-\frac{\Sb\ab}{\|\Sb\ab\|_2}.
%\end{equation*}
%This collinearity implies that the vector $\ab + \gamma_\infty \Sb\ab$ must itself be collinear with $\Sb\ab$. This is only possible if the vector $\ab$ is also collinear with $\Sb\ab$, which occurs if and only if $\ab$ is an eigenvector of the matrix $\Sb$. Thus, we need $\ab$ to be an eigenvector of the matrix $\Sb$ almost surely.
Since $\lim_{t\to\infty} Z_t$ is non-negative and $\EE[\lim_{t\to\infty} Z_t]=0$, we must have
\begin{equation*}
\lim_{t\to\infty} Z_t = \lim_{t\to\infty}\bigl(-\ab^\top \Sb\gb_{\hb}^{(t)}\bigr)
= -\ab^\top \Sb\gb_{\hb}^{(\infty)} = 0
\quad\text{for almost every }(\ab,\Xb).
\end{equation*}
By Step~2, this implies that, almost surely,
\begin{equation*}
\yb_{\infty}
=
\frac{\ab+\gamma_\infty \Sb\ab}{\|\ab+\gamma_\infty \Sb\ab\|_2}
=
\sigma\frac{\Sb\ab}{\|\Sb\ab\|_2},
\qquad \sigma\in\{+1,-1\}.
\end{equation*}
Hence $\ab+\gamma_\infty \Sb\ab$ is collinear with $\Sb\ab$, i.e., there exists $c\in\RR$ such that
\begin{equation*}
\ab+\gamma_\infty \Sb\ab = c\Sb\ab
\Longrightarrow
\ab = (c-\gamma_\infty)\Sb\ab .
\end{equation*}
Whenever $\Sb\ab\neq \mathbf{0}$ (which holds almost surely under our assumptions), this shows that $\ab$ is collinear with $\Sb\ab$, hence is an eigenvector of $\Sb$:
\begin{equation*}
\Sb\ab = \lambda\ab \quad\text{for some }\lambda\in\RR.
\end{equation*}
Therefore, the above collinearity would require $\ab$ to be an eigenvector of $\Sb$ almost surely.

However, under Assumption~\ref{assump:distribution} and the independence of $\ab$ and $\Sb$, this event has probability zero: conditioning on $\Sb$, the eigenvector set on $\mathbb{S}^{d-1}$ is either finite (simple spectrum) or a lower-dimensional submanifold (eigenvalue multiplicity), hence has measure zero; the only degenerate case $\Sb\propto \Ib_d$ occurs with probability zero for random $\Sb=\Xb\Xb^\top$.

\paragraph{Final conclusion.}
The initial assumption—that $\{\gamma_t\}$ converges to a finite limit—is false. Since $\{\gamma_t\}$ is a monotonically increasing sequence, the only remaining possibility is that it is unbounded. Therefore, it must diverge to positive infinity.
\begin{equation*}
\lim_{t\to\infty} \gamma_t = \infty.
\end{equation*}
\end{proof}

\subsubsection{Growth Rate for \texorpdfstring{$\gamma_t$}{gammat}}\label{appendix cubic difference}

The proof starts from the established symmetric update rules for the parameters $\alpha_t$ and $\beta_t$:
\begin{align*}
\alpha_{t+1} &= \alpha_t - \eta \beta_t F_t \\
\beta_{t+1} &= \beta_t - \eta \alpha_t F_t
\end{align*}
where $F_t$ is the common factor. From these, the exact update rule for $\gamma_t$ is derived:
\begin{equation*}
\gamma_{t+1} = \gamma_t - \eta F_t (\alpha_t^2 + \beta_t^2) + \eta^2 \gamma_t F_t^2.
\end{equation*}
We also rely on the previously established rigorous result that $\gamma_t$ diverges to infinity. The proof proceeds in three steps.

\paragraph{Step 1: The exact increment of \texorpdfstring{$\gamma_t^3$}{gammat3}}
Let $\Delta_t = \gamma_{t+1} - \gamma_t$. The exact expression for $\Delta_t$ is:
\begin{equation*}
\Delta_t = - \eta F_t (\alpha_t^2 + \beta_t^2) + \eta^2 \gamma_t F_t^2.
\end{equation*}
We calculate the expansion for $(\gamma_t + \Delta_t)^3$:
\begin{align*}
\gamma_{t+1}^3 - \gamma_t^3 &= (\gamma_t + \Delta_t)^3 - \gamma_t^3 = 3\gamma_t^2\Delta_t + 3\gamma_t\Delta_t^2 + \Delta_t^3.
\end{align*}
Substituting the full expression for $\Delta_t$ yields the complete, exact expression for the increment:
\begin{align*}
\gamma_{t+1}^3 - \gamma_t^3 = & \ 3\gamma_t^2 \left( - \eta F_t (\alpha_t^2 + \beta_t^2) + \eta^2 \gamma_t F_t^2 \right) \\
& + 3\gamma_t \left( - \eta F_t (\alpha_t^2 + \beta_t^2) + \eta^2 \gamma_t F_t^2 \right)^2 \\
& + \left( - \eta F_t (\alpha_t^2 + \beta_t^2) + \eta^2 \gamma_t F_t^2 \right)^3.
\end{align*}

\paragraph{Step 2: Limit analysis of the increment as \texorpdfstring{$t \to \infty$}{t tends to infinity}}

To compute $\lim_{t \to \infty} (\gamma_{t+1}^3 - \gamma_t^3)$, we must first rigorously establish the limits of the key sequences that appear in the expression above.
\paragraph{Key Results.}
\begin{enumerate}
    \item[(1)] $\lim_{t \to \infty} \gamma_t = \infty$.
    \item[(2)] $\lim_{t \to \infty} \frac{\alpha_t^2 + \beta_t^2}{2\gamma_t} = 1$.
    \item[(3)] $\lim_{t \to \infty} \gamma_t^3 F_t = -C_1$, where $C_1$ is a strictly positive constant.
\end{enumerate}
\subparagraph{Proof of Result (1).}
Previously established.

\subparagraph{Proof of Result (2).}
First, we establish that both $\alpha_t$ and $\beta_t$ diverge to infinity. From the update rules, we can derive the quasi-conservation law \begin{align*}
\beta_{t+1}^2 - \alpha_{t+1}^2 &= (\beta_t^2 - \alpha_t^2) + (\eta^2 \alpha_t^2 F_t^2 - \eta^2 \beta_t^2 F_t^2) \\
&= (\beta_t^2 - \alpha_t^2) - \eta^2 F_t^2 (\beta_t^2 - \alpha_t^2) \\
&= (\beta_t^2 - \alpha_t^2) (1 - \eta^2 F_t^2).
\end{align*} Since $\beta_0^2 - \alpha_0^2 = 1$ and $F_t^2 \ge 0$, the sequence $\{\beta_t^2 - \alpha_t^2\}$ is non-increasing and bounded below by 0, thus it is a bounded sequence. We also have $\gamma_t = \alpha_t \beta_t \to \infty$. Therefore, both $\alpha_t$ and $\beta_t$ must diverge to infinity.

Now we prove the main result. Consider the ratio $\alpha_t/\beta_t$. We can write:
\begin{equation*}
1 - \left(\frac{\alpha_t}{\beta_t}\right)^2 = \frac{\beta_t^2 - \alpha_t^2}{\beta_t^2}.
\end{equation*}
Since the numerator, $\beta_t^2 - \alpha_t^2$, is a bounded sequence and the denominator $\beta_t^2 \to \infty$, the right-hand side converges to 0.
\begin{equation*}
\lim_{t \to \infty} \left( 1 - \left(\frac{\alpha_t}{\beta_t}\right)^2 \right) = 0 \implies \lim_{t \to \infty} \left(\frac{\alpha_t}{\beta_t}\right)^2 = 1.
\end{equation*}
For all $t$, $\alpha_t$ and $\beta_t$ have the same sign (both positive), so we can take the positive square root:
\begin{equation}\label{alpha/beta limit}
\lim_{t \to \infty} \frac{\alpha_t}{\beta_t} = 1.
\end{equation}
Now we can evaluate the limit in question:
\begin{equation*}
\lim_{t \to \infty} \frac{\alpha_t^2 + \beta_t^2}{2\gamma_t} = \lim_{t \to \infty} \frac{\alpha_t^2 + \beta_t^2}{2\alpha_t\beta_t} = \lim_{t \to \infty} \frac{1}{2}\left( \frac{\alpha_t}{\beta_t} + \frac{\beta_t}{\alpha_t} \right).
\end{equation*}
Since $\lim \alpha_t/\beta_t = 1$, we also have $\lim \beta_t/\alpha_t = 1$. By the algebra of limits:
\begin{equation*}
\lim_{t \to \infty} \frac{\alpha_t^2 + \beta_t^2}{2\gamma_t} = \frac{1}{2}(1 + 1) = 1. 
\end{equation*}

\subparagraph{Proof of Result (3).}
We have $F_t = \frac{1}{d} \EE[\ab^\top\Sb\mathbf{g}_\hb^{(t)}]$. Our goal is to find the limit of $\EE[\gamma_t^3 \ab^\top\Sb\mathbf{g}_\hb^{(t)}]$. We use the Dominated Convergence Theorem to swap the limit and expectation.
Let $Z_t(\ab, \Xb) = \gamma_t^3 \ab^\top\Sb\mathbf{g}_\hb^{(t)}$; we simply write $\Sb\ab$ as $\ub $.
\begin{enumerate}
    \item \textbf{Pointwise convergence:} For a fixed sample $(\ab, \Xb)$, we analyze $Z_t$. 
    \begin{align*}
        Z_t &= \gamma_t^3 \cdot \frac{-2}{r^{(t)}\|\ub\|_2} \left( \|\ub\|_2^2 - (\ub^\top \mathbf{y}^{(t)})^2 \right) \\
        &= \gamma_t^3 \cdot \frac{-2}{r^{(t)}\|\ub\|_2} \left( \frac{\|\ub\|_2^2\|\hb^{(t)}\|_2^2 - (\ub^\top \hb^{(t)})^2}{\|\hb^{(t)}\|_2^2} \right) \\
        &= \gamma_t^3 \cdot \frac{-2}{r^{(t)}\|\ub\|_2} \left( \frac{\|\ub\|_2^2\|\ab + \gamma_t\ub\|_2^2 - (\ub^\top (\ab + \gamma_t\ub))^2}{(r^{(t)})^2} \right) \\
        &= \frac{-2\gamma_t^3}{(r^{(t)})^3 \|\ub\|_2} \left( \|\ub\|_2^2(\|\ab\|_2^2 + 2\gamma_t(\ab^\top\ub) + \gamma_t^2\|\ub\|_2^2) - ((\ab^\top\ub) + \gamma_t\|\ub\|_2^2)^2 \right) \\
        &= \frac{-2\gamma_t^3}{(r^{(t)})^3 \|\ub\|_2} \left( \|\ub\|_2^2\|\ab\|_2^2 - (\ab^\top\ub)^2 \right).
    \end{align*}
    As $t \to \infty$, we have $\lim_{t \to \infty} \frac{\gamma_t}{r^{(t)}} = \lim_{t \to \infty} \frac{\gamma_t}{\|\ab + \gamma_t\ub\|_2} = \frac{1}{\|\ub\|_2}$. Thus, $\lim_{t \to \infty} \frac{\gamma_t^3}{(r^{(t)})^3} = \frac{1}{\|\ub\|_2^3}$.
    The pointwise limit is:
    \begin{equation*}
    \lim_{t \to \infty} Z_t = \frac{-2}{\|\ub\|_2^4} (\|\ub\|_2^2\|\ab\|_2^2 - (\ab^\top\ub)^2).
    \end{equation*}
    \item \textbf{Domination:} The term ${\gamma_t}/{r^{(t)}}$ is uniformly bounded for $t \ge 1$: \begin{equation*}\frac{\gamma_t}{r^{(t)}} = \sqrt{\frac{\gamma_t^2}{\|\ab + \gamma_t\ub\|_2^2}}= \sqrt{\frac{\gamma_t^2}{\|\ab\|_2^2 + \gamma_t^2\|\ub\|_2^2+2\gamma_t\ab^\top \Sb\ab}}\leq \frac{1}{\|\ub\|_2}.\end{equation*} Thus, $|Z_t|$ is bounded by $\frac{2}{\|\ub\|_2^4} |\|\ub\|_2^2\|\ab\|_2^2 - (\ab^\top\ub)^2|$, which is a function independent of $t$. By Assumption \ref{assump:distribution} and Lemma~\ref{lem:lambda_min_to_directional_inverse}, we have
    \begin{equation*}
        \EE\left[\frac{2}{\|\ub\|_2^4} |\|\ub\|_2^2\|\ab\|_2^2 - (\ab^\top\ub)^2|\right]\leq \EE\left[\frac{2}{\|\ub\|_2^2} \right]=\EE\left[\frac{2}{\|\Xb\Xb^\top \eb_1\|_2^2} \right]<\infty.
    \end{equation*}
    The equality holds because of the rotational invariance property. Thus, this dominating function is integrable.
\end{enumerate}
By the Dominated Convergence Theorem, we can swap the limit and expectation:
\begin{equation*}
\lim_{t \to \infty} \gamma_t^3 F_t
= \frac{1}{d}\EE \left[ \lim_{t \to \infty} Z_t \right]
= \frac{1}{d}\EE \left[
\frac{-2\bigl(\|\Sb\ab\|_2^2\|\ab\|_2^2 - (\ab^\top\Sb\ab)^2\bigr)}{\|\Sb\ab\|_2^4}
\right].
\end{equation*}
Since $\|\ab\|_2=1$, the numerator simplifies via the projection identity
\begin{equation*}
\|\Sb\ab\|_2^2\|\ab\|_2^2 - (\ab^\top\Sb\ab)^2
= \|\Sb\ab\|_2^2 - (\ab^\top\Sb\ab)^2
= \bigl\|(\Ib_d-\ab\ab^\top)\Sb\ab\bigr\|_2^2 \ge 0.
\end{equation*}
Moreover, this quantity equals $0$ if and only if $(\Ib_d-\ab\ab^\top)\Sb\ab=\mathbf{0}$, i.e., $\Sb\ab$ is collinear with $\ab$, which happens if and only if $\ab$ is an eigenvector of $\Sb$. By Assumption~\ref{assump:distribution}, the independence of $\ab$ and $\Sb$, and for the same reason stated in Lemma \ref{lemma divergence}, this event occurs with probability zero. Hence
\begin{equation*}
\bigl\|(\Ib_d-\ab\ab^\top)\Sb\ab\bigr\|_2^2 > 0
\quad\text{almost surely},
\end{equation*}
and since the denominator $\|\Sb\ab\|_2^4>0$ almost surely, the integrand inside the expectation is strictly negative almost surely. Therefore the expectation is strictly negative, and we may define
\begin{equation*}
\lim_{t \to \infty} \gamma_t^3 F_t = -C_1,
\qquad C_1 \coloneqq \frac{2}{d}\EE \left[
\frac{\|\Sb\ab\|_2^2\|\ab\|_2^2 - (\ab^\top\Sb\ab)^2}{\|\Sb\ab\|_2^4}
\right] > 0.
\end{equation*}

\paragraph{Limit of the increment expression.}
We now return to the full expression for $\gamma_{t+1}^3 - \gamma_t^3$ and analyze the limit of each term.
\begin{equation*}
\gamma_{t+1}^3 - \gamma_t^3 = \underbrace{-3\eta \gamma_t^2 F_t (\alpha_t^2 + \beta_t^2)}_{\text{Term A}} + \underbrace{3\eta^2 \gamma_t^3 F_t^2}_{\text{Term B}} + \underbrace{3\gamma_t \Delta_t^2 + \Delta_t^3}_{\text{Term C}}.
\end{equation*}
\subparagraph{Analysis of Term A.}
We begin with Term A. To leverage our known limits, we perform an algebraic rearrangement to isolate the terms $\gamma_t^3 F_t$ and $(\alpha_t^2 + \beta_t^2)/\gamma_t$:
\begin{align*}
\lim_{t \to \infty} \text{Term A} &= \lim_{t \to \infty} \left[ -3\eta \gamma_t^2 F_t (\alpha_t^2 + \beta_t^2) \right] \\
&= \lim_{t \to \infty} \left[ -3\eta \left( \gamma_t^3 F_t \right) \left( \frac{\alpha_t^2 + \beta_t^2}{\gamma_t} \right) \right].
\end{align*}
Since the limits of the individual factors exist, we can apply the product rule for limits:
\begin{align*}
&= -3\eta \cdot \left( \lim_{t \to \infty} \gamma_t^3 F_t \right) \cdot \left( \lim_{t \to \infty} \frac{\alpha_t^2 + \beta_t^2}{\gamma_t} \right).
\end{align*}
We now substitute the known values. From result (3), the first limit is $-C_1$. From result (2), we have $\lim (\alpha_t^2 + \beta_t^2)/(2\gamma_t) = 1$, which implies that the second limit is $2$. Substituting these values gives:
\begin{equation*}
\lim_{t \to \infty} \text{Term A} = -3\eta \cdot (-C_1) \cdot (2) = 6\eta C_1.
\end{equation*}
This is a finite, positive constant.

\subparagraph{Analysis of Term B.}
Next, we analyze Term B. We rearrange the expression to make use of the known limit for $\gamma_t^3 F_t$:
\begin{align*}
\lim_{t \to \infty} \text{Term B} &= \lim_{t \to \infty} \left[ 3\eta^2 \gamma_t^3 F_t^2 \right] \\
&= \lim_{t \to \infty} \left[ 3\eta^2 \frac{(\gamma_t^3 F_t)^2}{\gamma_t^3} \right].
\end{align*}
We examine the limit of the numerator and the denominator separately:
\begin{itemize}
    \item \textbf{Numerator:} $\lim_{t \to \infty} 3\eta^2 (\gamma_t^3 F_t)^2 = 3\eta^2 \left( \lim_{t \to \infty} \gamma_t^3 F_t \right)^2 = 3\eta^2 (-C_1)^2 = 3\eta^2 C_1^2$. The numerator converges to a finite constant.
    \item \textbf{Denominator:} $\lim_{t \to \infty} \gamma_t^3 = \infty$, as established by result (1).
\end{itemize}
Since the numerator approaches a finite value and the denominator diverges to infinity, the limit of the fraction is zero:
\begin{equation*}
\lim_{t \to \infty} \text{Term B} = 0.
\end{equation*}

\subparagraph{Analysis of Term C.}
Finally, we analyze Term C, which is $3\gamma_t \Delta_t^2 + \Delta_t^3$. To find its limit, we first need to determine the asymptotic order of magnitude of $\Delta_t = \gamma_{t+1} - \gamma_t$.

First, let us find the limit of $\gamma_t^2 \Delta_t$. Recall that $\Delta_t = - \eta F_t (\alpha_t^2 + \beta_t^2) + \eta^2 \gamma_t F_t^2$.
\begin{align*}
\lim_{t \to \infty} \gamma_t^2 \Delta_t &= \lim_{t \to \infty} \gamma_t^2 \left( - \eta F_t (\alpha_t^2 + \beta_t^2) + \eta^2 \gamma_t F_t^2 \right) \\
&= \lim_{t \to \infty} \left( - \eta \gamma_t^2 F_t (\alpha_t^2 + \beta_t^2) + \eta^2 \gamma_t^3 F_t^2 \right) \\
&= \frac{1}{3} \lim_{t \to \infty} \left( \text{Term A} \right) + \frac{1}{3} \lim_{t \to \infty} \left( \text{Term B} \right) \\
&= \frac{1}{3} (6\eta C_1) + \frac{1}{3} (0) = 2\eta C_1.
\end{align*}
Since $\lim_{t \to \infty} \gamma_t^2 \Delta_t$ is a finite non-zero constant, it implies that $\Delta_t$ is asymptotically proportional to $\gamma_t^{-2}$. Using big-O notation, we have:
\begin{equation*}
    \Delta_t = O(\gamma_t^{-2}).
\end{equation*}
With this information, we can analyze the components of Term C:
\begin{itemize}
    \item For the first part, $3\gamma_t \Delta_t^2$: Since $\Delta_t = O(\gamma_t^{-2})$, it follows that $\Delta_t^2 = \left(O(\gamma_t^{-2})\right)^2 = O(\gamma_t^{-4})$. Therefore, $\gamma_t \Delta_t^2 = \gamma_t \cdot O(\gamma_t^{-4}) = O(\gamma_t^{-3})$. As $t \to \infty$, $\gamma_t \to \infty$, so this term converges to 0.
    \item For the second part, $\Delta_t^3$: $\Delta_t^3 = \left(O(\gamma_t^{-2})\right)^3 = O(\gamma_t^{-6})$. This term also converges to 0.
\end{itemize}
Thus, the limit of Term C is:
\begin{equation*}
\lim_{t \to \infty} \text{Term C} = \lim_{t \to \infty} (3\gamma_t \Delta_t^2 + \Delta_t^3) = 0 + 0 = 0.
\end{equation*}
Summing the limits, we rigorously find:
\begin{equation*}
\lim_{t \to \infty} (\gamma_{t+1}^3 - \gamma_t^3) = 6\eta C_1.
\end{equation*}

\paragraph{Step 3: Final conclusion}

We now derive an upper and lower bound for $\gamma_t$.

\begin{lemma}\label{appendix thm:gamma_growth rate}
There exist positive constants $C_2$, $C_3$, and an integer $T$ such that for all $t \ge T$:
\begin{equation*}
C_2  t^{1/3} \le \gamma_t \le C_3  t^{1/3}.
\end{equation*}
\end{lemma}

\begin{proof}
The proof begins with the result derived previously:
\begin{equation*}
\lim_{t \to \infty} \big(\gamma_{t+1}^3 - \gamma_t^3\big) = 6\eta C_1, \quad \text{with } C_1 > 0.
\end{equation*}
Equivalently, there exists a sequence $\{\epsilon_t\}$ such that $\lim_{t \to \infty} \epsilon_t = 0$ and
\begin{equation*}
\gamma_{t+1}^3 - \gamma_t^3 = 6\eta C_1 + \epsilon_t.
\end{equation*}
Since $\epsilon_t \to 0$, there exists an integer $T_0$ such that for all $t \ge T_0$,
\begin{equation*}
|\epsilon_t| \le \frac{6\eta C_1}{2} = 3\eta C_1.
\end{equation*}
Hence, for all $t \ge T_0$,
\begin{equation*}
3\eta C_1 \le \gamma_{t+1}^3 - \gamma_t^3 \le 9\eta C_1.
\end{equation*}
Summing these inequalities from $T_0$ to $t-1$ (for any $t > T_0$) yields
\begin{equation*}
(t - T_0) 3\eta C_1 \le \gamma_t^3 - \gamma_{T_0}^3 \le (t - T_0) 9\eta C_1.
\end{equation*}
Therefore,
\begin{equation*}
\gamma_t^3 \ge \gamma_{T_0}^3 + (t - T_0) 3\eta C_1, 
\qquad
\gamma_t^3 \le \gamma_{T_0}^3 + (t - T_0) 9\eta C_1.
\end{equation*}
Choose $T \ge T_0$  such that \begin{equation*}
    T\geq \max\{3T_0-\frac{\gamma_{T_0}^3}{\eta C_1},\frac{\gamma_{T_0}^3}{\eta C_1}-9T_0\}. 
\end{equation*}Then for all $t \ge T$,
\begin{equation*}
2\eta C_1 t \le \gamma_t^3 \le 10\eta C_1 t.
\end{equation*}
Taking cube roots gives constants
\begin{equation*}
C_2 = (2\eta C_1)^{1/3}, \qquad
 C_3 = (10\eta C_1)^{1/3},
\end{equation*}
such that for all $t \ge T$,
\begin{equation*}
C_2  t^{1/3} \le \gamma_t \le C_3  t^{1/3}.
\end{equation*}
This establishes the desired bound.
\end{proof}
Moreover, noting that
\begin{equation*}
    C_1 = \frac{2}{d}\,\EE \left[
    \frac{\|\Sb\ab\|_2^2 \|\ab\|_2^2 - (\ab^\top\Sb\ab)^2}{\|\Sb\ab\|_2^4}
    \right]
    = \frac{2\Upsilon_3}{d},
\end{equation*}
we obtain
\begin{equation*}
    C_2 = (2\eta C_1)^{1/3}
        = \left(\frac{4\eta\Upsilon_3}{d}\right)^{1/3}
    \quad \text{and} \quad
    C_3 = (10\eta C_1)^{1/3}
        = \left(\frac{20\eta\Upsilon_3}{d}\right)^{1/3},
\end{equation*}
which implies
\begin{equation*}
\gamma_t = \Theta\!\left(\bigl(\eta\Upsilon_3 t / d\bigr)^{1/3}\right).
\end{equation*}
Combined with \eqref{alpha/beta limit}, we obtain
\begin{equation*}
\alpha_t = \Theta\!\left(\bigl(\eta\Upsilon_3 t / d\bigr)^{1/6}\right)
\quad\text{and}\quad
\beta_t  = \Theta\!\left(\bigl(\eta\Upsilon_3 t / d\bigr)^{1/6}\right).
\end{equation*}

\subsection{Convergence of the loss function}
\begin{lemma}
    Under Assumption \ref{assump:distribution}, suppose that $n\ge 1$, $d\ge 3$, and $\eta>0$. For any sufficiently small $\epsilon>0$, there exists a finite number of iterations $T^\star=O(  {d^{{3}/{2}}\Upsilon(n,d)^{{1}/{2}}}/{\eta n^{{1}/{2}}\epsilon^{{3}/{2}}})$,
such that for all training steps $t\ge T^\star$, it holds that:
\begin{equation*}
    \mathcal{L}(\btheta^{(t)}) \le \epsilon.
\end{equation*}
\end{lemma}
\begin{proof}
The proof consists of three main steps:
\begin{enumerate}
    \item We first express the loss function $\mathcal{L}(\btheta^{(t)})$ as a function of the parameter $\gamma_t = \alpha_t\beta_t$.
    \item We derive an upper bound for the loss, showing that it is inversely proportional to $\gamma_t^2$.
    \item We combine this bound with the growth rate of $\gamma_t$ established in the appendix to determine the required number of steps $T^\star$ in terms of $\epsilon$.
\end{enumerate}

\paragraph{Step 1: Exact formulation of the Loss}
The loss at step $t$ is defined as \begin{equation*}
    \mathcal{L}(\btheta^{(t)}) = \EE\left[ \left\| \yb^{(t)} - \mathbf{y}_{\text{target}} \right\|_2^2 \right].
\end{equation*} Since both the prediction $\yb^{(t)}$ and the target $\mathbf{y}_{\text{target}}$ are unit vectors, we can rewrite the loss as:
\begin{equation*}
    \left\| \yb^{(t)} - \mathbf{y}_{\text{target}} \right\|_2^2 = \|\yb^{(t)}\|_2^2 - 2\yb^{(t)\top}\mathbf{y}_{\text{target}} + \|\mathbf{y}_{\text{target}}\|_2^2 = 2\left(1 - \yb^{(t)\top}\mathbf{y}_{\text{target}}\right).
\end{equation*}
From the analysis before, we know the prediction is given by $\yb^{(t)} = \frac{\ab + \gamma_t \Sb\ab}{\|\ab + \gamma_t \Sb\ab\|_2}$ and the target is $\mathbf{y}_{\text{target}} = \frac{\Sb\ab}{\|\Sb\ab\|_2}$. The loss is therefore related to the cosine of the angle $\phi_t$ between these two vectors, and can be expressed exactly as $2(1-\cos\phi_t)$.

\paragraph{Step 2: Bounding the Loss}
To derive a rigorous upper bound, we first analyze the angle $\phi_t$ between the prediction and the target. The cosine of this angle is their inner product:
\begin{equation*}
\cos\phi_t = (\yb^{(t)})^\top \mathbf{y}_{\text{target}} = \frac{(\ab + \gamma_t \Sb\ab)^\top(\Sb\ab)}{\|\ab + \gamma_t \Sb\ab\|_2 \|\Sb\ab\|_2} = \frac{\ab^\top\Sb\ab + \gamma_t \|\Sb\ab\|_2^2}{\|\ab + \gamma_t \Sb\ab\|_2 \|\Sb\ab\|_2}\geq 0.
\end{equation*}
Since $\cos\phi_t \geq 0$ for all training steps $t$, the angle $\phi_t$ is always in the interval $[0, \pi/2]$. For any angle in this range, the inequality $1-\cos\phi_t \le \sin^2(\phi_t)$ is valid. We can therefore use this inequality to bound the loss for all $t$:
\begin{equation*}
    \mathcal{L}(\btheta^{(t)}) = \EE[2(1 - \cos\phi_t)] \le \EE[2\sin^2(\phi_t)].
\end{equation*}
The squared sine of the angle is:
\begin{align*}
    \sin^2(\phi_t)&= \frac{\|(\ab + \gamma_t \Sb\ab) \times (\Sb\ab)\|_2^2}{\|\ab + \gamma_t \Sb\ab\|_2^2 \|\Sb\ab\|_2^2} \\
    &= \frac{\|\ab \times \Sb\ab\|_2^2}{\|\ab + \gamma_t \Sb\ab\|_2^2 \|\Sb\ab\|_2^2} \\
    &= \frac{\|\ab\|_2^2 \|\Sb\ab\|_2^2 - (\ab^\top \Sb\ab)^2}{\left( \|\ab\|_2^2 + 2\gamma_t \ab^\top\Sb\ab + \gamma_t^2 \|\Sb\ab\|_2^2 \right) \|\Sb\ab\|_2^2}.
\end{align*}
Since $\Sb = \Xb\Xb^\top$ is positive semi-definite, we have $\ab^\top\Sb\ab \geq 0$. Therefore, we can lower-bound the denominator by dropping this non-negative term:
\begin{equation*}
    \|\ab\|_2^2 + 2\gamma_t \ab^\top\Sb\ab + \gamma_t^2 \|\Sb\ab\|_2^2 \ge \gamma_t^2 \|\Sb\ab\|_2^2.
\end{equation*}
Substituting this lower bound gives us an upper bound for the loss:
\begin{equation*}
    \mathcal{L}(\btheta^{(t)}) \le 2 \cdot \EE\left[ \frac{\|\ab\|_2^2 \|\Sb\ab\|_2^2 - (\ab^\top \Sb\ab)^2}{\gamma_t^2 \|\Sb\ab\|_2^4} \right] = \frac{1}{\gamma_t^2} \cdot 2 \cdot \EE\left[ \frac{\|\ab\|_2^2}{\|\Sb\ab\|_2^2} - \frac{(\ab^\top \Sb\ab)^2}{\|\Sb\ab\|_2^4} \right].
\end{equation*}
The expectation term in the expression above depends only on the distribution of the data $\Xb$ and the vector $\ab$, but not on the step $t$. We can thus define it as a positive constant, $C_{4} $:
\begin{equation*}
    C_{4} = 2 \cdot \EE\left[ \frac{\|\ab\|_2^2 \|\Sb\ab\|_2^2 - (\ab^\top \Sb\ab)^2}{\|\Sb\ab\|_2^4} \right].
\end{equation*}
This yields the inequality:
\begin{equation*}
    \mathcal{L}(\btheta^{(t)}) \le \frac{C_{4} }{\gamma_t^2}.
\end{equation*}

\paragraph{Step 3: Determining the required iterations}
From Lemma~\ref{appendix thm:gamma_growth rate}, there exists a positive constant
\begin{equation*}
C_2=(2\eta C_1)^{1/3}
\end{equation*}
and an integer $T$ such that for all $t\ge T$,
\begin{equation*}
\gamma_t \ge C_2 t^{1/3}.
\end{equation*}
Substituting this lower bound into the loss inequality, we obtain that for $t\ge T$,
\begin{equation*}
\mathcal{L}(\btheta^{(t)})
\le \frac{C_4}{(C_2 t^{1/3})^2}
= \frac{C_4}{C_2^{2} t^{2/3}}.
\end{equation*}
Recalling $C_2=(2\eta C_1)^{1/3}$ and $C_4=dC_1$, we obtain
\begin{equation*}
\mathcal{L}(\btheta^{(t)})
\le \frac{dC_1^{1/3}}{(2\eta)^{2/3}t^{2/3}}.
\end{equation*}

\noindent\textbf{Bounding \(C_1\).}
By Assumption~\ref{assump:distribution}, we have
\begin{align*}
    C_1
    &=
    \frac{1}{d}
    \EE
    \left[
    \frac{
    2\bigl(\|\Sb\ab\|_2^2\|\ab\|_2^2-(\ab^\top\Sb\ab)^2\bigr)
    }{
    \|\Sb\ab\|_2^4
    }
    \right] \\
    &\le
    \frac{1}{d}
    \EE
    \left[
    \frac{2}{\|\Sb\ab\|_2^2}
    \right] \\
    &\le
    \frac{2}{d}
    \cdot
    C_d\,\frac d n\,\Upsilon(n,d)
    =
    \frac{2C_d}{n}\Upsilon(n,d),
\end{align*}
where the last inequality follows from
Lemma~\ref{lem:lambda_min_to_directional_inverse}. Hence, for \(t\ge T\),
\begin{equation*}
\mathcal{L}(\btheta^{(t)})
\le
\frac{dC_1^{1/3}}{(2\eta)^{2/3}t^{2/3}}
\le
\frac{
d\left(C_d\Upsilon(n,d)/n\right)^{1/3}
}{
2^{1/3}\eta^{2/3}t^{2/3}
}.
\end{equation*}
To ensure \(\mathcal{L}(\btheta^{(t)})\le \epsilon\), it suffices that
\begin{equation*}
\frac{
d\left(C_d\Upsilon(n,d)/n\right)^{1/3}
}{
2^{1/3}\eta^{2/3}t^{2/3}
}
\le \epsilon,
\quad\Longleftrightarrow\quad
t
\ge
\frac{
d^{3/2}\left(C_d\Upsilon(n,d)/n\right)^{1/2}
}{
\sqrt{2}\eta\epsilon^{3/2}
}.
\end{equation*}
Therefore, it suffices to take the iteration number \(T^\star\) as:
\begin{equation*}
    T^\star
    =
    \max
    \left\{
    T,\,
    \left\lceil
    \frac{
    d^{3/2}\left(C_d\Upsilon(n,d)/n\right)^{1/2}
    }{
    \sqrt{2}\eta\epsilon^{3/2}
    }
    \right\rceil
    \right\}.
\end{equation*}
In particular, for any sufficiently small \(\epsilon>0\):
\begin{equation*}
     T^\star
     =
     \max
     \left\{
     T,\,
     \left\lceil
     \frac{
     d^{3/2}\left(C_d\Upsilon(n,d)/n\right)^{1/2}
     }{
     \sqrt{2}\eta\epsilon^{3/2}
     }
     \right\rceil
     \right\}
     =
     O\left(
     \frac{
     d^{3/2}\left(\Upsilon(n,d)/n\right)^{1/2}
     }{
     \eta\epsilon^{3/2}
     }
     \right).
\end{equation*}
This completes the proof.
\end{proof}

\section{Proof of Theorem \ref{thm:looped error}}\label{proof of theorem 2}

\begin{proof}
The proof proceeds in four steps. The first three steps provide a detailed derivation for the normalized model. The final step provides a parallel derivation for the unnormalized model.

%First, we analyze the model's output in the eigenbasis of the covariance matrix. Second, we derive an exact mathematical expression for the angular error. Finally, we establish the upper and lower bounds for this error by analyzing its exact expression.

\paragraph{Step 1: Analysis of the looped transformer in the eigenbasis}
Let $\Sb_{\text{test}} = \Xb_{\text{test}}\Xb_{\text{test}}^\top$. As established in Theorem~\ref{thm:convergence}, after $t$ training steps, the single-layer block implements a linear transformation (before normalization) represented by the matrix $\Mb_{\text{test}} \coloneqq \Ib_d + \gamma_t \Sb_{\text{test}}$. Stacking this layer $L$ times with shared parameters is equivalent to applying the matrix power $\Mb_{\text{test}}^L$ to the initial vector $\ab$. The final output is therefore:
\begin{equation*}
\yb_L^{(t)} = \frac{\Mb_{\text{test}}^L\ab}{\|\Mb_{\text{test}}^L\ab\|_2}.
\end{equation*}
A key property of $\Mb_{\text{test}}$ is that it shares the same eigenvectors $\{\vb_i\}$ and eigenspaces as the covariance matrix $\Sb_{\text{test}}$. The eigenvalues of $\Mb_{\text{test}}$, denoted by $\mu_i$, are directly related to the eigenvalues $\lambda_i$ of $\Sb_{\text{test}}$:
\begin{equation*}
\mu_i \coloneqq 1 + \gamma_t\lambda_i, \quad \text{for } i=1, \ldots, d.
\end{equation*}
Given the ordering of $\lambda_i$ and that $\gamma_t > 0$, the eigenvalues of $\Mb_{\text{test}}$ also have a multiplicity of $r$ for the largest value:
\begin{equation*}
\mu_1 = \dots = \mu_r > \mu_{r+1} \ge \dots \ge \mu_d \ge 1.
\end{equation*}

We can now express the action of $\Mb_{\text{test}}^L$ on $\ab$ by decomposing $\ab$ into the eigenbasis: $\ab = \sum_{i=1}^d c_i \vb_i$. Applying the transformation $\Mb_{\text{test}}^L$ yields:
\begin{equation*}
\Mb_{\text{test}}^L\ab = \Mb_{\text{test}}^L \left(\sum_{i=1}^d c_i \vb_i\right) = \sum_{i=1}^d c_i (\Mb_{\text{test}}^L \vb_i) = \sum_{i=1}^d c_i \mu_i^L \vb_i.
\end{equation*}
Using the orthonormality of the eigenvectors ($\vb_i^\top \vb_j = \delta_{ij}$), we compute the squared norm of this vector:
\begin{equation*}
\|\Mb_{\text{test}}^L\ab\|_2^2 = \left(\sum_{i=1}^d c_i \mu_i^L \vb_i\right)^\top \left(\sum_{j=1}^d c_j \mu_j^L \vb_j\right) = \sum_{i=1}^d \sum_{j=1}^d c_i c_j \mu_i^L \mu_j^L (\vb_i^\top \vb_j) = \sum_{i=1}^d c_i^2 \mu_i^{2L}.
\end{equation*}

\paragraph{Step 2: Exact expression for the angular error}
The angle $\phi_L^{(t)}$ is defined between the vector $\yb_L^{(t)}$ and the principal subspace $\cZ$. Its sine is equal to the distance from $\yb_L$ to $\cZ$, which is the norm of the component of $\yb_L^{(t)}$ orthogonal to $\cZ$.
\begin{equation*}
\sin(\phi_L^{(t)}) = \dist(\yb_L^{(t)}, \cZ) = \|\yb_L^{(t)} - \Proj_{\cZ}(\yb_L^{(t)})\|_2.
\end{equation*}
Since $\{\vb_i\}_{i=1}^d$ is an orthonormal basis, the orthogonal complement of $\cZ$ is $\cZ^\perp = \spanv\{\vb_{r+1}, \dots, \vb_d\}$. The squared error is the squared norm of the projection of $\yb_L^{(t)}$ onto $\cZ^\perp$:
\begin{equation*}
\sin^2\phi_L^{(t)} = \left\| \Proj_{\cZ^\perp}(\yb_L^{(t)}) \right\|_2^2 = \left\| \sum_{i=r+1}^d (\yb_L^{(t)\top} \vb_i)\vb_i \right\|_2^2 = \sum_{i=r+1}^d (\yb_L^{(t)\top} \vb_i)^2.
\end{equation*}
The projection of the output vector $\yb_L^{(t)}$ onto each basis vector $\vb_i$ is:
\begin{equation*}
\yb_L^{(t)\top}\vb_i = \frac{(\Mb_{\text{test}}^L\ab)^\top \vb_i}{\|\Mb_{\text{test}}^L\ab\|_2} = \frac{\left(\sum_{j=1}^d c_j \mu_j^L \vb_j\right)^\top \vb_i}{\|\Mb_{\text{test}}^L\ab\|_2} = \frac{c_i \mu_i^L}{\|\Mb_{\text{test}}^L\ab\|_2}.
\end{equation*}
Substituting this into the expression for $\sin^2\phi_L^{(t)}$ gives the exact formula:
\begin{equation}
\label{eq:sin2_exact_multiplicity}
\sin^2\phi_L^{(t)} = \frac{\sum_{i=r+1}^d (c_i \mu_i^L)^2}{\|\Mb_{\text{test}}^L\ab\|_2^2} = \frac{\sum_{i=r+1}^d c_i^2 \mu_i^{2L}}{\sum_{j=1}^d c_j^2 \mu_j^{2L}}.
\end{equation}

\paragraph{Step 3: Derivation of the error bounds}
We now derive the upper and lower bounds from the exact expression \eqref{eq:sin2_exact_multiplicity}.

\subparagraph{Upper bound:}
To find an upper bound for $\sin^2\phi_L^{(t)}$, we find an upper bound for the numerator and a lower bound for the denominator.
\begin{itemize}
    \item \textbf{Numerator:} Since $\mu_i \le \mu_{r+1}$ for all $i \ge r+1$:
    \begin{equation*}
    \sum_{i=r+1}^d c_i^2 \mu_i^{2L} \le \sum_{i=r+1}^d c_i^2 \mu_{r+1}^{2L} = \left(\sum_{i=r+1}^d c_i^2\right) \mu_{r+1}^{2L} = \|\ab_{\perp}\|_2^2 \mu_{r+1}^{2L}.
    \end{equation*}
    \item \textbf{Denominator:} We can obtain a lower bound by dropping the non-negative second term:
    \begin{equation*}
    \sum_{j=1}^d c_j^2 \mu_j^{2L} \ge \sum_{j=1}^r c_j^2 \mu_j^{2L} = \left(\sum_{j=1}^r c_j^2\right) \mu_1^{2L} = \|\ab_{\cZ}\|_2^2 \mu_1^{2L}.
    \end{equation*}
\end{itemize}
Combining these bounds yields the upper bound for $\sin^2\phi_L^{(t)}$:
\begin{equation*}
\sin^2\phi_L^{(t)} \le \frac{\|\ab_{\perp}\|_2^2 \mu_{r+1}^{2L}}{\|\ab_{\cZ}\|_2^2 \mu_1^{2L}} = \frac{\|\ab_{\perp}\|_2^2}{\|\ab_{\cZ}\|_2^2} \left(\frac{\mu_{r+1}}{\mu_1}\right)^{2L}.
\end{equation*}
Taking the square root of both sides gives the final upper bound.

\subparagraph{Lower bound:}
To find a lower bound for $\sin^2\phi_L^{(t)}$, we find a lower bound for the numerator and an upper bound for the denominator.
\begin{itemize}
    \item \textbf{Numerator:} The sum of non-negative terms is greater than or equal to any single term. We choose the term corresponding to $i=r+1$:
    \begin{equation*}
    \sum_{i=r+1}^d c_i^2 \mu_i^{2L} \ge c_{r+1}^2 \mu_{r+1}^{2L}.
    \end{equation*}
    \item \textbf{Denominator:} Since $\mu_j \le \mu_1$ for all $j \ge 1$:
    \begin{equation*}
    \sum_{j=1}^d c_j^2 \mu_j^{2L} \le \sum_{j=1}^d c_j^2 \mu_1^{2L} = \left(\sum_{j=1}^d c_j^2\right) \mu_1^{2L} = \|\ab\|_2^2 \mu_1^{2L}=\mu_1^{2L}.
    \end{equation*}
\end{itemize}
Combining these bounds gives the lower bound for $\sin^2\phi_L^{(t)}$:
\begin{equation*}
\sin^2\phi_L^{(t)} \ge \frac{c_{r+1}^2 \mu_{r+1}^{2L}}{\mu_1^{2L}} = {c_{r+1}^2} \left(\frac{\mu_{r+1}}{\mu_1}\right)^{2L}.
\end{equation*}
Taking the square root of both sides gives the final lower bound.
By substituting $\mu_i = 1+\gamma_t\lambda_i$ into the rate terms, we have
\begin{equation*}
    {|c_{r+1}|}\left(\frac{1+\gamma_t\lambda_{r+1}}{1+\gamma_t\lambda_1}\right)^{L}
    \le \sin\phi_L^{(t)}
    \le
    \frac{\|\ab_{\perp}\|_2}{\|\ab_{\cZ}\|_2}\left(\frac{1+\gamma_t\lambda_{r+1}}{1+\gamma_t\lambda_1}\right)^{L}.
\end{equation*}
Taking $t\to \infty$ completes the proof for the normalized model.

\paragraph{Step 4: Derivation for the unnormalized model}
As established in Theorem \ref{thm:unnorm convergence}, after $t$ training steps, the trained model is characterized by the parameter $\tilde{\gamma}_t$. The corresponding linear update is $\tilde{\Mb}_{\text{test}} \coloneqq \Ib_d + \tilde{\gamma}_t \Sb_{\text{test}}$, with eigenvalues $\tilde{\mu_i} = 1 + \tilde{\gamma}_t \lambda_i$.
The output after $L$ loops is therefore:
\begin{equation*}
\tilde{\yb}_L^{(t)} = \tilde{\Mb}_{\text{test}}^L\ab = \sum_{i=1}^d c_i (\tilde{\mu_i})^L \vb_i.
\end{equation*}
The angle $\tilde{\phi}_L^{(t)}$ is the angle between this vector $\tilde{\yb}_L^{(t)}$ and the principal subspace $\cZ$. The sine of the angle between a vector and a subspace is defined as the ratio of the norm of the vector's projection onto the subspace's orthogonal complement to the norm of the vector itself:
\begin{equation*}
\sin\tilde{\phi}_L^{(t)} = \frac{\|\Proj_{\cZ^\perp}(\tilde{\yb}_L^{(t)})\|_2}{\|\tilde{\yb}_L^{(t)}\|_2}.
\end{equation*}
We compute the norms of the numerator and the denominator:
\begin{itemize}
    \item \textbf{Numerator:} The projection of $\tilde{\yb}_L^{(t)}$ onto $\cZ^\perp$ is $\sum_{i=r+1}^d c_i (\tilde{\mu_i})^L \vb_i$. Its norm is:
    \begin{equation*}
    \|\Proj_{\cZ^\perp}(\tilde{\yb}_L^{(t)})\|_2 = \sqrt{\sum_{i=r+1}^d c_i^2 (\tilde{\mu_i})^{2L}}.
    \end{equation*}
    \item \textbf{Denominator:} The squared norm of the vector $\tilde{\yb}_L^{(t)}$ itself is:
    \begin{equation*}
    \|\tilde{\yb}_L^{(t)}\|_2 = \|\tilde{\Mb}_{\text{test}}^L\ab\|_2 = \sqrt{\sum_{j=1}^d c_j^2 (\tilde{\mu}_j)^{2L}}.
    \end{equation*}
\end{itemize}
Combining these, we find the expression for the squared sine of the angle:
\begin{equation*}
\sin^2\tilde{\phi}_L^{(t)} = \frac{\sum_{i=r+1}^d c_i^2 (\tilde{\mu}_i)^{2L}}{\sum_{j=1}^d c_j^2 (\tilde{\mu}_j)^{2L}}.
\end{equation*}
Crucially, this expression for $\sin^2\tilde{\phi}_L^{(t)}$ is algebraically identical to the expression for $\sin^2\phi_L^{(t)}$ in \eqref{eq:sin2_exact_multiplicity}, with $\mu_{i}$ replaced by $\tilde{\mu}_i$. Therefore, the same bounding arguments from Step 3 can be applied directly; we have
\begin{equation*}
    {|c_{r+1}|}\left(\frac{1+\tilde{\gamma}_t\lambda_{r+1}}{1+\tilde{\gamma}_t\lambda_1}\right)^{L}
    \le\sin\tilde{\phi}_L^{(t)}
    \le
    \frac{\|\ab_{\perp}\|_2}{\|\ab_{\cZ}\|_2}\left(\frac{1+\tilde{\gamma}_t\lambda_{r+1}}{1+\tilde{\gamma}_t\lambda_1}\right)^{L}.
\end{equation*}
By Theorem \ref{thm:unnorm convergence}, $\tilde{\gamma}_t\to \gamma^\star$ as $t\to \infty$. This completes the proof.
\end{proof}

\section{Proof of Theorem \ref{thm:end-to-end looped error}}\label{appendix: proof of end-to-end looped error}

\begin{proof}
Let
\[
\Sb_{\mathrm{test}}
:=
\Xb_{\mathrm{test}}\Xb_{\mathrm{test}}^\top .
\]
Under the conditions of Theorem~\ref{thm:looped error}, write
\[
\Sb_{\mathrm{test}}
=
\sum_{i=1}^d \lambda_i \vb_i\vb_i^\top,
\qquad
\lambda_1=\cdots=\lambda_r>\lambda_{r+1}\ge\cdots\ge\lambda_d\ge0.
\]
Recall that
\[
\cZ=\spanv\{\vb_1,\ldots,\vb_r\}.
\]
Write
\[
\ab=\sum_{i=1}^d c_i\vb_i,
\qquad
\ab_{\cZ}=\sum_{i=1}^r c_i\vb_i,
\qquad
\ab_\perp=\sum_{i=r+1}^d c_i\vb_i.
\]
By the assumptions of Theorem~\ref{thm:looped error}, we have
\[
\ab_{\cZ}\neq 0.
\]

By Theorem~\ref{thm:looped_end_to_end}, after \(t\) training steps, the looped
\(L\)-layer model has parameter matrices of the form
\[
\Wb^{(t)}
=
\begin{pmatrix}
\mathbf 0&w_t\Ib_d\\
\mathbf 0&\mathbf 0
\end{pmatrix},
\qquad
\Vb^{(t)}
=
\begin{pmatrix}
\mathbf 0&\mathbf 0\\
v_t\Ib_d&\mathbf 0
\end{pmatrix},
\]
with
\[
\rho_t:=w_tv_t\to\infty .
\]
Since \(\Xb_{\mathrm{test}}\) satisfies the unit-column condition in
Assumption~\ref{test distribution}, the same forward computation as in
Theorem~\ref{thm:looped_end_to_end} gives
\[
\yb_L^{(t)}
=
\frac{
(\Ib_d+\rho_t\Sb_{\mathrm{test}})^L\ab
}{
\|(\Ib_d+\rho_t\Sb_{\mathrm{test}})^L\ab\|_2
}.
\]
Define
\[
\mu_i^{(t)}
:=
1+\rho_t\lambda_i,
\qquad
i=1,\ldots,d.
\]
Then
\[
\mu_1^{(t)}
=
\cdots
=
\mu_r^{(t)}
>
\mu_{r+1}^{(t)}
\ge
\cdots
\ge
\mu_d^{(t)}
\ge1,
\]
and
\[
(\Ib_d+\rho_t\Sb_{\mathrm{test}})^L\ab
=
\sum_{i=1}^d c_i(\mu_i^{(t)})^L\vb_i.
\]
Therefore,
\[
\|(\Ib_d+\rho_t\Sb_{\mathrm{test}})^L\ab\|_2^2
=
\sum_{j=1}^d c_j^2(\mu_j^{(t)})^{2L}.
\]
Since \(\psi_L^{(t)}\) is the canonical angle between
\(\yb_L^{(t)}\) and \(\cZ\), we have
\[
\sin^2\psi_L^{(t)}
=
\left\|
\Proj_{\cZ^\perp}\yb_L^{(t)}
\right\|_2^2.
\]
Using the above eigenbasis expansion, this gives the exact identity
\[
\sin^2\psi_L^{(t)}
=
\frac{
\sum_{i=r+1}^d c_i^2(\mu_i^{(t)})^{2L}
}{
\sum_{j=1}^d c_j^2(\mu_j^{(t)})^{2L}
}.
\]

We first prove the upper bound. Since
\[
\mu_i^{(t)}\le \mu_{r+1}^{(t)}
\qquad
\text{for all }i\ge r+1,
\]
the numerator satisfies
\[
\sum_{i=r+1}^d c_i^2(\mu_i^{(t)})^{2L}
\le
(\mu_{r+1}^{(t)})^{2L}
\sum_{i=r+1}^d c_i^2
=
(\mu_{r+1}^{(t)})^{2L}\|\ab_\perp\|_2^2.
\]
Also, since
\[
\mu_1^{(t)}=\cdots=\mu_r^{(t)},
\]
the denominator satisfies
\[
\sum_{j=1}^d c_j^2(\mu_j^{(t)})^{2L}
\ge
\sum_{j=1}^r c_j^2(\mu_j^{(t)})^{2L}
=
(\mu_1^{(t)})^{2L}
\|\ab_{\cZ}\|_2^2.
\]
Hence
\[
\sin^2\psi_L^{(t)}
\le
\frac{\|\ab_\perp\|_2^2}{\|\ab_{\cZ}\|_2^2}
\left(
\frac{\mu_{r+1}^{(t)}}{\mu_1^{(t)}}
\right)^{2L}.
\]
Taking square roots yields
\[
\sin\psi_L^{(t)}
\le
\frac{\|\ab_\perp\|_2}{\|\ab_{\cZ}\|_2}
\left(
\frac{\mu_{r+1}^{(t)}}{\mu_1^{(t)}}
\right)^L.
\]

Next, we prove the lower bound. Since the numerator is a sum of nonnegative
terms,
\[
\sum_{i=r+1}^d c_i^2(\mu_i^{(t)})^{2L}
\ge
c_{r+1}^2(\mu_{r+1}^{(t)})^{2L}.
\]
On the other hand, since \(\mu_j^{(t)}\le\mu_1^{(t)}\) for every \(j\),
\[
\sum_{j=1}^d c_j^2(\mu_j^{(t)})^{2L}
\le
(\mu_1^{(t)})^{2L}
\sum_{j=1}^d c_j^2
=
(\mu_1^{(t)})^{2L},
\]
where we used \(\|\ab\|_2=1\). Therefore,
\[
\sin^2\psi_L^{(t)}
\ge
c_{r+1}^2
\left(
\frac{\mu_{r+1}^{(t)}}{\mu_1^{(t)}}
\right)^{2L}.
\]
Taking square roots gives
\[
\sin\psi_L^{(t)}
\ge
|c_{r+1}|
\left(
\frac{\mu_{r+1}^{(t)}}{\mu_1^{(t)}}
\right)^L.
\]

Combining the two bounds and substituting
\[
\mu_i^{(t)}=1+\rho_t\lambda_i
\]
gives
\[
|c_{r+1}|
\left(
\frac{1+\rho_t\lambda_{r+1}}{1+\rho_t\lambda_1}
\right)^L
\le
\sin\psi_L^{(t)}
\le
\frac{\|\ab_\perp\|_2}{\|\ab_{\cZ}\|_2}
\left(
\frac{1+\rho_t\lambda_{r+1}}{1+\rho_t\lambda_1}
\right)^L.
\]
Since \(\rho_t\to\infty\), we have
\[
\frac{1+\rho_t\lambda_{r+1}}{1+\rho_t\lambda_1}
\to
\frac{\lambda_{r+1}}{\lambda_1}.
\]

It remains only to note that the limit
\[
\lim_{t\to\infty}\sin\psi_L^{(t)}
\]
exists. Indeed, dividing the exact formula for
\(\sin^2\psi_L^{(t)}\) by \((\mu_1^{(t)})^{2L}\), we obtain
\[
\sin^2\psi_L^{(t)}
=
\frac{
\sum_{i=r+1}^d c_i^2
\left(
\frac{\mu_i^{(t)}}{\mu_1^{(t)}}
\right)^{2L}
}{
\sum_{j=1}^d c_j^2
\left(
\frac{\mu_j^{(t)}}{\mu_1^{(t)}}
\right)^{2L}
}.
\]
For each \(i\),
\[
\frac{\mu_i^{(t)}}{\mu_1^{(t)}}
=
\frac{1+\rho_t\lambda_i}{1+\rho_t\lambda_1}
\to
\frac{\lambda_i}{\lambda_1}.
\]
Moreover, the denominator limit is nonzero because
\[
\sum_{j=1}^r c_j^2=\|\ab_{\cZ}\|_2^2>0.
\]
Thus \(\sin^2\psi_L^{(t)}\), and hence \(\sin\psi_L^{(t)}\), has a limit.

Taking \(t\to\infty\) in the two-sided bound yields
\[
{|c_{r+1}|}
\left(
\frac{\lambda_{r+1}}{\lambda_1}
\right)^L
\le
\lim_{t\to\infty}\sin\psi_L^{(t)}
\le
\frac{\|\ab_\perp\|_2}{\|\ab_{\cZ}\|_2}
\left(
\frac{\lambda_{r+1}}{\lambda_1}
\right)^L.
\]
This completes the proof.
\end{proof}

\section{Auxiliary Lemmas}

\begin{lemma}[Schur's Lemma, simplified for $\RR^d$]\label{Schur's lemma}
Let $d\geq3$ and $\Mb \in \RR^{d\times d}$ be a matrix. If for every rotation matrix $\Rb \in \mathrm{SO}(d)$, the matrix $\Mb$ satisfies 
\begin{equation*}
    \Mb = \Rb\Mb\Rb^\top,
\end{equation*}
then $\Mb$ is a scalar multiple of the identity matrix:
\begin{equation*}\Mb = c\Ib_d.\end{equation*}
\end{lemma}
\begin{proof}
We prove this by applying the general representation-theoretic version of Schur's Lemma to the natural representation of $\mathrm{SO}(d)$ on $\mathbb{R}^d$.

\medskip

\noindent \textbf{Step 1: Establish the representation-theoretic framework.}

Let
\begin{itemize}
    \item $G = \mathrm{SO}(d)$
    \item $V = \mathbb{R}^d$
    \item $\pi: G \to \mathrm{GL}(V)$ be the defining representation: $\pi(\mathbf{R}) = \mathbf{R}$
    \item $T: V \to V$ be the linear map $T(\mathbf{x}) = \mathbf{M} \mathbf{x}$
\end{itemize}

The commutativity condition $\mathbf{M} \mathbf{R} = \mathbf{R} \mathbf{M}$ for all $\mathbf{R} \in \mathrm{SO}(d)$ is equivalent to:
\begin{equation*}
T \circ \pi(g) = \pi(g) \circ T \quad \forall g \in G,
\end{equation*}
identifying $T$ as an intertwining operator of the representation $(\pi, V)$.
\medskip

\noindent \textbf{Step 2: Irreducibility of the representation.}

A classical result in representation theory \citep{Fulton2013Representation} asserts that the defining representation of $\mathrm{SO}(d)$ on $\mathbb{R}^d$ is irreducible for all $d \geq 3$. 

\medskip

\noindent \textbf{Step 3: Apply Schur's Lemma.}

The general form of Schur's Lemma states that if $(\pi, V)$ is an irreducible representation of a group $G$ over a field $\mathbb{F}$, and $T: V \to V$ is an intertwining operator, then $T = \lambda \cdot \mathrm{Id}_V$ for some $\lambda \in \mathbb{F}$.

In the present context:
\begin{itemize}
    \item $G = \mathrm{SO}(d)$, $V = \mathbb{R}^d$, $\mathbb{F} = \mathbb{R}$
    \item $(\pi, V)$ is irreducible (by Step 2)
    \item $T$ is an intertwining operator (by Step 1)
\end{itemize}

Therefore, $T = c \cdot \mathrm{Id}_{\mathbb{R}^d}$ for some $c \in \mathbb{R}$, which means $\mathbf{M} = c \mathbf{I}_d$.

\end{proof}

\begin{lemma}
\label{lem: bounds for Sa}
Under Assumption \ref{assump:distribution}, define the sample covariance matrix $\Sb \coloneqq \Xb\Xb^\top$. Let $\ab$ be a random vector uniformly distributed on the unit sphere $\mathbb{S}^{d-1}$ and independent of $\Xb$. Then, the expected norm of $\Sb\ab$ satisfies
\begin{equation*}
    \frac{1}{d}\EE\big[\tr(\Sb)\big] \le \EE\big[\|\Sb\ab\|_2\big] \le \left(\frac{1}{d}\EE\big[\tr(\Sb^2)\big]\right)^{1/2}.
\end{equation*}
In particular, under the unit-column property of $\Xb$, these bounds simplify to
\begin{equation*}
    \frac{n}{d} \le \EE\big[\|\Sb\ab\|_2\big]  \le \frac{n}{\sqrt{d}}.
\end{equation*}
\end{lemma}

\begin{proof}
\medskip
\noindent{(i) Lower bound via Rayleigh quotient.}
Since $\ab$ is a unit vector and $\Sb \succeq 0$, we have
\begin{equation*}
\ab^\top \Sb\ab = \langle \Sb\ab,\ab\rangle \le \|\Sb\ab\|_2\cdot\|\ab\|_2 = \|\Sb\ab\|_2.
\end{equation*}
Taking the expectation and using $\EE[\ab\ab^\top]=\frac{1}{d}\Ib_d$,
\begin{equation*}
\EE\big[\|\Sb\ab\|_2\big]  \ge \EE\big[\ab^\top \Sb\ab\big]
= \EE\big[\tr(\Sb\EE[\ab\ab^\top])\big]
= \frac{1}{d}\EE\big[\tr(\Sb)\big].
\end{equation*}
With unit columns, $\tr(\Sb)=\tr(\Xb\Xb^\top)=\sum_{i=1}^n\|\xb_i\|_2^2=n$, hence
\begin{equation*}
\EE\big[\|\Sb\ab\|_2\big]  \ge \frac{n}{d}.
\end{equation*}

\medskip
\noindent{(ii) Upper bound via Cauchy--Schwarz.}
Conditioning on $\Xb$,
\begin{equation*}
\EE\big[\|\Sb\ab\|_2^2\mid \Xb\big]
=\EE\big[\ab^\top \Sb^2 \ab\mid \Xb\big]
=\tr\Big(\Sb^2\EE[\ab\ab^\top]\Big)
=\frac{1}{d}\tr(\Sb^2).
\end{equation*}
By Cauchy--Schwarz,
\begin{equation*}
\EE\big[\|\Sb\ab\|_2\big]  \le \Big(\EE\big[\|\Sb\ab\|_2^2\big] \Big)^{1/2}
=\Big(\frac{1}{d}\EE\big[\tr(\Sb^2)\big]\Big)^{1/2}.
\end{equation*}
Observing that $(\xb_i^\top \xb_j)^2\le 1$ and $\tr(\Sb^2)=\sum_{i,j}(\xb_i^\top \xb_j)^2\le n^2$, we obtain
\begin{equation*}
\EE\big[\|\Sb\ab\|_2\big]  \le \sqrt{\frac{n^2}{d}} = \frac{n}{\sqrt{d}}.
\end{equation*}
This completes the proof.
\end{proof}

\begin{remark}
Bounds on $\EE[\ab^\top \Sb \ab],\EE[\|\Sb\ab\|_2^2]$ under Assumption~\ref{assump:distribution}: We established above that $\EE[\ab^\top \Sb \ab]={\EE[\tr(\Sb)]}/{d}={n}/{d}$. 
%Since $\tr(\Sb)=n$ and
%$\tr(\Sb^2)=\sum_{i,j}(\xb_i^\top \xb_j)^2\in[n,n^2]$ for the columns $\{\xb_j\}_{j=1}^n$, and using $\sum_{i=1}^d \lambda_i^2\ge \frac{1}{d}\big(\sum_{i=1}^d \lambda_i\big)^2$ for eigenvalues $\{\lambda_i\}$ of $\Sb$, we obtain
%\begin{equation*}
 %\max\!\Big\{\frac{n}{d},\ \frac{n^2}{d^2}\Big\}\ \le\ \EE[\|\Sb\ab\|_2^2]\ =\ \frac{1}{d}\EE[\tr(\Sb^2)]\ \le\ \frac{n^2}{d}.
%\end{equation*}
For $\EE[\|\Sb\ab\|_2^2]$, since $\tr(\Sb)=n$ and $\tr(\Sb^2)=\sum_{i,j}(\xb_i^\top \xb_j)^2 \in [n, n^2]$, and utilizing the inequality $\sum_{i=1}^d \lambda_i^2 \ge \frac{1}{d}\big(\sum_{i=1}^d \lambda_i\big)^2$ for the eigenvalues $\{\lambda_i\}$ of $\Sb$, we obtain the following bounds:
\begin{equation*}
 \max\!\left\{\frac{n}{d},\ \frac{n^2}{d^2}\right\} \le \EE[\|\Sb\ab\|_2^2] = \frac{1}{d}\EE[\tr(\Sb^2)] \le \frac{n^2}{d}.
\end{equation*}
\end{remark}

\begin{lemma}
\label{lem:lambda_min_to_directional_inverse}
Under Assumption~\ref{assump:distribution}, with
\[
\Upsilon(n,d)
:=
\EE_{\Xb\sim\PP_{\Xb}}
\left[
\lambda_d(\Xb\Xb^\top)^{-1}
\right]
<\infty,
\]
where \(\lambda_d(\Xb\Xb^\top)\) is defined in Definition \ref{def:principal_component}. Let
\(\ab\sim\mathrm{Unif}(\SSS^{d-1})\) be independent of \(\Xb\). Then, for
\(d\ge3\), there exists a constant \(C_d<\infty\), depending only on \(d\),
such that
\[
\EE_{\Xb,\ab}
\left[
\frac{1}{\|\Xb\Xb^\top\ab\|_2^2}
\right]
\le
C_d\,\frac d n\,\Upsilon(n,d).
\]
\end{lemma}

\begin{proof}
Condition on \(\Xb\). Let
\[
\Xb\Xb^\top
=
\sum_{i=1}^d \lambda_i(\Xb)\vb_i(\Xb)\vb_i(\Xb)^\top,
\qquad
\lambda_1(\Xb)\ge\cdots\ge\lambda_d(\Xb)>0.
\]
The positivity of \(\lambda_d(\Xb)\) holds almost surely because
\[
\EE\!\left[\lambda_d(\Xb\Xb^\top)^{-1}\right]<\infty.
\]
Since \(\ab\sim\mathrm{Unif}(\SSS^{d-1})\) is independent of \(\Xb\), the
coordinates of \(\ab\) in the eigenbasis
\(\{\vb_i(\Xb)\}_{i=1}^d\) are distributed as a uniform vector
\(\ub=(u_1,\ldots,u_d)\in\SSS^{d-1}\). Hence
\[
\|\Xb\Xb^\top\ab\|_2^2
=
\sum_{i=1}^d \lambda_i(\Xb)^2 u_i^2.
\]
Using only the largest and smallest eigenvalues,
\[
\sum_{i=1}^d \lambda_i(\Xb)^2u_i^2
\ge
\lambda_1(\Xb)^2u_1^2
+
\lambda_d(\Xb)^2(1-u_1^2).
\]
For \(d\ge3\), the one-dimensional marginal \(u_1^2\) has a
\(\mathrm{Beta}(1/2,(d-1)/2)\) distribution. Therefore there exists a finite
constant \(C_d\), depending only on \(d\), such that for all
\(\lambda_1\ge\lambda_d>0\),
\[
\EE_{\ub}
\left[
\frac{1}{
\lambda_1^2u_1^2+\lambda_d^2(1-u_1^2)}
\right]
\le
\frac{C_d}{\lambda_1\lambda_d}.
\]
Consequently,
\[
\EE_{\ab}
\left[
\frac{1}{\|\Xb\Xb^\top\ab\|_2^2}
\,\middle|\,\Xb
\right]
\le
\frac{C_d}{
\lambda_1(\Xb\Xb^\top)\lambda_d(\Xb\Xb^\top)
}.
\]
By Assumption~\ref{assump:distribution}(A1), every column of \(\Xb\) has unit
norm. Hence
\[
\operatorname{tr}(\Xb\Xb^\top)=n,
\]
and therefore
\[
\lambda_1(\Xb\Xb^\top)\ge \frac nd.
\]
Thus
\[
\EE_{\ab}
\left[
\frac{1}{\|\Xb\Xb^\top\ab\|_2^2}
\,\middle|\,\Xb
\right]
\le
C_d\,\frac d n\,\lambda_d(\Xb\Xb^\top)^{-1}.
\]
Taking expectation over \(\Xb\), we obtain
\[
\EE_{\Xb,\ab}
\left[
\frac{1}{\|\Xb\Xb^\top\ab\|_2^2}
\right]
\le
C_d\,\frac d n\,
\EE_{\Xb}
\left[
\lambda_d(\Xb\Xb^\top)^{-1}
\right]
=
C_d\,\frac d n\,\Upsilon(n,d).
\]
This proves the claim.
\end{proof}

\section{Additional Experimental Results}\label{additional experiment results}
\subsection{Training Dynamics and Looped Performance of One-layer Transformers with LN}

In this subsection, we empirically validate the main result of Theorem~\ref{thm:convergence}, as well as the performance of the trained model in predicting the leading principal component when looped at inference time.

\noindent\textbf{Loss convergence.}
We train the model with various configurations of $d\in\{8,16\}$ and $n\in\{16,32,64\}$. Figure~\ref{fig:loss_curves} plots the training and test losses against the number of iterations. The losses are observed to decay consistently across all configurations. Notably, there is no discernible gap between the training loss and the test loss, which is uncommon in the literature. We emphasize that this occurs because our model learns precisely the update of the power method, as visualized in Figure~\ref{fig:heatmaps}. Consequently, the training and test loss curves are nearly identical.

%Moreover, the ratio of samples to dimension, $n/d$, appears to be a key indicator of task difficulty. As observed in the plot, the configuration with the highest ratio ($d=8, n=64$, giving $n/d=8$) achieves the fastest convergence and lowest final MSE, while the configuration with the lowest ratio ($d=16, n=16$, giving $n/d=1$) performs the worst. 

\begin{figure*}[ht!]
\centering
\subfigure[Training loss, Task 1]{\includegraphics[width=0.4\textwidth]{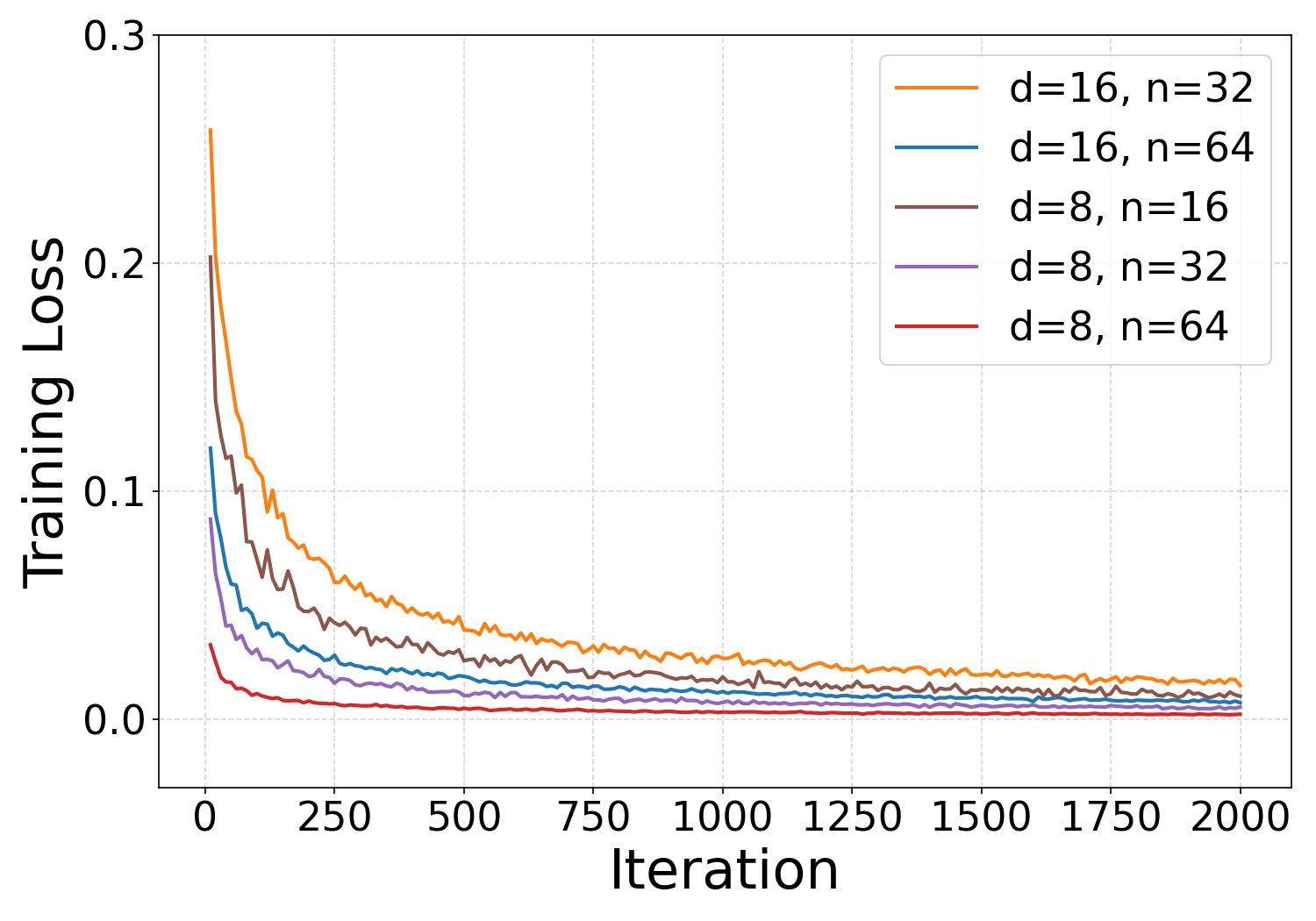}}
\subfigure[Test loss, Task 1]{\includegraphics[width=0.4\textwidth]{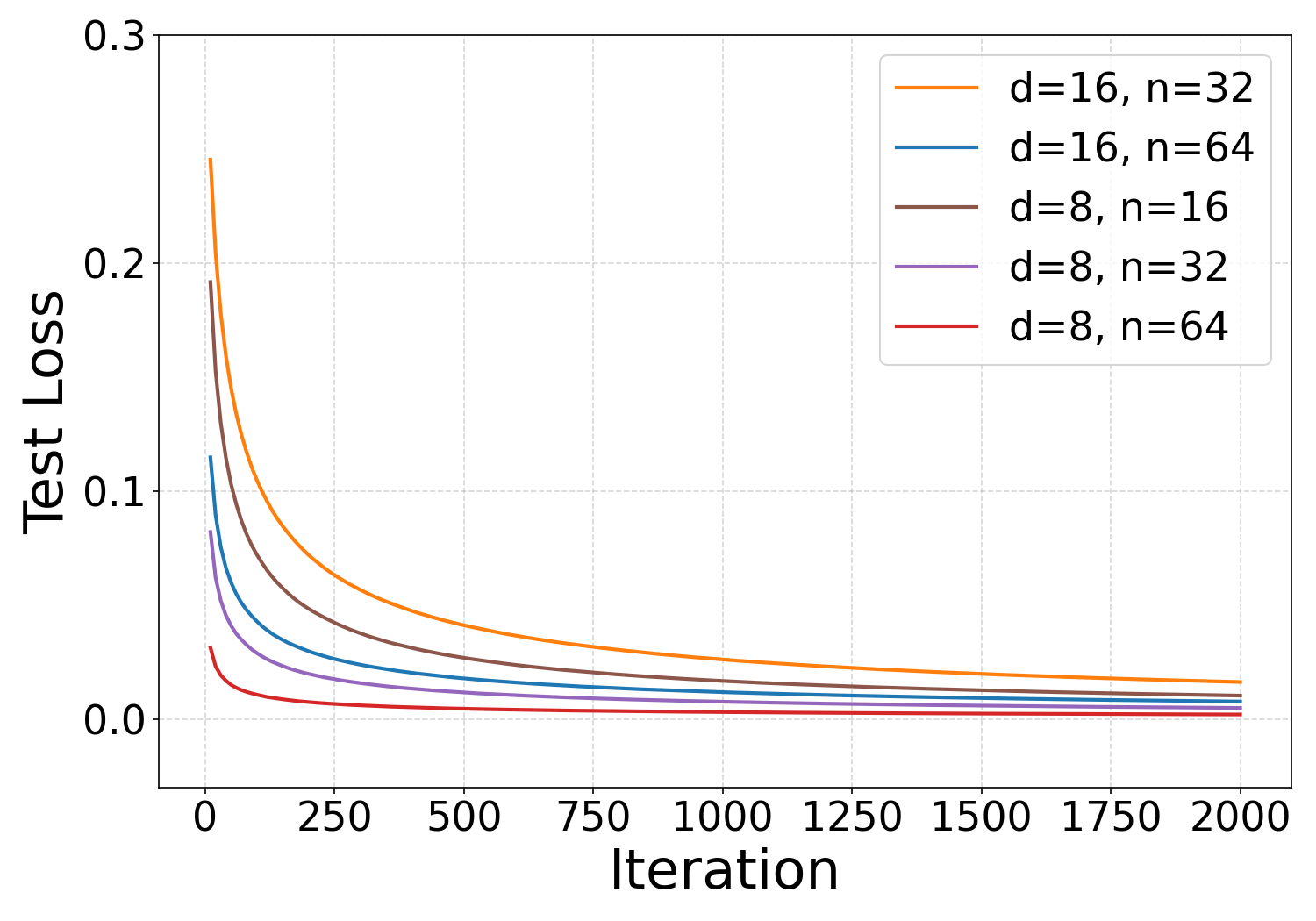}}
\subfigure[Training loss, Task 2]{\includegraphics[width=0.4\textwidth]{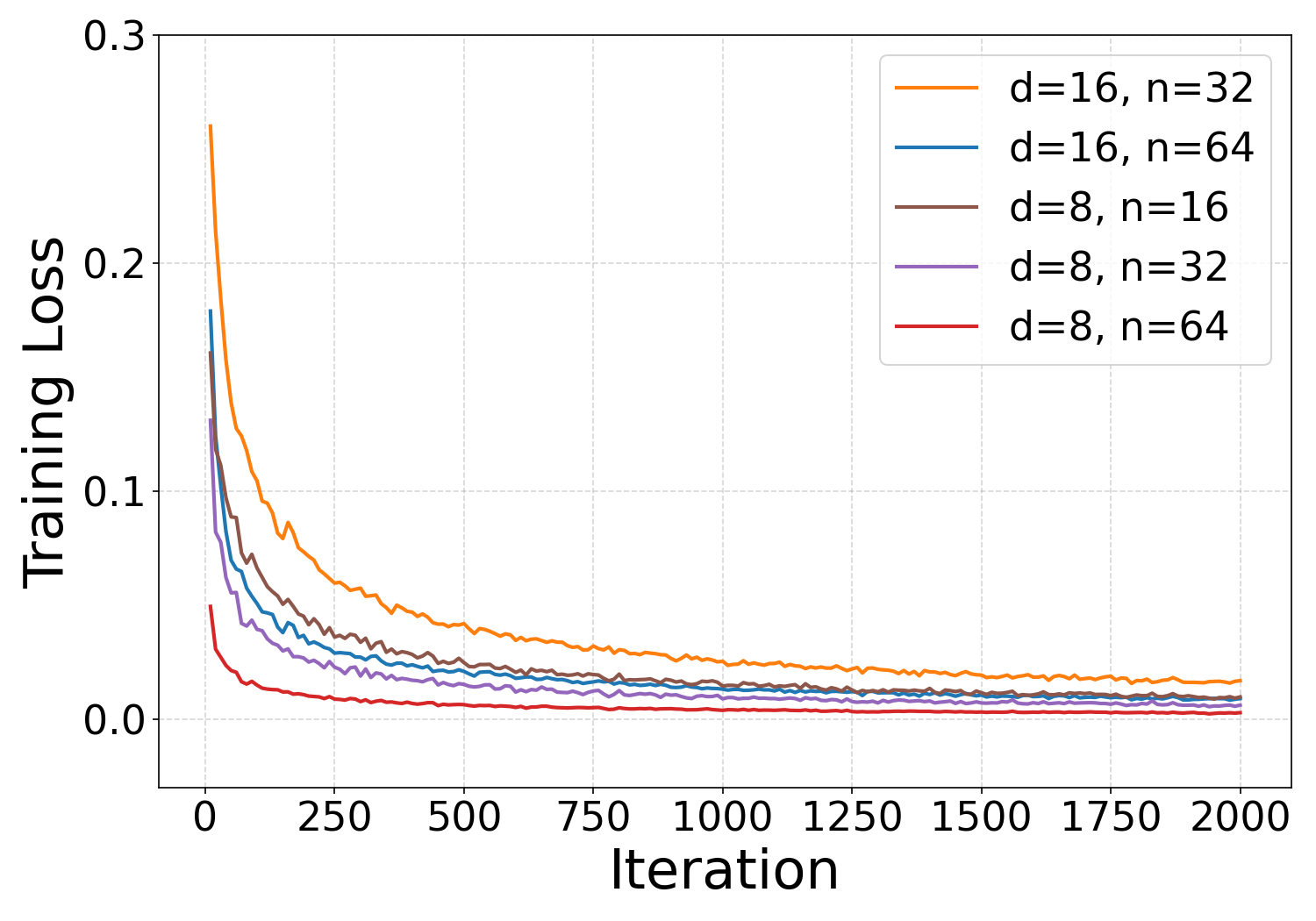}}
\subfigure[Test loss, Task 2]{\includegraphics[width=0.4\textwidth]{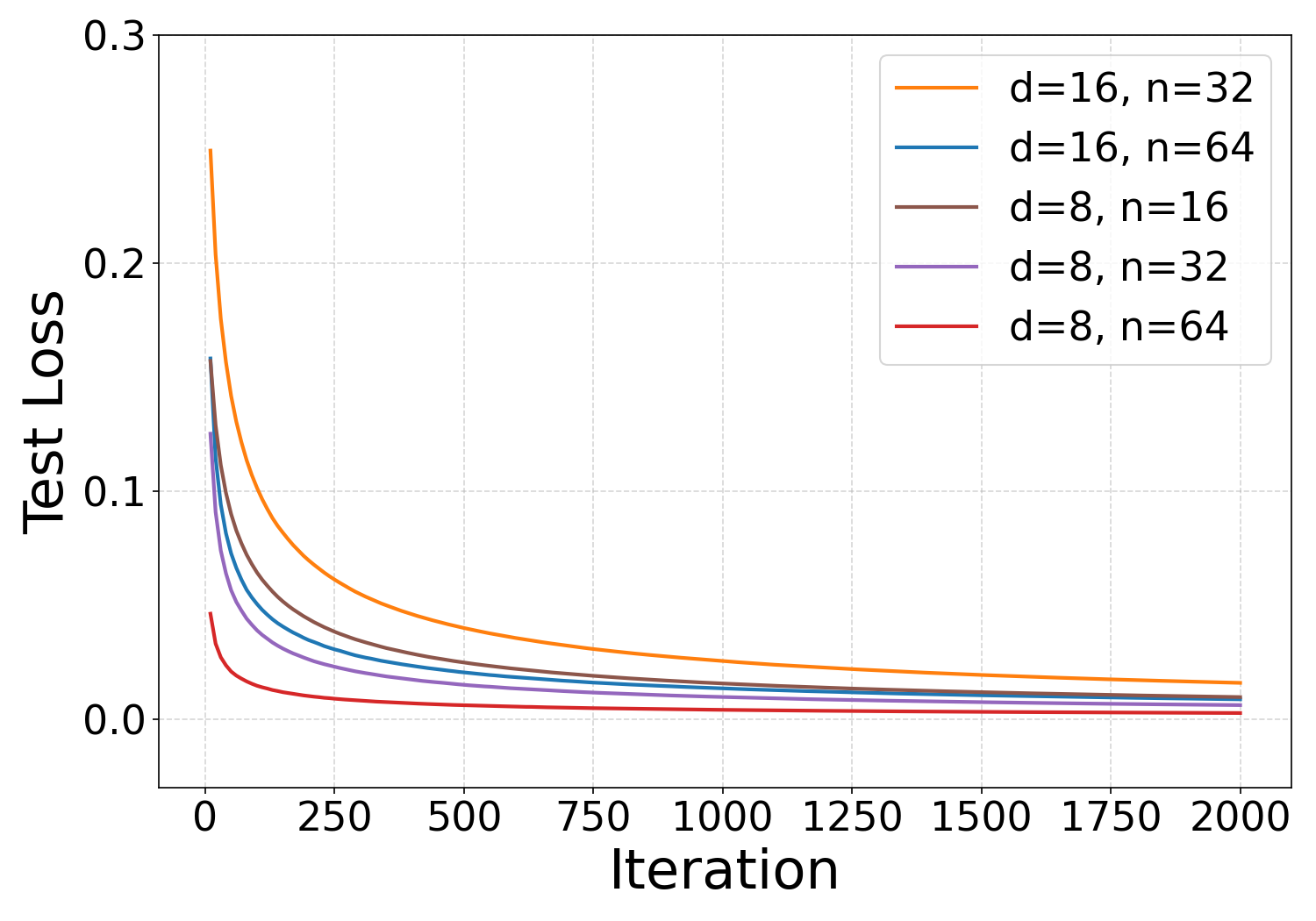}}
\caption{Training and test losses for tasks 1 and 2 with different $d$ and $n$.}
\label{fig:loss_curves}
 \vspace{-5mm}
\end{figure*}

\noindent\textbf{Learned parameter structure.}
%To further verify our theory, we visualize the  matrices $\mathbf{V}$ and $\mathbf{W}$ after 2000 training steps for $d=16, n=32$. As shown in Figure~\ref{fig:heatmaps}, the matrices converge to the sparse structure predicted by our analysis in Theorem \ref{thm:convergence}. Specifically, all entries of both matrices converge to zero except for the elements on the diagonal in the bottom-left block of $\mathbf{V}$ ($\mathbf{V}_{21}$) and the top-right block of $\mathbf{W}$ ($\mathbf{W}_{12}$). This confirms that gradient descent successfully drives the parameters to a solution functionally equivalent to the one-step Power Method update.
We plot the heatmaps of the matrices $\mathbf{V}$ and $\mathbf{W}$ after 2000 training iterations and the evolution of the learned weights for $d=16, n=32$. As shown in Figure~\ref{fig:heatmaps}, all entries of both matrices converge to zero except for the elements on the diagonal in the bottom-left block of $\mathbf{V}$ and the top-right block of $\mathbf{W}$, aligning with the structure predicted by our analysis in Theorem \ref{thm:convergence}. Furthermore, the evolution of these active parameters reveals a polynomial growth pattern. As shown in the log-log plots, the empirical trajectories asymptotically exhibit a slope that aligns with the theoretical $\Theta(t^{1/6})$ prediction, confirming the specific growth rate in Theorem \ref{thm:convergence}.
\begin{figure}[ht!]
\centering
\subfigure[ $\Wb/ \|\Wb \|_{\mathsf{max}}$, Task 1]{\includegraphics[width=0.3\textwidth]{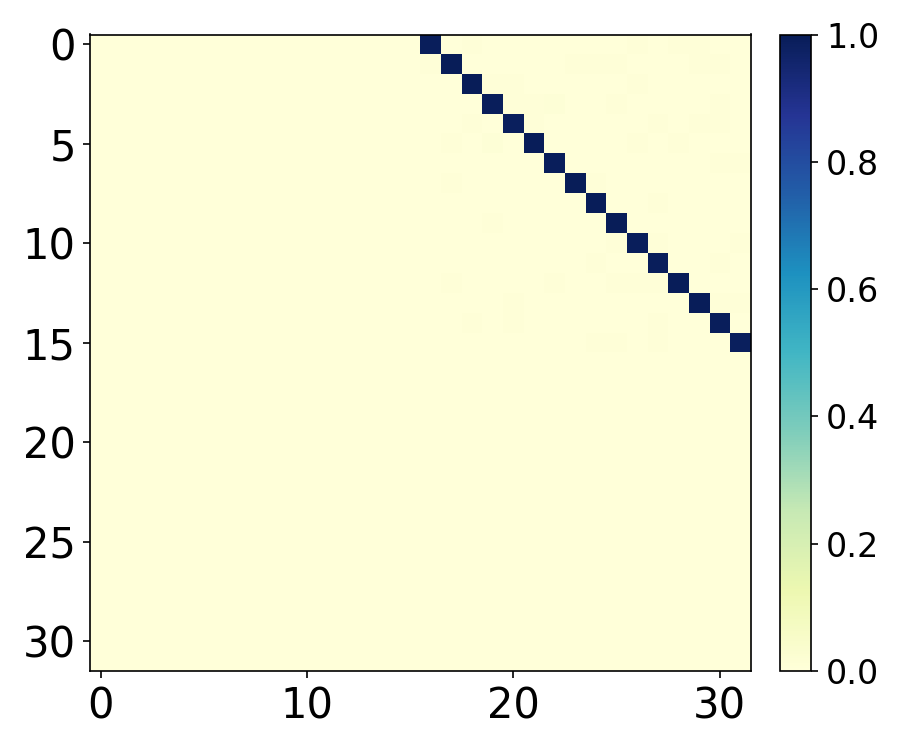}}
\subfigure[ $\Vb/ \|\Vb \|_{\mathsf{max}}$, Task 1]{\includegraphics[width=0.3\textwidth]{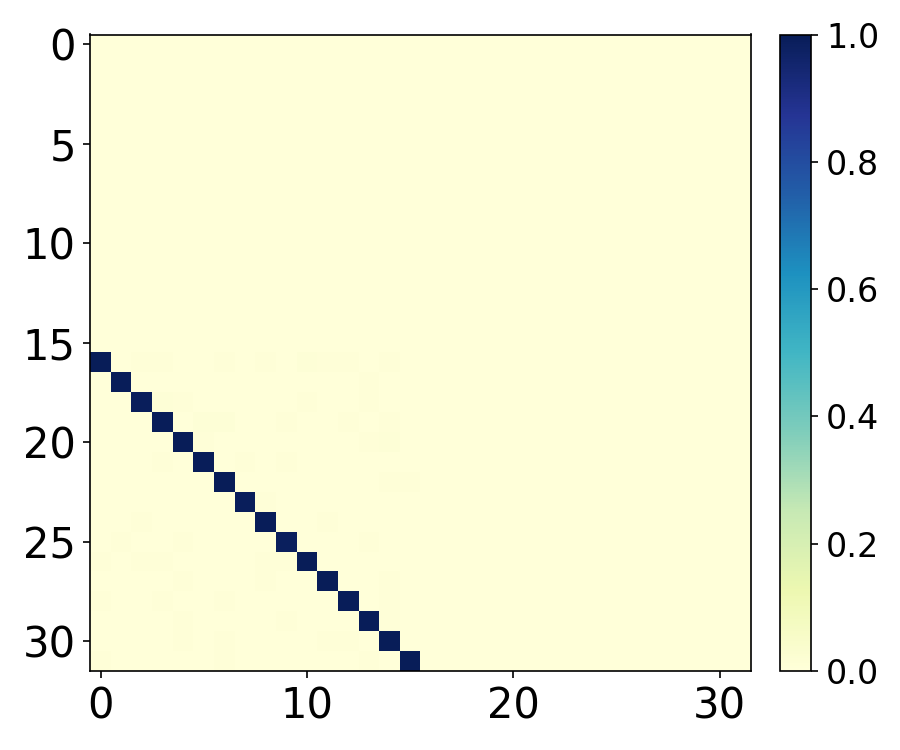}}
\subfigure[$\|\Wb \|_{\mathsf{max}},\|\Vb \|_{\mathsf{max}}$, Task 1]{\includegraphics[width=0.3\textwidth]{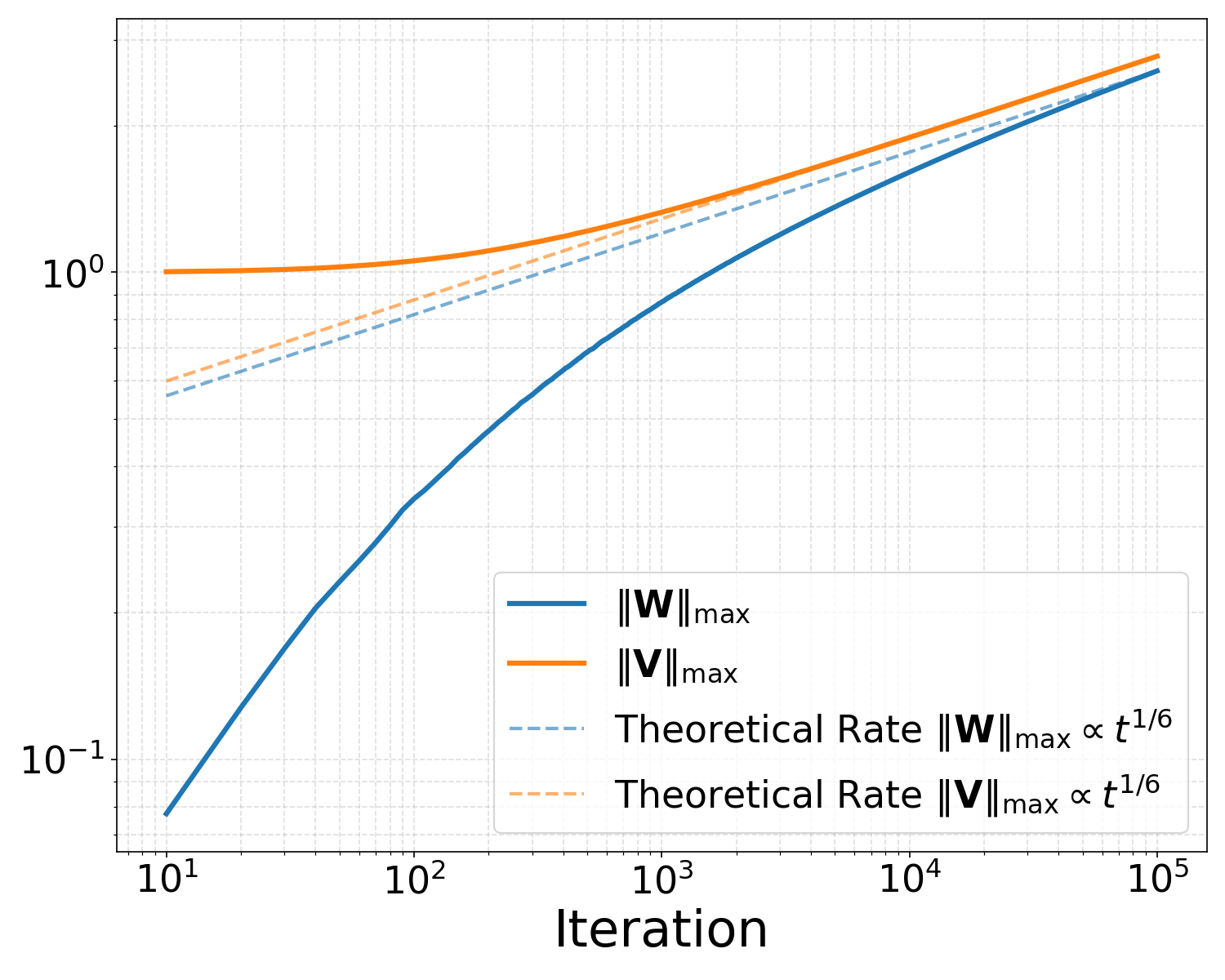}}
\subfigure[ $\Wb/ \|\Wb \|_{\mathsf{max}}$, Task 2]{\includegraphics[width=0.3\textwidth]{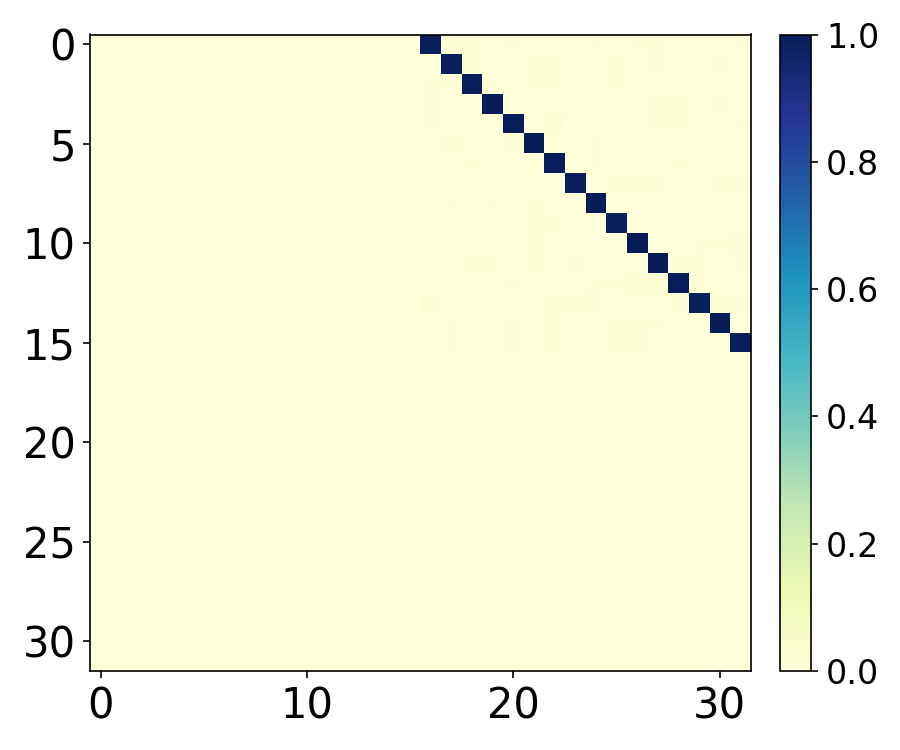}}
\subfigure[$\Vb/ \|\Vb \|_{\mathsf{max}}$, Task 2]{\includegraphics[width=0.3\textwidth]{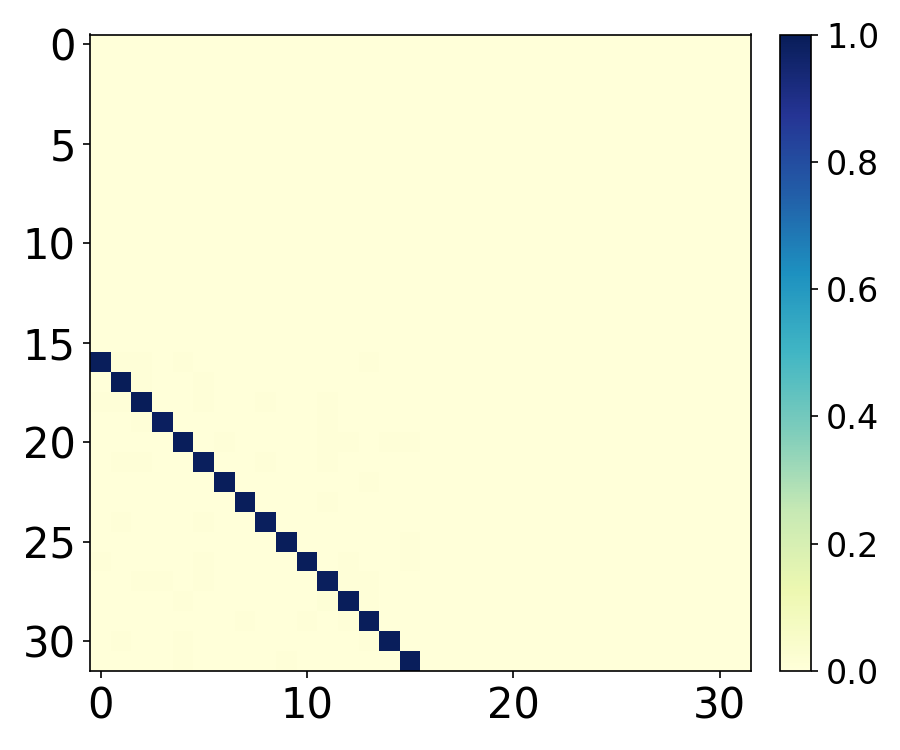}}
\subfigure[$\|\Wb \|_{\mathsf{max}},\|\Vb \|_{\mathsf{max}}$, Task 2]{\includegraphics[width=0.3\textwidth]{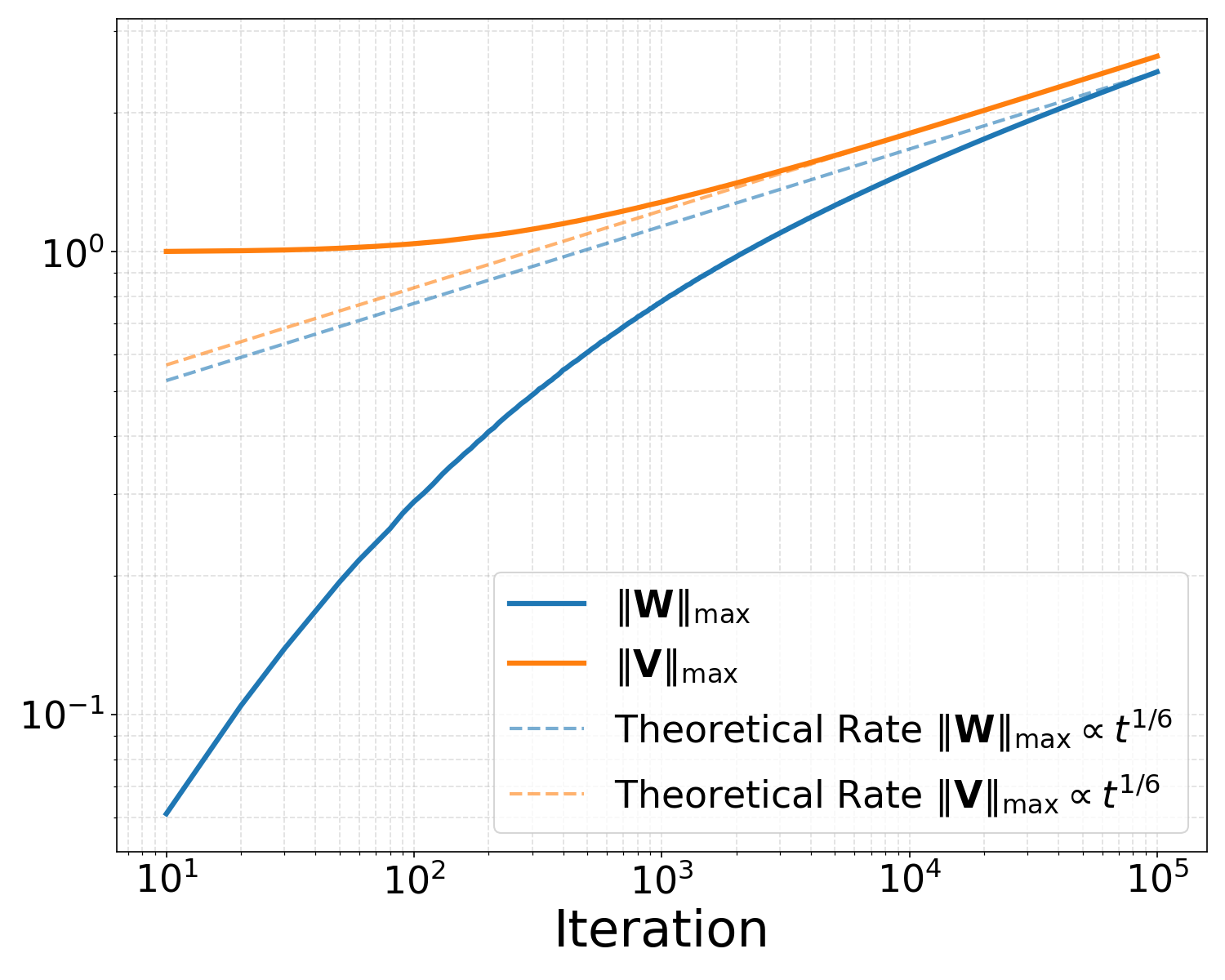}}
\caption{Heatmaps of parameter matrices and evolution of active parameters. Since the active parameters of both matrices diverge, we plot heatmaps to visualize the structures of $\Wb/ \|\Wb \|_{\mathsf{max}}$ and $\Vb/ \|\Vb \|_{\mathsf{max}}$. The dashed lines in the log-log plots represent the theoretical growth rate $\Theta(t^{1/6})$, vertically anchored to the final empirical data points.}
\label{appendix:fig:heatmaps}
 \vspace{-5mm}
\end{figure}

    %Absolute values of the final weight matrices $\mathbf{W}$ (left) and $\mathbf{V}$ (right) after 2000 training steps. The parameters converge to a sparse structure where only the $\mathbf{W}_{12}$ and $\mathbf{V}_{21}$ blocks contain non-zero weights, aligning with the theoretical predictions in Theorem \ref{thm:convergence}.

\noindent\textbf{Performance of Looped Transformer} We empirically verify the performance of the looped transformer. For a fixed test point $(\mathbf{X}_{\text{test}}, \mathbf{a}_{\text{test}})$, the ground-truth principal eigenvector is numerically computed from the empirical covariance matrix $\mathbf{X}_{\text{test}}\mathbf{X}_{\text{test}}^\top$. We set $d=16$, $n=32$, take model snapshots at different training steps $T \in \{100, 300, 2000\}$, and apply them for $L\in\{1, 2,\ldots,25\}$ iterations on the fixed test point. Figure~\ref{fig:looped_eval} plots the error, measured by the angular deviation, against the number of loops $L$ on a semi-logarithmic scale. For each trained snapshot, the error curve appears as a nearly linear trajectory. This linear relationship clearly shows that the error decays exponentially as the number of loops $L$ increases. Furthermore, we also plot the error curve for the power method and observe that models trained for more iterations (larger $T$) exhibit an error curve closer to that of the power method, indicating a faster convergence rate.

\begin{figure}[ht!]
\centering
\subfigure[Looped transformer, Task 1]{\includegraphics[width=0.45\textwidth]{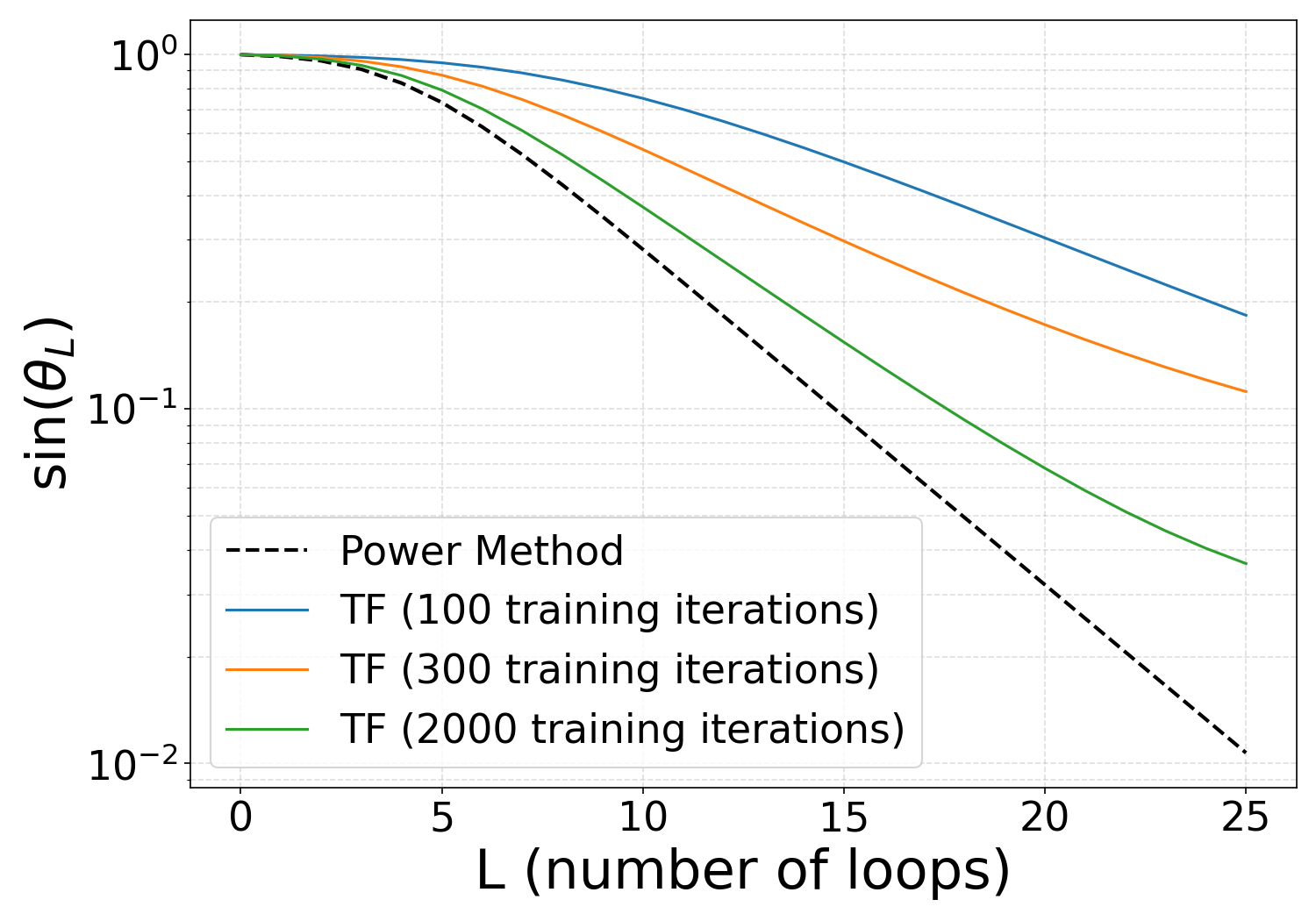}}
\subfigure[Looped transformer, Task 2]{\includegraphics[width=0.45\textwidth]{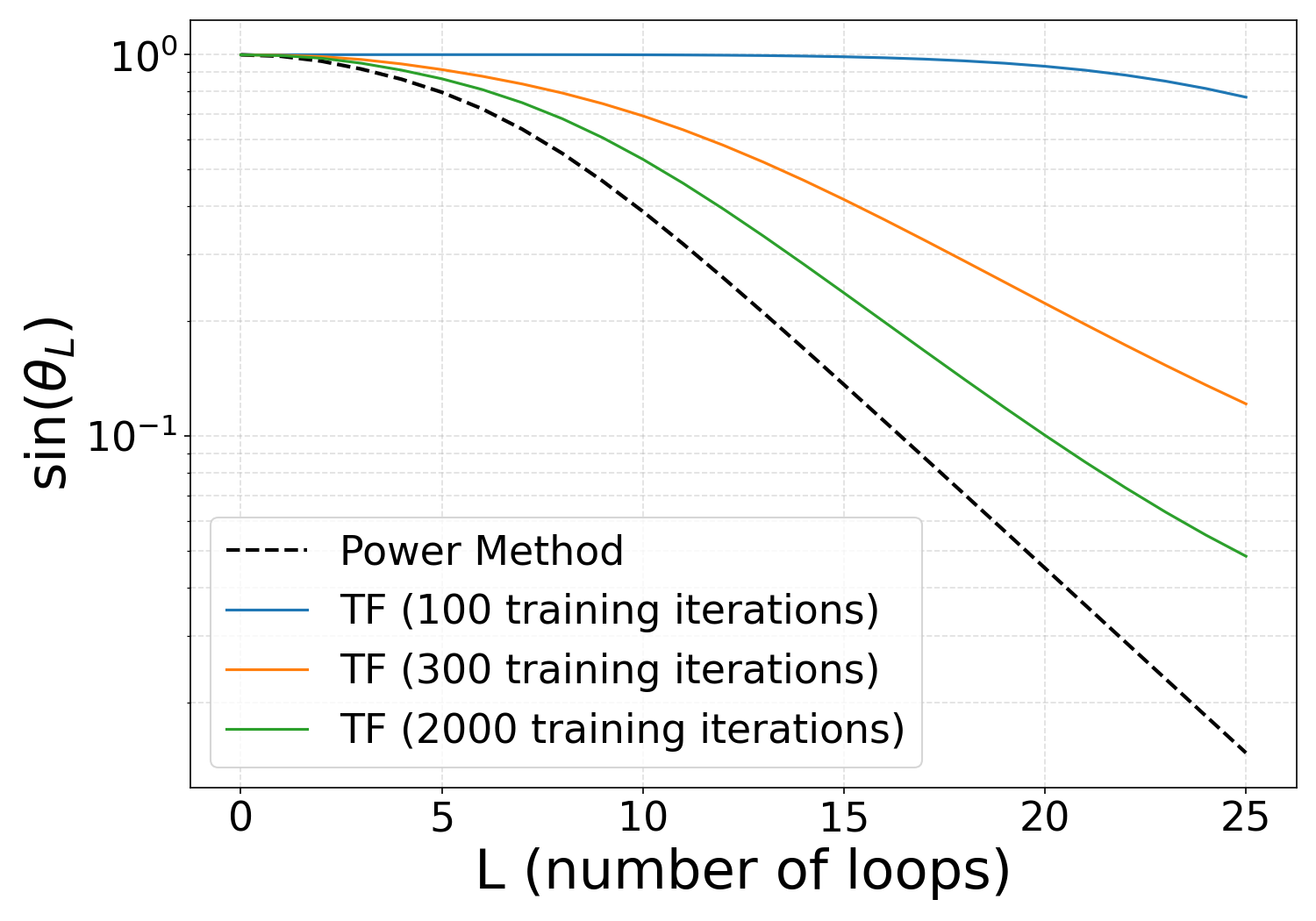}}
\caption{Semi-log plots of the error $\sin(\theta_L)$ versus the number of loops $L$ for models trained for $T \in \{100, 300, 2000\}$ iterations on tasks 1 and 2.}
 \label{fig:looped_eval}
 \vspace{-5mm}
\end{figure}
\subsection{Learned Parameter Matrices for Unnormalized Variants}
In this subsection, we visualize the structure of the parameter matrices predicted by Theorem~\ref{thm:unnorm convergence}.

Figure~\ref{fig:learned parameter unnormalized task 1} visualizes the learned parameter matrices and the evolution of active parameters for both unnormalized variants on task 1. These results confirm the structure and convergence to fixed points as predicted by Theorem~\ref{thm:unnorm convergence}.

\begin{figure}[ht!]
% \vspace{-2mm}
\centering   
\subfigure[$\Wb/ \|\Wb \|_{\mathsf{max}}$ for $\yb_{\text{target}}^{(1)}$]{\includegraphics[width=0.3\textwidth]{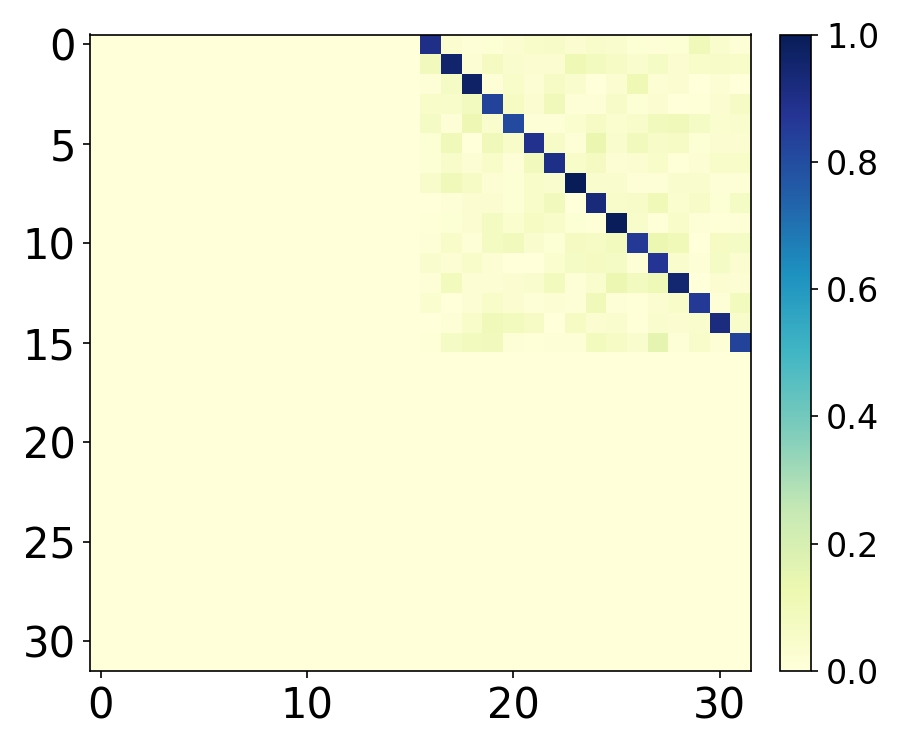}}
\subfigure[$\Vb/ \|\Vb \|_{\mathsf{max}}$ for $\yb_{\text{target}}^{(1)}$]{\includegraphics[width=0.3\textwidth]{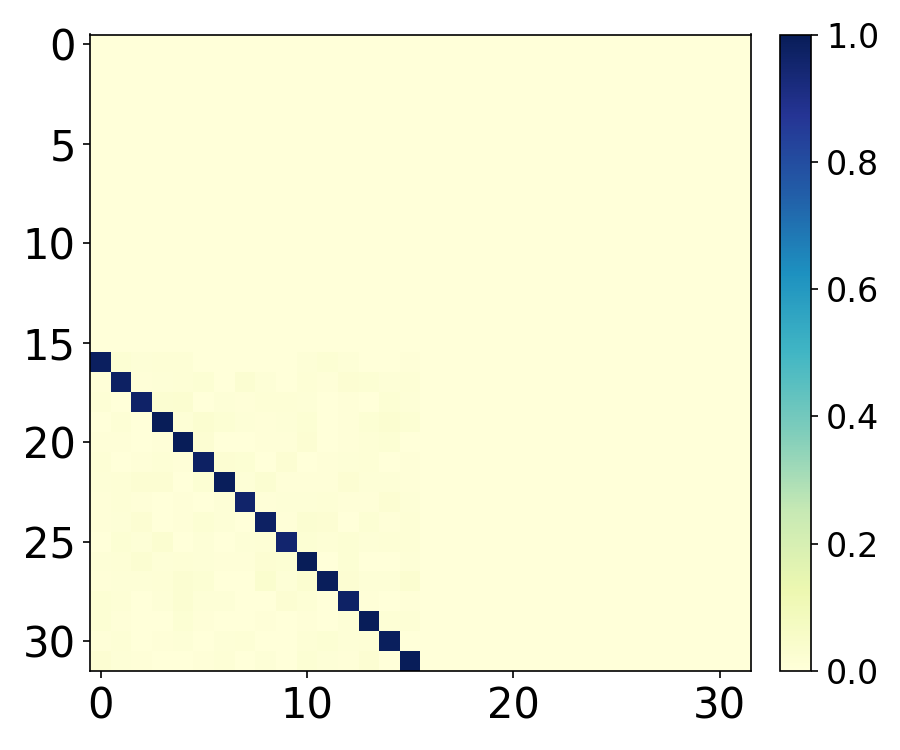}}
\subfigure[$\|\Wb \|_{\mathsf{max}},\|\Vb \|_{\mathsf{max}}$ for $\yb_{\text{target}}^{(1)}$]{\includegraphics[width=0.3\textwidth]{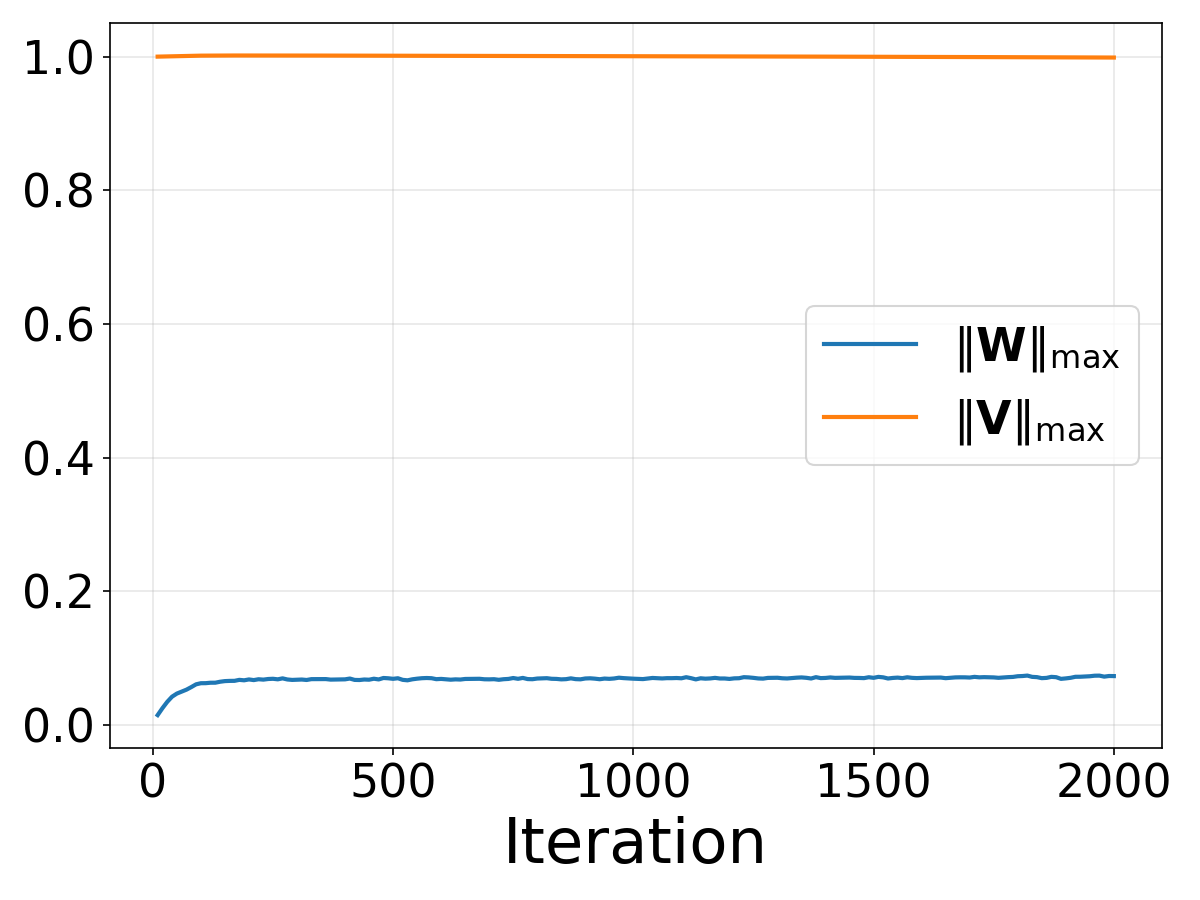}}
\subfigure[$\Wb/ \|\Wb \|_{\mathsf{max}}$ for $\yb_{\text{target}}^{(2)}$]{\includegraphics[width=0.3\textwidth]{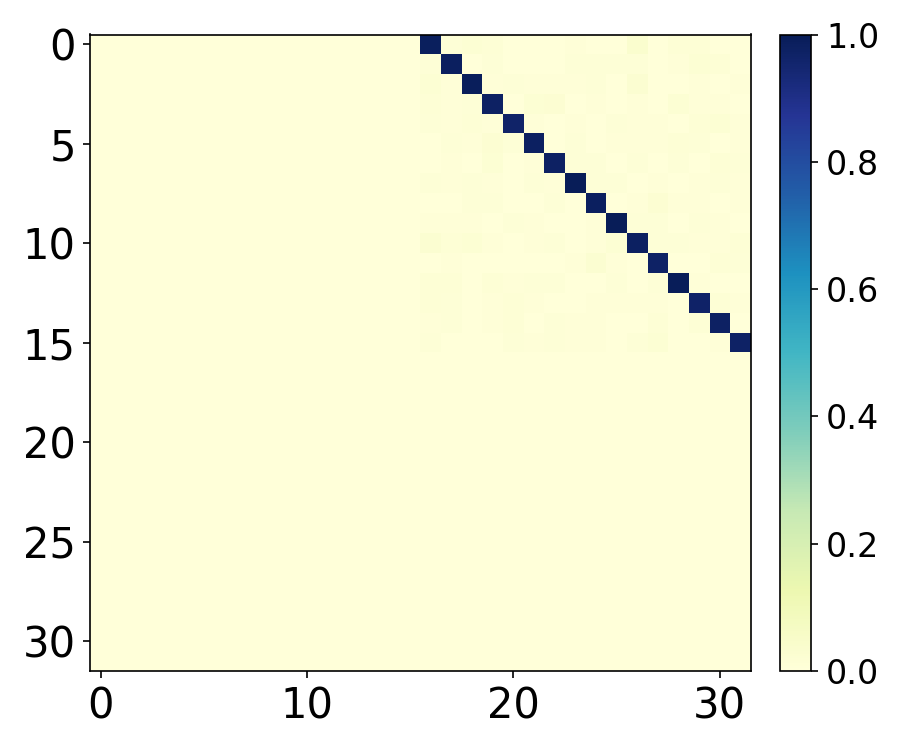}}
\subfigure[$\Vb/ \|\Vb \|_{\mathsf{max}}$ for $\yb_{\text{target}}^{(2)}$]{\includegraphics[width=0.3\textwidth]{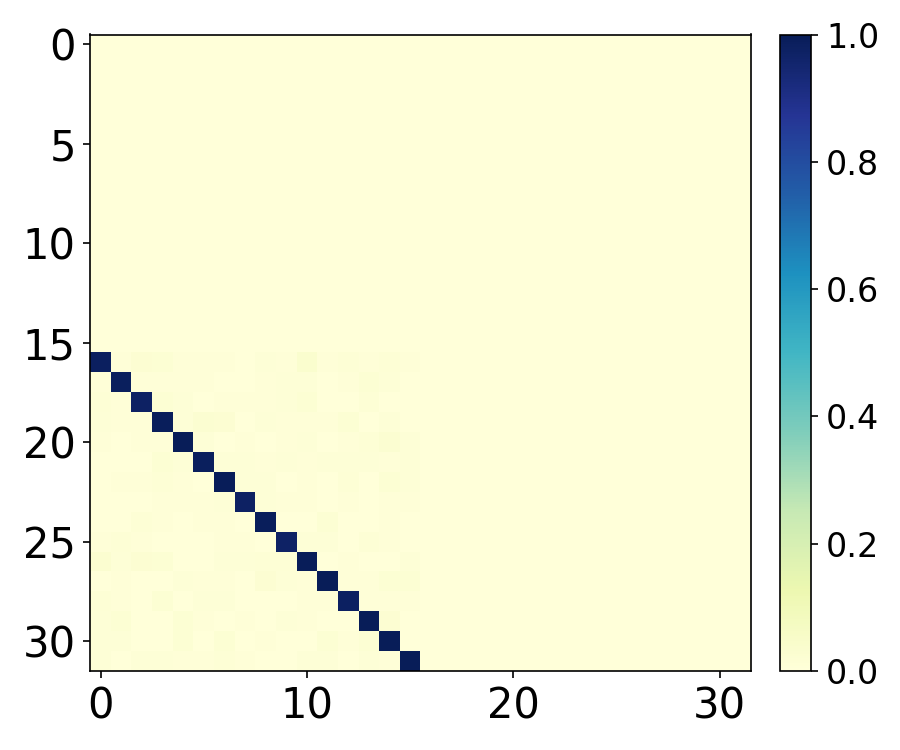}}
\subfigure[$\|\Wb \|_{\mathsf{max}},\|\Vb \|_{\mathsf{max}}$ for $\yb_{\text{target}}^{(2)}$]{\includegraphics[width=0.3\textwidth]{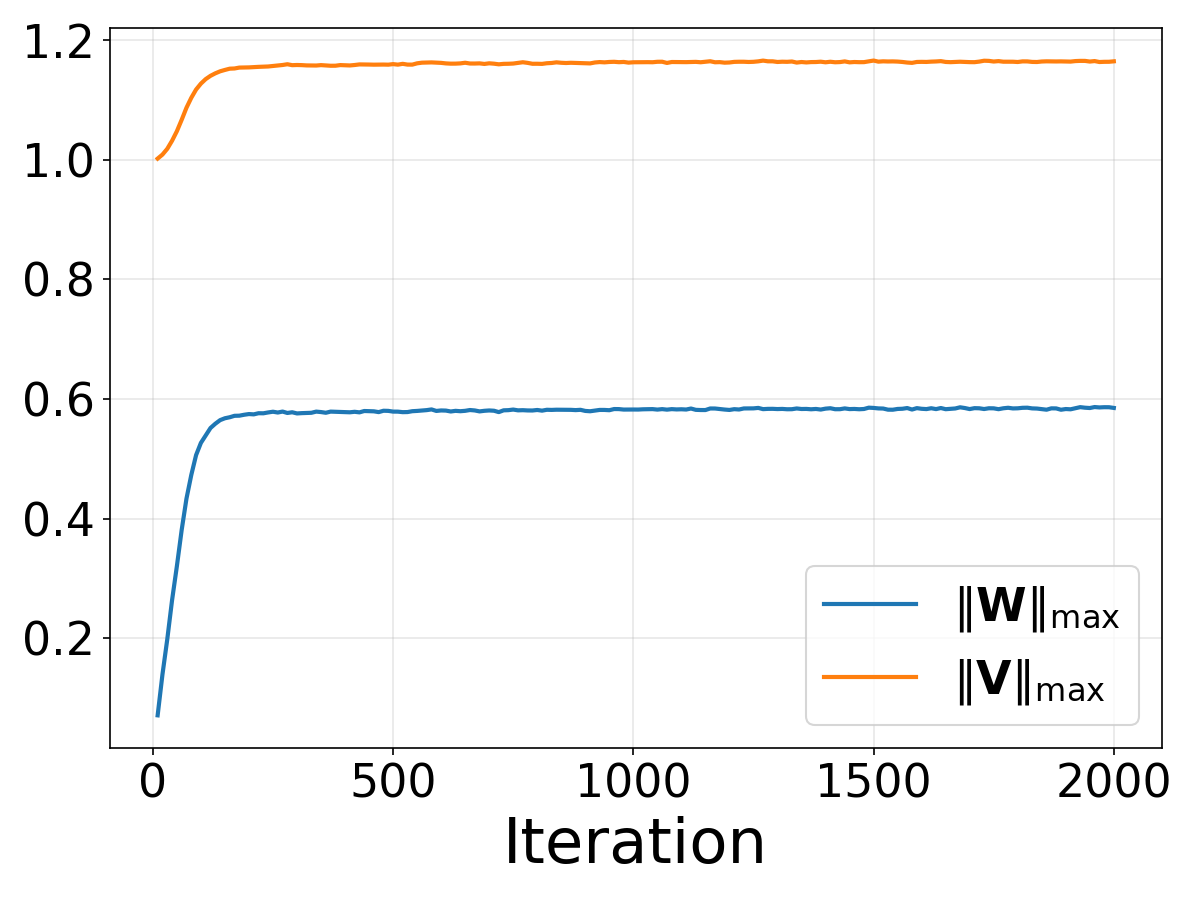}}
\caption{Heatmaps of parameter matrices and evolution of active parameters for both unnormalized models on task 1.}
\label{fig:learned parameter unnormalized task 1}
%\vspace{-5mm}
\end{figure}

Figure~\ref{fig:LN-critical parameter} visualizes the learned parameter matrices and the evolution of active parameters for both unnormalized variants on task 2. All entries of matrices $\Wb$ and $\Vb$ converge to zero except for the elements in the bottom-left block of $\mathbf{V}$ and the top-right block of $\mathbf{W}$, aligning with the sparse structure predicted by our analysis in Theorem \ref{thm:unnorm convergence}. Furthermore, the evolution of these active parameters quickly converges to fixed points. These results confirm the theoretical predictions of Theorem~\ref{thm:unnorm convergence}.

\begin{figure}[ht!]
% \vspace{-2mm}
\centering   
\subfigure[$\Wb/ \|\Wb \|_{\mathsf{max}}$ for $\yb_{\text{target}}^{(1)}$]{\includegraphics[width=0.3\textwidth]{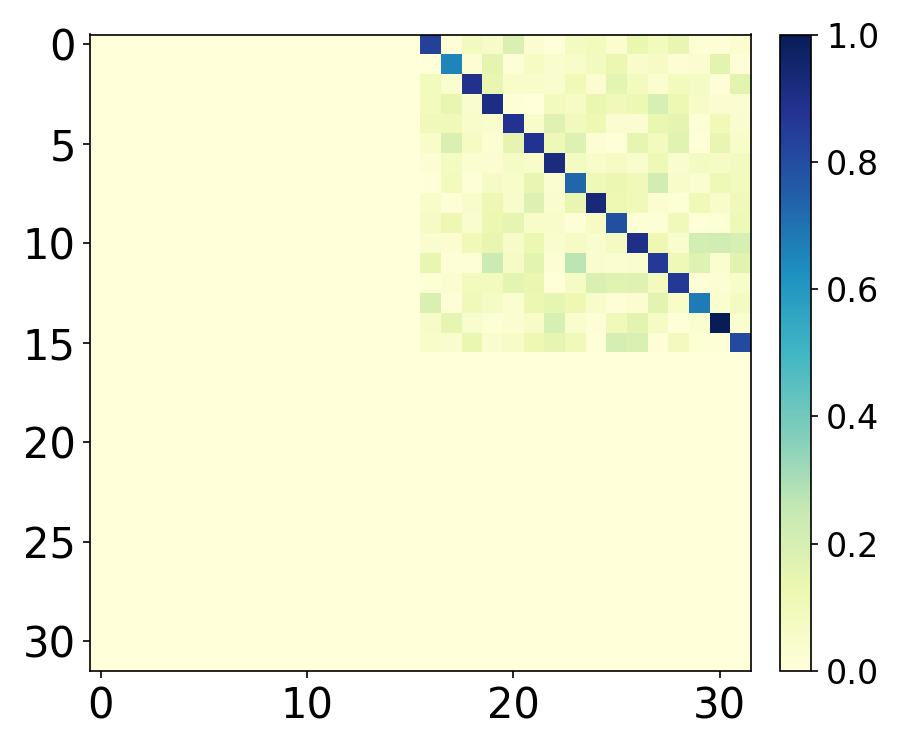}}
\subfigure[$\Vb/ \|\Vb \|_{\mathsf{max}}$ for $\yb_{\text{target}}^{(1)}$]{\includegraphics[width=0.3\textwidth]{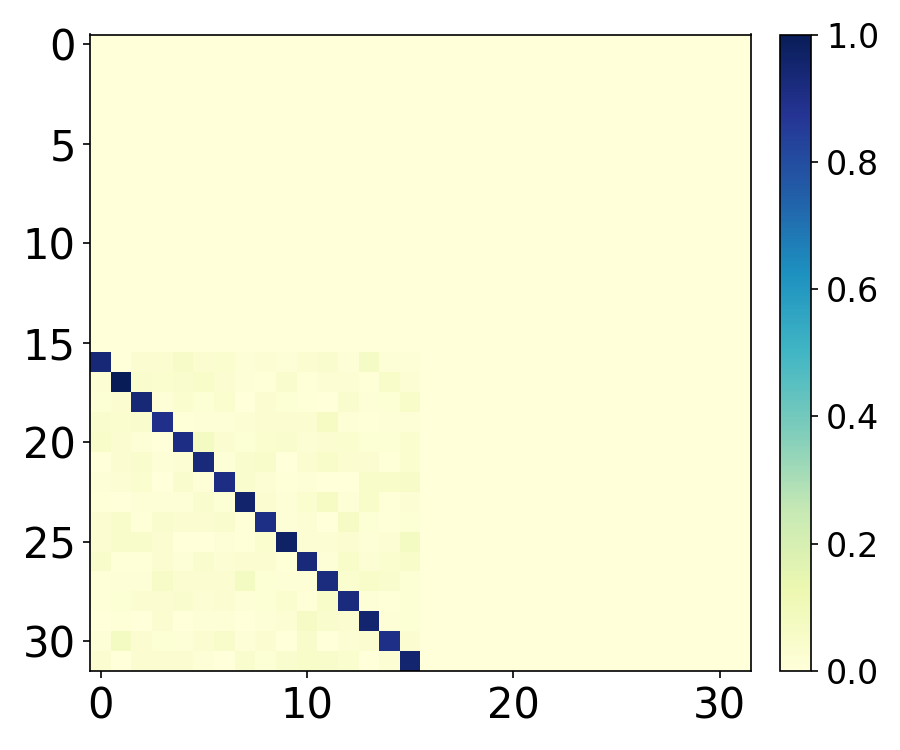}}
\subfigure[$\|\Wb \|_{\mathsf{max}},\|\Vb \|_{\mathsf{max}}$ for $\yb_{\text{target}}^{(1)}$]{\includegraphics[width=0.3\textwidth]{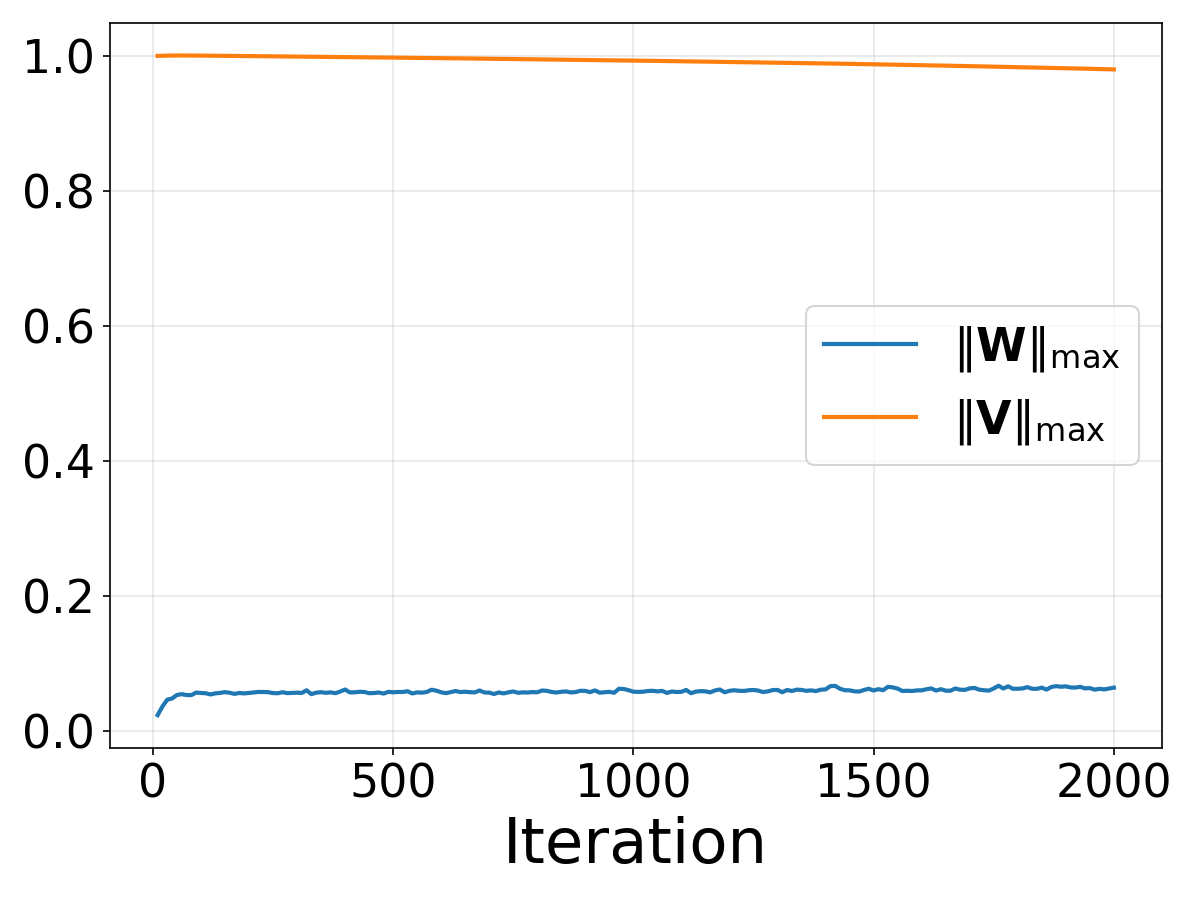}}
\subfigure[$\Wb/ \|\Wb \|_{\mathsf{max}}$ for $\yb_{\text{target}}^{(2)}$]{\includegraphics[width=0.3\textwidth]{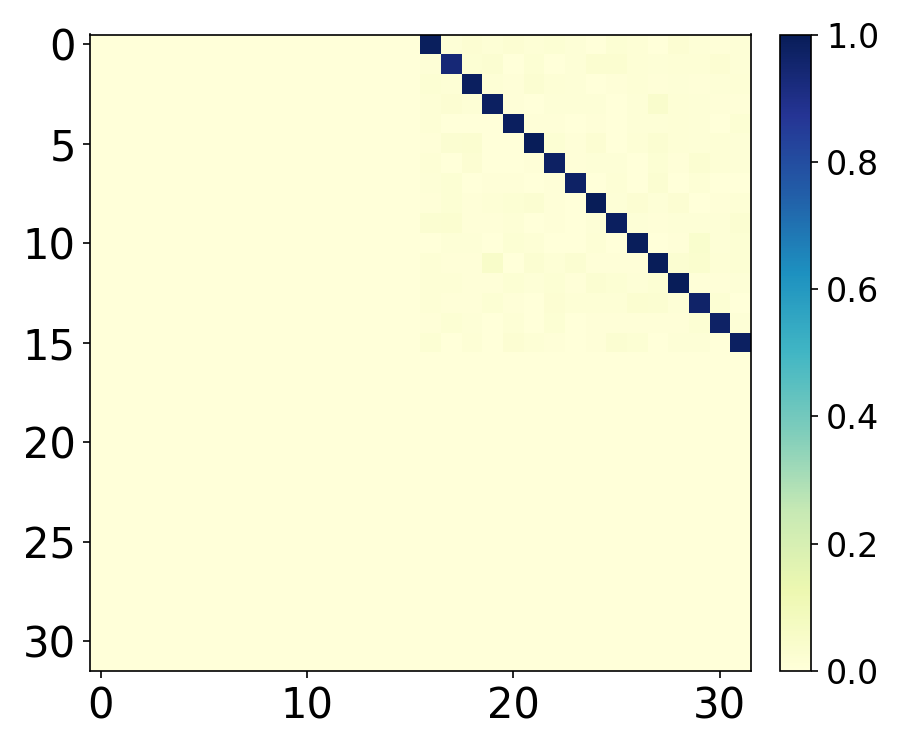}}
\subfigure[$\Vb/ \|\Vb \|_{\mathsf{max}}$ for $\yb_{\text{target}}^{(2)}$]{\includegraphics[width=0.3\textwidth]{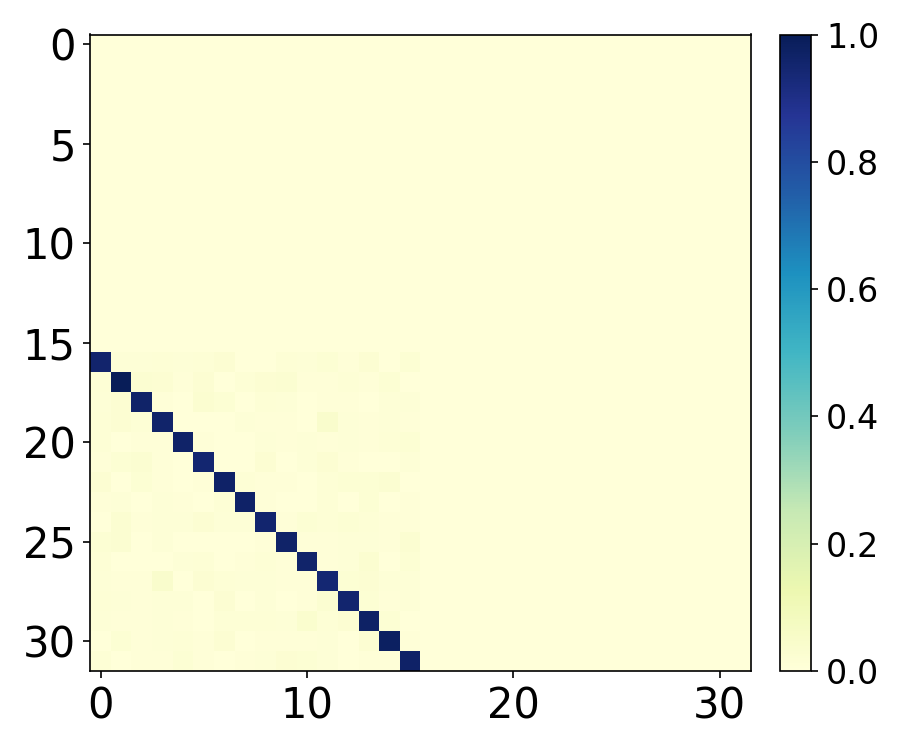}}
\subfigure[$\|\Wb \|_{\mathsf{max}},\|\Vb \|_{\mathsf{max}}$ for $\yb_{\text{target}}^{(2)}$]{\includegraphics[width=0.3\textwidth]{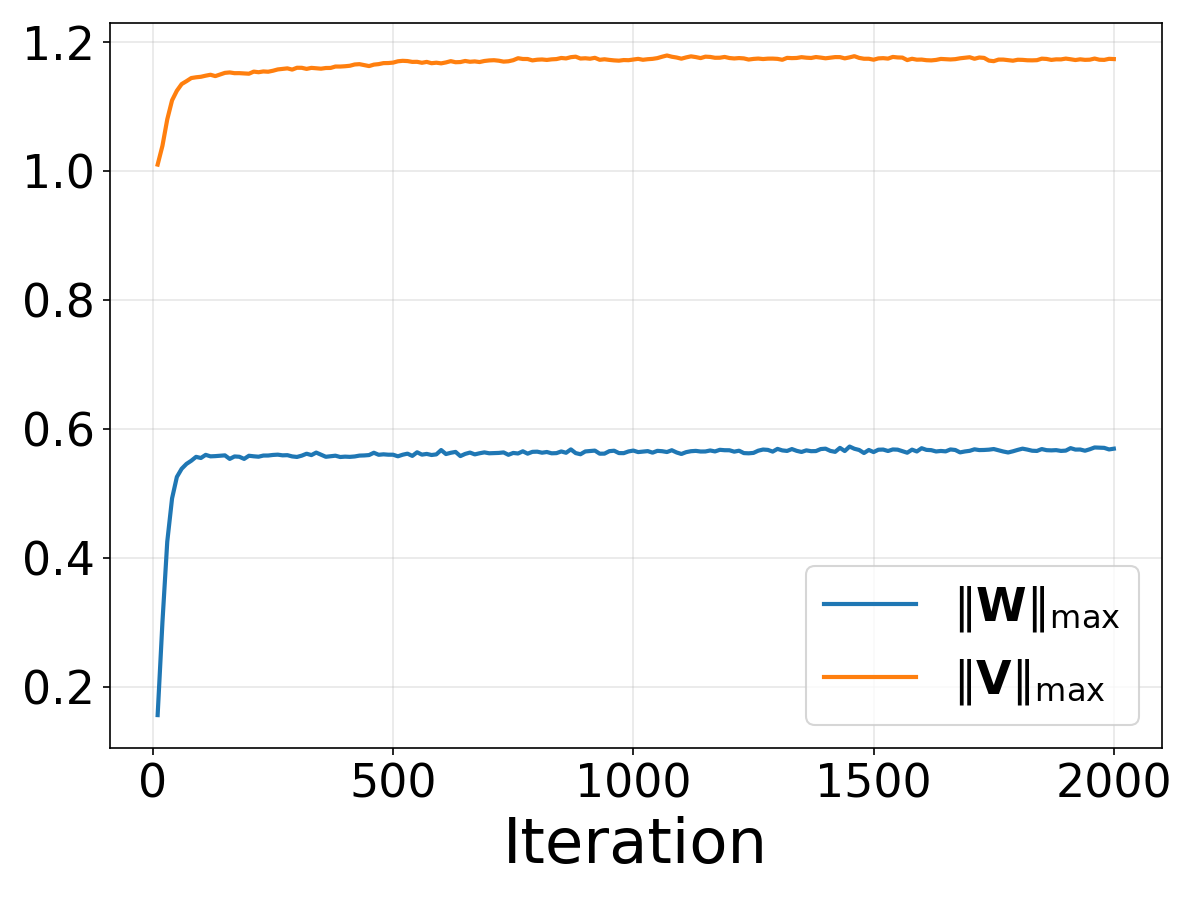}}
\caption{Heatmaps of parameter matrices and evolution of active parameters for both unnormalized models on task 2.}
    \label{fig:LN-critical parameter}
%\vspace{-5mm}
\end{figure}

\subsection{OOD Performance Looped Transformer}

We empirically verify the performance of the looped transformer. For a fixed test point $(\mathbf{X}_{\text{test}}, \mathbf{a}_{\text{test}})$, the ground-truth principal eigenvector is numerically computed from the empirical covariance matrix $\mathbf{X}_{\text{test}}\mathbf{X}_{\text{test}}^\top$. We set $d=16$, $n=32$. We choose $L=10$ and train the looped model for $T=2000$ steps and take snapshots of $\Wb$ and $\Vb$ after training. Using the resulting $\Wb$ and $\Vb$, we construct looped transformers with $L \in \{1,2,\ldots,25\}$ layers and evaluate them on the fixed test point.

Figure~\ref{fig:looped_eval} plots the error, measured by the angular deviation, against the number of layers $L$ on a semi-logarithmic scale. The error curve appears as a nearly linear trajectory. This linear relationship clearly shows that the error decays exponentially as the number of layers $L$ increases. %Moreover, the error curve is close to that of the power method.

\begin{figure}[ht!]
\centering
\subfigure[Looped transformer, Task 1]{\includegraphics[width=0.45\textwidth]{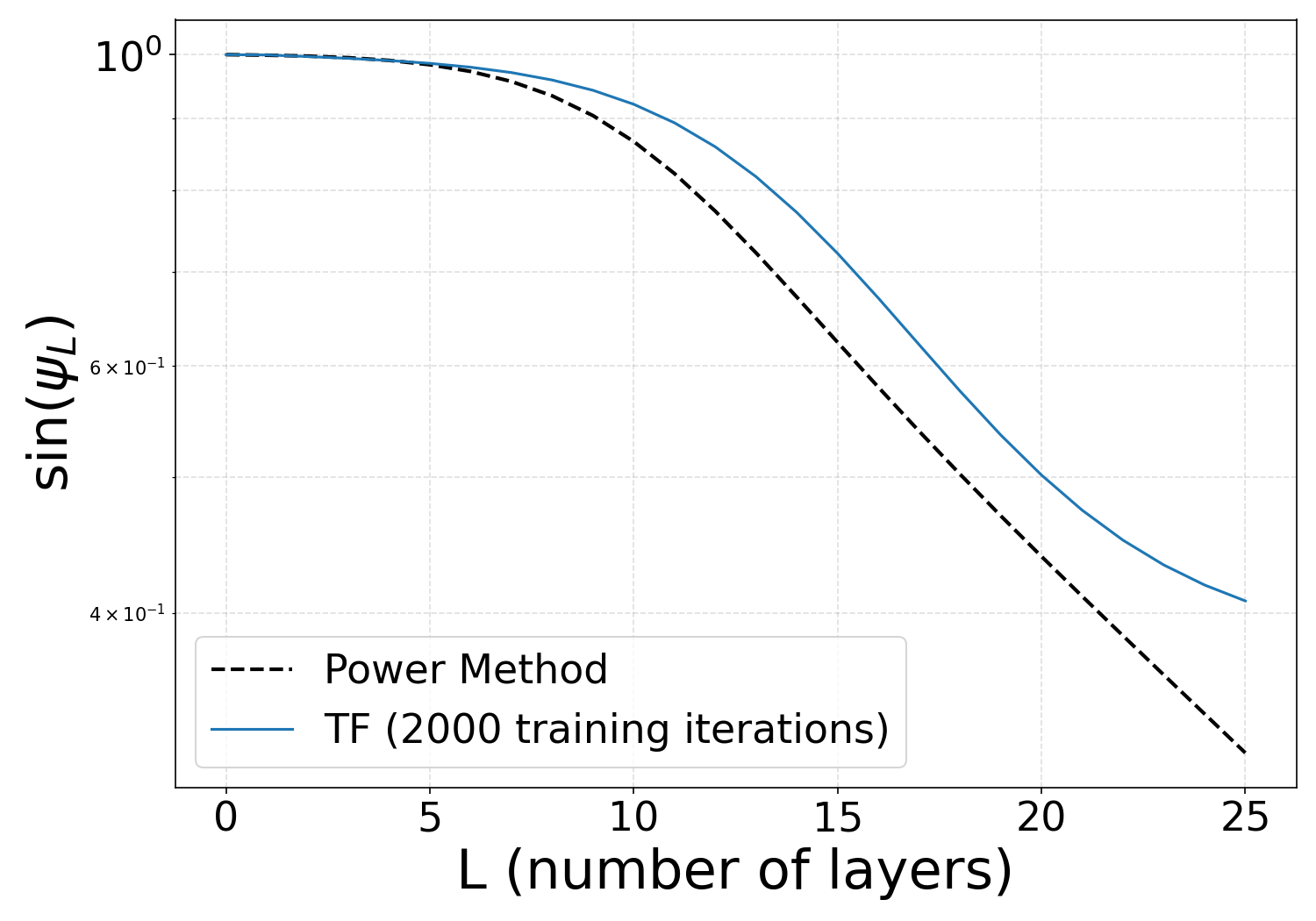}}
\subfigure[Looped transformer, Task 2]{\includegraphics[width=0.45\textwidth]{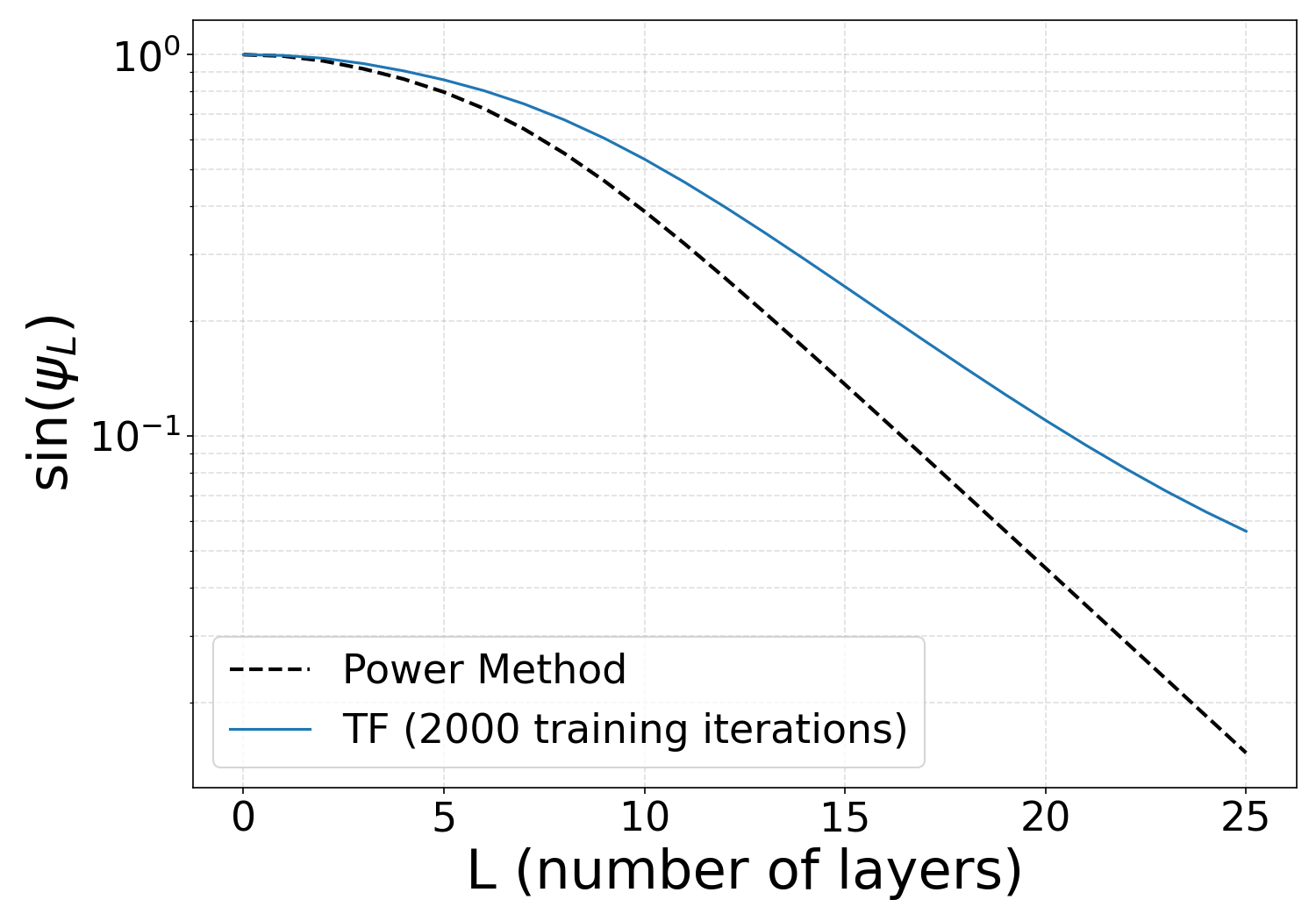}}
\caption{Semi-log plots of the error $\sin(\psi_L)$ versus the number of layers $L$ for models trained for $T = 2000$ iterations on tasks 1 and 2.}
 \label{fig:e2e_looped_eval}
 \vspace{-5mm}
\end{figure}

\bibliography{ref}

\bibliographystyle{ims}

\end{document}